%% file: main.tex
\documentclass[11pt,letterpaper]{article}
\input{preamble}

\DeclareMathOperator{\boldX}{\boldsymbol{X}}

\DeclareMathOperator{\reals}{\mathbb{R}}

\SetKwProg{Init}{Initialize: }{}

\usepackage{graphicx}
\newcommand{\widesim}[2][1.5]{
  \mathrel{\overset{#2}{\scalebox{#1}[1]{$\sim$}}}
}

\SetKwProg{Init}{Initialize: }{}

\begin{document}

\title{Robust Learnability of Sample-Compressible Distributions under Noisy or Adversarial Perturbations}

\author{Arefe Boushehrian \and Amir Najafi\thanks{Corresponding author 
}}

\date{
\vspace*{1mm}
\parbox{\linewidth}{
\centering
  Department of Computer Engineering,\\ 
  Sharif University of Technology, Tehran, Iran
  \\[2mm]
  E-mails:~\texttt{\{arefe.boushehrian82,amir.najafi\}@sharif.edu}
  }
}

\maketitle

\begin{abstract}
\vspace*{1mm}
Learning distribution families over $\reals^d$ is a fundamental problem in unsupervised learning and statistics. A central question in this setting is whether a given family of distributions possesses sufficient structure to be (at least) information-theoretically learnable and, if so, to characterize its sample complexity. In 2018, Ashtiani et al. reframed \emph{sample compressibility}—originally due to Littlestone and Warmuth (1986)—as a structural property of distribution classes, proving that it guarantees PAC-learnability. This discovery subsequently enabled a series of recent advancements in deriving nearly tight sample complexity bounds for various high-dimensional open problems. It has been further conjectured that the converse also holds: every learnable class admits a tight sample compression scheme.

In this work, we establish that sample compressible families remain learnable even from perturbed samples, subject to a set of necessary and sufficient conditions. We analyze two models of data perturbation: (i) an additive independent noise model, and (ii) an adversarial corruption model, where an adversary manipulates a limited subset of the samples unknown to the learner. Our results are general and rely on as minimal assumptions as possible. We develop a perturbation-quantization framework that interfaces naturally with the compression scheme and leads to sample complexity bounds that scale gracefully with the noise level and corruption budget. As concrete applications, we establish new sample complexity bounds for learning finite mixtures of high-dimensional uniform distributions under both noise and adversarial perturbations, as well as for learning Gaussian mixture models from adversarially corrupted samples—resolving two open problems in the literature.
\end{abstract}

\tableofcontents


\section{Introduction}
\label{sec:intro}

\input{Intro}


\subsection{Notations}
\label{sec:notations}
\input{Notations}

\subsection{Formal Problem Definition}
\label{sec:problemDef}
\input{Problem_definition}




\section{Additive Noise Model}
\label{sec:Gaussian}
\input{Gaussian}


\section{Adversarial Perturbations}
\label{sec:adversarial}

\input{Adversary}


\section{Some Theoretical Examples}
\label{sec:examples}

\input{Example}

\section{Conclusions}
\label{sec:conclusions}
\input{conclusions}

\subsection{Open Problems}
\input{open_problems}



\bibliographystyle{alpha}
\bibliography{ref}


\appendix

\section{Proofs of Section \ref{sec:Gaussian}: Additive Noise Model}
\label{sec:app:proofs_additive_noise}

\input{App_Proposition_Sample_Compression_Proof}
\vspace{4mm}
\input{Proof_of_Gaussian}

\subsection{Proofs of Claims: Part I}
\label{sec:app:proofs:claims:I}

\input{lipschitz_gaussian_claim_proof}
\vspace{4mm}
\input{lipschitz_uniform_claim_proof}
\vspace{4mm}
\input{lipschitz_mixture_claim_proof}
\vspace{4mm}
\input{Claim_Proof_Impossibility}

\subsection{Proofs of Claims: Part II}
\label{sec:app:proofs:claims:II}

\input{Claim_proof_low_frequency}
\vspace{4mm}
\input{frequency_gaussian_claim_proof}
\vspace{4mm}
\input{frequency_uniform_claim_proof}

\section{Proofs of Section \ref{sec:adversarial}: Adversarial Perturbation Model}
\label{sec:app:proofs_adversarial}
\input{Adversarial_Sample_Decoder_Proof}
\vspace{4mm}
\input{adversarial_main_thoerem_proof}

\section{Proofs of Section \ref{sec:examples}: Theoretical Examples}
\label{sec:app:example}

\input{App_Examples_Proof}





\end{document}

%% file: preamble.tex
\usepackage{amsmath,amsthm,nicefrac}
\usepackage{amsfonts, amstext}
\usepackage[square,numbers]{natbib}
\usepackage{subcaption}

\usepackage{comment}

\usepackage{typearea}
\typearea{11}
\usepackage[font={small,it}]{caption}

\usepackage{setspace}

\usepackage{enumitem}
\usepackage[breaklinks]{hyperref}
\hypersetup{colorlinks=true,%
            citebordercolor={.6 .6 .6},linkbordercolor={.6 .6 .6},%
citecolor=blue,urlcolor=black,linkcolor=blue}

\usepackage[nameinlink]{cleveref}
\Crefname{algocf}{Algorithm}{Algorithms}
\crefname{algocfline}{line}{lines}
\Crefname{invariant}{Invariant}{Invariants}
\Crefname{claim}{Claim}{Claims}
\Crefname{subclaim}{Subclaim}{Subclaims}

\usepackage{epsfig}
\usepackage{amsthm,amssymb}
\usepackage{amsthm,amssymb}

\usepackage{mathrsfs}
\usepackage{xspace}
\usepackage{soul}
\usepackage{latexsym}
\usepackage{bbm}

\usepackage{tikz}
\usepackage{pgfplots}
\pgfplotsset{compat=1.18}
\usetikzlibrary{arrows.meta}
\usepackage{sidecap}

\usepackage{framed}

\usepackage{xcolor}
\definecolor{DarkGray}{rgb}{0.66, 0.66, 0.66}
\definecolor{DarkPowderBlue}{rgb}{0.0, 0.2, 0.6}
\definecolor{fluorescentyellow}{rgb}{0.8, 1.0, 0.0}

\usepackage[ruled,vlined,algonl]{algorithm2e}
\SetEndCharOfAlgoLine{}
\SetKwComment{Comment}{\footnotesize$\triangleright$\ }{}

\SetCommentSty{mycommfont}

\usepackage{thmtools,thm-restate}

\allowdisplaybreaks

\newcounter{note}[section]

\sethlcolor{fluorescentyellow}

\newcommand{\initOneLiners}{%
    \setlength{\itemsep}{0pt}
    \setlength{\parsep }{0pt}
    \setlength{\topsep }{0pt}
}

\newtheorem{theorem}{Theorem}[section]
\newtheorem{lemma}[theorem]{Lemma}
\newtheorem{claim}[theorem]{Claim}

\newtheorem{corollary}[theorem]{Corollary}
\newtheorem{assumption}[theorem]{Assumption}

\newtheorem{proposition}[theorem]{Proposition}

\theoremstyle{definition}
\newtheorem{definition}[theorem]{Definition}

\theoremstyle{remark}
\newtheorem{remark}[theorem]{Remark}

\makeatletter 
\renewcommand{\theinvariant}{(I\@arabic\c@invariant)}
\makeatother

\newcommand{\eps}{\varepsilon}

\newcommand{\junk}[1]{}
\newcommand{\eat}[1]{}

\newif\ifhideproofs

\ifhideproofs
\usepackage{environ}
\NewEnviron{hide}{}

\fi


%% file: Intro.tex
Learning parametric distribution families over $\mathbb{R}^d$ lies at the core of unsupervised learning and statistical inference. It underpins a broad range of applications, including clustering, density estimation \cite{DevroyeLugosi2001}, anomaly detection \cite{nassif2021machine}, and generative modeling \cite{mitzenmacher2004brief}. A central question in this context is whether a given distribution family $\mathcal{F}$ is learnable—and if so, what the associated sample complexity is. Traditionally, learnability has been studied within the PAC (Probably Approximately Correct) framework, where the goal is to learn, with probability at least $1 - \delta$, a distribution $f \in \mathcal{F}$ to within a target error $\epsilon$, using a number of samples depending on $\epsilon$, $\delta$, and the complexity of $\mathcal{F}$ \cite{valiant1984theory, DevroyeLugosi2001}. The choice of error metric varies across settings, with common choices including Total Variation (TV) distance \cite{Devroye2018TheTV}, Minkowski norms \cite{DevroyeLugosi2001}, and Wasserstein distances \cite{niles2022minimax}. An important line of work in this area focuses on \emph{efficient} PAC learnability—achieving learnability in polynomial time. See, for example, \cite{anderson2013efficient} for bounds on the efficient learnability of high-dimensional simplices, and more recently \cite{afzali2024mixtures, ashtiani2022private} for private and polynomial-time learning of mixtures of Gaussians. In contrast, the present paper studies \emph{information-theoretic learnability}, where no computational constraints are imposed and algorithms may run in exponential time.

Information-theoretic learnability is deeply rooted in statistical learning theory, aiming to derive tight sample complexity bounds. The existing literature ranges from classical approaches based on uniform convergence and Vapnik–Chervonenkis (VC) dimension—such as the ``minimum distance estimate'' in \cite{DevroyeLugosi2001}—to structural assumptions, such as low-dimensional manifold supports \cite{berenfeld2021density}. Several works employ concepts from Shannon entropy and mutual information \cite{bresler2015efficiently,najafi2020reliable}, while others introduce specialized techniques such as Fourier-based methods for learning mixtures \cite{qiao2022fourier}. In particular, \cite{bresler2015efficiently} derives tight samples complexity bounds for learning high-dimensional Ising models, while a similar mutual information-based approach was then used by \cite{najafi2020reliable} to establish sample complexity bounds for learning $k$-mixtures of high-dimensional Bernoulli models. Recent work also intersects with contemporary concerns such as privacy-preserving (differentially private) learning \cite{arbas2023polynomial, bun2024not}. Concurrently, several lines of research study \emph{impossibility} results, such as minimax lower bounds for nonparametric learning \cite{bilodeau2023minimax, wang2022minimax, devroye2020minimax}, and their partial extensions to adversarial losses \cite{tang2023minimax}.

A unifying theme across all these approaches is the identification/assumption of ``structural properties'' that control the complexity of $\mathcal{F}$ and enable learnability. One such property is \emph{sample compressibility}, originally introduced by Littlestone and Warmuth \cite{littlestone1986relating} in the context of supervised learning and recently adapted for distribution learning by Ashtiani et al. \cite{ashtiani2018nearly} (see Definition~\ref{sample_compression}). Sample compressibility characterizes a family $\mathcal{F}$ by the existence of an encoder that can, with high probability, compress $n \gg 1$ i.i.d. samples from any $f \in \mathcal{F}$ into a much smaller representative sample set of size $\tau$ and $t$ additional bits (both typically $\mathcal{O}(1)$ with respect to $\epsilon$ and $\delta$), sufficient to recover the essential structure of $f$ in a PAC sense. This notion has strong theoretical implications: it implies PAC learnability and yields near-optimal sample complexity bounds for a wide range of families \cite{ashtiani2018nearly, ben2023private, saberi2023sample}. 

For example, Ashtiani et al.\ employed their sample compression scheme to resolve a long-standing open problem by establishing nearly tight sample complexity bounds of $\widetilde{\mathcal{O}}(kd^2/\epsilon^2)$ for learning arbitrary $d$-dimensional Gaussian mixture models with $k$ components, up to a total variation error of $\epsilon$. Subsequent works \cite{najafi2021statistical, saberi2023sample} extended this approach to derive tight sample complexity bounds for the class of uniform distributions supported on high-dimensional simplices, under both clean and noisy observations. In particular, \cite{najafi2021statistical} proved that simplices in $\mathbb{R}^d$ can be learned from $n \ge \mathcal{O}(d^2 \log d / \epsilon)$ samples up to a total variation error of $\epsilon$, and that this rate is tight. Later, Saberi et al.\ \cite{saberi2023sample} showed—via a Fourier-based technique—that when samples are corrupted by isotropic Gaussian noise with variance $\sigma^2$, the required sample complexity increases by a factor of $\epsilon^{-1} e^{\mathcal{O}(d\sigma^2)}$. See also \cite{ben2023private} for recent progress in privately learning distributions from public data using sample compression.

Conversely, it is straightforward to show that every PAC-learnable distribution family admits a (possibly inefficient) sample compression scheme via a naive encoding argument. However, such schemes typically incur suboptimal sample complexity. This has led to the conjecture that every PAC-learnable family admits a \emph{matching} sample compression scheme—that is, one achieving optimal sample complexity—thus establishing an equivalence between the two notions. Sample compression has proven particularly powerful in cases that were previously intractable using classical approaches, such as extending learnability from base models to $k$-mixtures or to product distributions in high-dimensional settings. In this work, we show that a similar phenomenon arises in the context of learning from corrupted samples.

\textbf{Learning from corrupted samples.} In practical scenarios, data is often subject to noise or adversarial corruption, making it essential to understand how such perturbations affect learnability \cite{sinha2018certifying}. In particular, studies on adversarial perturbations have demonstrated that even small corruptions can significantly degrade the learnability of high-dimensional distributions \cite{mahloujifar2019curse}. Theoretical investigations of learning under adversarial conditions have made substantial progress in supervised learning (see \cite{konstantinov2020sample} and references therein), including several impossibility results \cite{charikar2017learning, kearns1988learning}.

In the unsupervised setting, existing results have primarily focused on nonparametric distribution learning \cite{tang2023minimax, wang2022minimax}, often yielding loose minimax bounds due to the generality of the nonparametric assumptions, which fail to exploit the structural constraints typical in parametric families. Other works target highly specific cases, such as \cite{saberi2023sample}, and do not generalize broadly. While significant advances have been made in understanding sample compressibility and learnability under ideal (noise-free) conditions, relatively little is known about how \emph{perturbed data} influences the learnability of parametric distribution families. To the best of our knowledge, providing general and rigorous theoretical guarantees for learning parametric families under noise or adversarial corruption remains a critical and largely open problem. In this work, we aim to fill this gap via establishing learnability guarantees for general sample-compressible families.

\subsection{Summary of Results}

We prove that sample-compressible distribution families remain learnable even when samples are perturbed, albeit with appropriately inflated sample complexity. Informally, suppose $\boldsymbol{X}_1,\ldots,\boldsymbol{X}_n\in\reals^d$ are i.i.d.\ samples from some $f\in\mathcal{F}$, for a general sample-compressible class $\mathcal{F}$. However, we have access only to perturbed samples
$
\widetilde{\boldX_i} = \boldX_i + \boldsymbol{\zeta}_i,
$
where we consider two models for the perturbation vectors $\boldsymbol{\zeta}_i$:
\begin{itemize}
    \item 
    \textbf{Stochastic noise model}: each $\boldsymbol{\zeta}_i$ is an independent sample from a fixed noise distribution (e.g., Gaussian or Laplace noise) with per-coordinate variance of at most $\sigma^2$.
    \item 
    \textbf{Adversarial corruption model}: an unknown subset of at most $s<n$ samples are arbitrarily perturbed by an adversary, under the constraint $\|\boldsymbol{\zeta}_i\|_{\infty} \leq C$ for some budget $C$, while for uncorrupted samples we have $\boldsymbol{\zeta}_i = 0$.
\end{itemize}
Throughout, our goal is to provide results that are as general as possible, making minimal assumptions on the distribution class $\mathcal{F}$, the ambient dimension $d$, or the statistical noise models described in Section \ref{sec:Gaussian}. Our main results establish bounds on the learning error measured both in $\ell_2$-norm and in total variation distance.
\\[1.5mm]
\textbf{Additive Noise Model}:
In Section \ref{sec:Gaussian}, via Proposition \ref{proposition:SampComp:NoisyF} and our main Theorem \ref{main_theorem}, we show that $\mathcal{F}$ is PAC-learnable from noisy samples within $\ell_2$ error $\epsilon$, provided that
    $$
    n \geq N_{\mathrm{clean}}(\epsilon,\delta) + \widetilde{\mathcal{O}}\left({d\tau}/{\epsilon^2}\right)\cdot\log\left(1+\sigma\right),
    $$
    where $N_{\mathrm{clean}}(\epsilon,\delta)$ is the sample complexity for learning $\mathcal{F}$ from clean samples, and $\tau$ corresponds to the sample compression scheme of $\mathcal{F}$ (see Definition \ref{sample_compression}). 
    Here, $\widetilde{\mathcal{O}}(\cdot)$ hides possible polylogarithmic dependencies on $\epsilon,\delta$. Proposition \ref{proposition:WaterFilling} and Corollary \ref{corl:noise:TVfromL2} further establish learnability guarantees under total variation distance. The latter results are general and could be of independent interest for other applications. 

We introduce two sufficient conditions (Assumptions \ref{assumption_decoder} and \ref{class_assumption_2}), which we show are also minimax necessary via Claims \ref{Impossibility_Decoder_Assumption} and \ref{impossibility_assumption2_example}. We also verify that these assumptions hold for several important distribution families, via a series of Claims \ref{claim_lipschitz_gaussian}, \ref{claim_lipschitz_uniform}, \ref{claim_lipschitz_mixture}, \ref{claim:lowFreq:kMixtureUniform}, and \ref{claim:lowFreq:UnifdDim}. Our proofs are based on a novel \emph{perturbation-quantization} technique, which naturally interfaces with the sample compression framework.
\\[1.5mm]
\textbf{Adversarial Perturbations}:
Our method extends to the case of adversarial perturbations in Section \ref{sec:adversarial}. Via Proposition \ref{adversarial_sample_decoder} and Theorem \ref{adversarial_main_theorem} (our main adversarial result), we prove TV-learnability of $\mathcal{F}$ as long as
$$
n \geq N_{\mathrm{clean}}(\epsilon,\delta) + \widetilde{\mathcal{O}}\left({(\tau s+ts+ds^2)}/{\epsilon^2}\right)\cdot\log (1+C).
$$
Notably, we do not need Assumption \ref{class_assumption_2} in this setting. 
\\[1.5mm]
\textbf{Theoretical Examples}:
In Section \ref{sec:examples}, we use our results to derive new sample complexity bounds for two theoretical examples. We first consider $k$-mixtures of uniform distributions over axis-aligned hyperrectangles in $\reals^d$. This family is particularly relevant for approximating smooth distributions via piecewise-constant functions with a bounded number of level sets \cite{browne2011model}, yet its sample complexity under noise or adversarial corruption has remained open \cite{najafi2021statistical,brunot2019gaussian}. We establish PAC-learnability of this family under both noisy and adversarial perturbations for the first time.

We also consider the recovery of Gaussian mixture models (GMM) from adversarially corrupted samples—another problem with previously unresolved sample complexity. In particular, we show that one can recover a $k$-mixture provided that
$$
n \geq \widetilde{O}\left(\frac{skd^2}{\epsilon^2}\right) \log\frac{1}{\delta}
+
\widetilde{O}\left(\frac{ds^2}{\epsilon^2}\right) \log\left(1+C\delta^{-1}\right).
$$


The paper is structured as follows: We begin by reviewing sample compression in Section \ref{sec:sampleCompression}, followed by the formal problem setup in Sections \ref{sec:notations} and \ref{sec:problemDef}. Our main techniques and results are presented in Sections \ref{sec:Gaussian}, \ref{sec:adversarial}, and \ref{sec:examples}. Finally, Section \ref{sec:conclusions} concludes with a discussion of new open problems and future research directions.


\subsection{Preliminaries: Sample Compression}
\label{sec:sampleCompression}

Following \cite{ashtiani2018nearly}, we define the distribution decoder and sample compression scheme as follows. Let $\mathcal{X}$ be a measurable space, typically $\mathcal{X} \subseteq \mathbb{R}^d$ for some dimension $d \in \mathbb{N}$, and let $\mathcal{F}$ be a class of distributions supported on $\mathcal{X}$. The \emph{decoders} for $\mathcal{F}$ are defined as:

\begin{definition}[Distribution Decoder]
A distribution decoder for $\mathcal{F}$ is a deterministic function $\mathcal{J}$ that takes a finite sequence of elements of $\mathcal{X}$ and a finite sequence of bits as input, and outputs an element of $\mathcal{F}$. Specifically, 
$$
\mathcal{J}: \bigcup_{i=1}^{\infty} \mathcal{X}^i \times \bigcup_{i=1}^{\infty} \{0, 1\}^i \to \mathcal{F}.
$$
\end{definition}

\begin{definition}[Sample Compression]
\label{sample_compression}
Let $\tau, t, m: (0,1) \to \mathbb{Z}_{\geq 0}$ be functions. A class $\mathcal{F}$ is said to admit $(\tau, t, m)$-sample compression (s.c.) if there exists a decoder $\mathcal{J}$ for $\mathcal{F}$ such that for any distribution $f \in \mathcal{F}$, the following holds:
For any $\epsilon, \delta \in (0, 1)$, if an i.i.d. sample set $S$ of size $n \ge m(\epsilon) \log\left(\frac{1}{\delta}\right)$ is drawn from $f$, then with probability at least $1 - \delta$, there exists a sequence $\mathbf{L}$ of at most $\tau(\epsilon)$ elements of $S$ and a sequence $\mathbf{B}$ of at most $t(\epsilon)$ bits, such that 
$$
\left\Vert \mathcal{J}(\mathbf{L}, \mathbf{B}) - f \right\Vert_{\mathrm{TV}} \leq \epsilon.
$$
\end{definition}
Here, $\left\Vert \cdot \right\Vert_{\mathrm{TV}}$ denotes the total variation distance, as defined in Section \ref{sec:notations}. For example, the class of Gaussian distributions $\mathcal{F} = \left\{\mathcal{N}(\mu, \sigma^2) \mid \mu \in \mathbb{R}, \sigma > 0 \right\}$ admits a $(2, 0, \mathcal{O}(\frac{1}{\epsilon} \log \frac{1}{\epsilon}))$-s.c. for any $\epsilon, \delta \in (0, 1)$. Specifically, given $\mathcal{O}(\frac{1}{\epsilon}\log\frac{1}{\epsilon} \log \frac{1}{\delta})$ i.i.d. samples from any $f \in \mathcal{F}$ (with mean $\mu$ and standard deviation $\sigma$), one can guarantee that with probability at least $1 - \delta$, two distinct samples will fall within an $\epsilon$-neighborhood of both $\mu - \sigma$ and $\mu + \sigma$. This results in a simple scheme (see Figure \ref{fig:SCofGaussian}) that $\epsilon$-estimates $f^*$ according to TV error. An interesting feature of sample compression is that we do not need to explicitly identify which samples, when fed to the decoder, recover $f$. The key requirement is that the sequences of $\tau(\epsilon)$ samples and $t(\epsilon)$ bits only \emph{exist}, without needing to specify an algorithm for finding them.

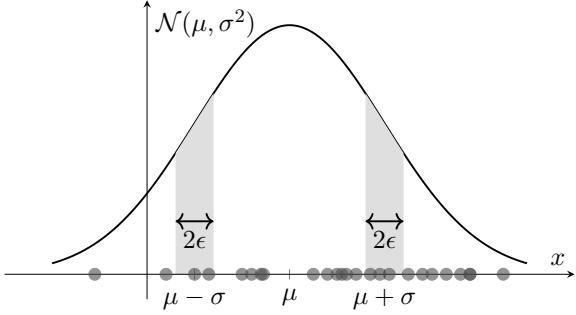
\begin{SCfigure}[1.2][t] 
\label{fig:SCofGaussian}
\centering
\begin{tikzpicture}[scale=0.9]

\def\muval{3}
\def\sigmaval{2}
\def\eps{0.4}

\begin{axis}[
    domain= -2:8,
    samples=200,
    axis lines=middle,
    xlabel={$x$},
    ylabel={$\mathcal{N}(\mu,\sigma^2)$},
    height=6cm,
    width=10cm,
    xtick={\muval-\sigmaval, \muval, \muval+\sigmaval},
    xticklabels={$\mu-\sigma$, $\mu$, $\mu+\sigma$},
    ytick=\empty,
    enlargelimits,
    clip=false,
    axis on top,
]

\addplot[thick,black] {1/(2*\sigmaval*sqrt(pi))*exp(-((x-\muval)^2)/(2*\sigmaval*\sigmaval))};

\addplot[
    draw=none,
    fill=gray!25,
    domain={\muval-\sigmaval-\eps}:{\muval-\sigmaval+\eps},
]
{1/(2*\sigmaval*sqrt(pi))*exp(-((x-\muval)^2)/(2*\sigmaval*\sigmaval))} \closedcycle;

\addplot[
    draw=none,
    fill=gray!25,
    domain={\muval+\sigmaval-\eps}:{\muval+\sigmaval+\eps},
]
{1/(2*\sigmaval*sqrt(pi))*exp(-((x-\muval)^2)/(2*\sigmaval*\sigmaval))} \closedcycle;

\draw[<->, thick]
    (axis cs:\muval-\sigmaval-\eps,0.03) -- (axis cs:\muval-\sigmaval+\eps,0.03)
    node[midway, below] {$2\epsilon$};

\draw[<->, thick]
    (axis cs:\muval+\sigmaval-\eps,0.03) -- (axis cs:\muval+\sigmaval+\eps,0.03)
    node[midway, below] {$2\epsilon$};

\addplot[
    only marks,
    mark=*,
    mark size=2.5pt,
    black!70,
    opacity=0.6,
]
coordinates {
    (-1.1,0) (0.4,0) (1.0,0) (1.3,0) (2,0)
    (2.2,0) (2.4,0) (2.45,0) (3.5,0) (3.8,0)
    (4.0,0) (4.1,0) (4.2,0) (4.4,0) (4.7,0)
    (4.9,0) (5.1,0) (5.5,0) (5.8,0) (6.0,0)
    (6.3,0) (6.6,0) (6.8,0) (6.8,0) (7.5,0)
};

\end{axis}

\end{tikzpicture}

\caption{Depiction of the $(2, 0, \mathcal{O}(\tfrac{1}{\epsilon} \log \tfrac{1}{\epsilon}))$-sample compression scheme for $\mathcal{N}(\mu, \sigma^2)$ over $\mathbb{R}$. Given $\mathcal{O}(\tfrac{1}{\epsilon} \log \tfrac{1}{\epsilon}\log\frac{1}{\delta})$ i.i.d. samples from this Gaussian (shown as grey dots), with probability at least $1 - \delta$, there exist two distinct samples, $X_i$ and $X_j$, falling within $\epsilon$-neighborhoods of $\mu - \sigma$ and $\mu + \sigma$, respectively. A simple decoder, defined as $\mathcal{N}\left({(X_i + X_j)}/{2}, {(X_j - X_i)^2}/{4}\right)$, then reconstructs an estimate of the distribution with TV error $\mathcal{O}(\epsilon)$. Importantly, the indices $i$ and $j$ need not be known — it suffices that such samples exist within the dataset.
}
\end{SCfigure}

Many naturally occurring and practical distribution families are known to admit \emph{efficient} sample compression schemes. In a weaker sense, any PAC-learnable class of distributions is sample-compressible by at least a naive scheme. Conversely, \cite{ashtiani2018nearly} demonstrated that sample compressibility guarantees PAC-learnability (albeit potentially through an exponential-time algorithm). The core idea behind their proof is simple: we divide the training dataset into two non-overlapping partitions. Since we need to determine a sequence of at most $\tau(\epsilon)$ elements out of $m(\epsilon) \log \frac{1}{\delta}$ samples (where repetition is allowed and order matters) and at most $t(\epsilon)$ bits, using the decoder $\mathcal{J}$ and the first partition, we can generate a finite set of at most $(m(\epsilon) \log \frac{1}{\delta})^{\tau(\epsilon)} 2^{t(\epsilon)}$ candidate distributions, with at least one of them guaranteed to be close to the true distribution $f$. We then apply a multiple hypothesis testing procedure with the second partition to approximate $f$. Specifically, we use Theorem \ref{existence_of_algorithm_with_samples} (from \cite{ashtiani2018nearly}, originally from \cite{DevroyeLugosi2001}) to guarantee that a multiple hypothesis testing procedure using enough samples will recover $f$ up to a small error.

\begin{theorem}[Theorem 3.4 of \cite{ashtiani2018nearly}]
\label{existence_of_algorithm_with_samples}
There exists a deterministic algorithm that, given candidate distributions $f_1, \ldots, f_M$, parameters $\epsilon,\delta > 0$, and ${\log \left( {M^2}/{\delta} \right)}/({2 \epsilon^2})$ i.i.d. samples from an unknown distribution $g$ (not necessarily in $\mathcal{F}$), outputs an index $j \in [M]$ such that 
$$
\left\Vert f_j - g \right\Vert_1 \leq \min_{i \in [M]} \left\Vert f_i - g \right\Vert_1 + 4\epsilon,
$$
with probability at least $1 - \delta$.
\end{theorem}
\begin{corollary}[Theorem 3.5 of \cite{ashtiani2018nearly}]
\label{thm:intro:sample_complexity_sc}
Suppose $\mathcal{F}$ admits $(\tau, t, m)$-sample compression for some functions $\tau,t,m: (0,1) \to \mathbb{N}$. Then, there exists a deterministic algorithm that, for any $\epsilon \in (0,1)$, given at least 
\begin{align}
\label{eq:corollarySC:NCleanFormula}
n \ge N^{\mathsf{Clean}}_{\tau,t,m} (\epsilon, \delta) \triangleq \widetilde{\mathcal{O}}\left( m\left( {\epsilon}/{6} \right) + \frac{\tau \left( {\epsilon}/{6} \right) + t \left( {\epsilon}/{6} \right)}{\epsilon^2}
\right)
\end{align}
i.i.d. samples from any unknown distribution $g \in \mathcal{F}$, outputs $\widehat{f} \in \mathcal{F}$ such that $\left\Vert f - g \right\Vert_{\mathrm{TV}} \leq \epsilon$ with high probability.
\end{corollary}
Corollary \ref{thm:intro:sample_complexity_sc} guarantees (information-theoretic) PAC-learnability of $\mathcal{F}$ given the existence of a sample compression scheme. The full version with complete poly-logarithmic dependencies is presented as Theorem \ref{sample_complexity_sc}. 
However, as noted, the sample complexity bound in \eqref{eq:corollarySC:NCleanFormula} might not be tight compared to other methods for proving PAC-learnability. A conjecture suggests that for every learnable class, there exists a specific sample compression scheme that closely corresponds to its sample complexity, as discussed in Corollary \ref{thm:intro:sample_complexity_sc}. If this conjecture holds, it would imply that \emph{every} information-theoretically learnable class can be approached through its sample compression scheme, without loss of generality.





%% file: Notations.tex
For $n\in\mathbb{N}$, we show the set $\left\{1,\ldots,n\right\}$ via $[n]$. Throughout the paper, vectors are shown by bold letters (e.g., $\boldsymbol{X}$ or $\boldsymbol{w}$), while scalars are shown by ordinary letters (e.g., $a$ or $T$). Assume a measurable space $\mathcal{X}$ and a corresponding $\sigma$-algebra $\mathcal{B}$. We usually assume $\mathcal{X}\subseteq\reals^d$.
For $p\ge 1$, two probability measures $P_1,P_2$ supported over $\mathcal{X}$  with respective density functions (with respect to Lebesgue measure) $f_1,f_2\in L^p(\mathcal{X})$, the $\ell_p$-distance is defined as
\begin{align}
\left\Vert f_1-f_2\right\Vert_p
\triangleq
\left(
\int_{\mathcal{X}}\left\vert f_1(\boldsymbol{x})-f_2(\boldsymbol{x})\right\vert^p\mathrm{d}\boldsymbol{x}
\right)^{\frac{1}{p}}.
\end{align}
In this work, we only use $p=1,2$. Total Variation (TV) distance is defined as
\begin{align}
\mathsf{TV}(f_1, f_2) 
\triangleq 
\sup_{B\in\mathcal{B}}~\left\vert P_1(B)-P_2(B) \right\vert
= 
\frac{1}{2}\left\Vert f_1 - f_2\right\Vert_1.
\end{align}
Also, the KL divergence between $f_1$ and $f_2$ is defined as 
$
\mathsf{KL}\left(f_1 \Vert f_2\right) \triangleq
\mathbb{E}_{P_1}\left[
\log\left({f_1(\boldsymbol{x})}/{f_2(\boldsymbol{x})}
\right)\right]$. We have $\mathsf{KL}\left(f_1 \Vert f_2\right)=\infty$ if $f_1$ is not absolutely continuous w.r.t. $f_2$. For a function $g:\reals^d\to\reals$ with $g\in L^2(\reals^d)$, its \emph{Fourier} transform $\mathsf{F}\left\{g(\boldsymbol{x})\right\}\left(\boldsymbol{w}\right)$ for $\boldsymbol{w}\in\mathbb{R}^d$ is defined as follows:
\begin{equation}
\mathsf{F}\{g(\boldsymbol{x})\}(\boldsymbol{w}) = G(\boldsymbol{w}) 
\triangleq 
\int_{\reals^d} g(\boldsymbol{x})e^{-i\boldsymbol{x}.\boldsymbol{w}}\mathrm{d}\boldsymbol{x},
\end{equation}
where we usually refer to $\boldsymbol{w}$ as the \emph{frequency vector}. 



%% file: Problem_definition.tex
Let us formally define our problem. Consider a class of probability distributions $\mathcal{F}$ supported over a measurable space $\mathcal{X}$, where $\mathcal{X} \subseteq \mathbb{R}^d$ for a given dimension $d \in \mathbb{N}$. For simplicity, we assume that all members of $\mathcal{F}$ admit densities. Consequently, without loss of generality, we interpret $f \in \mathcal{F}$ as a probability density function.

\begin{assumption}[$\mathcal{F}\subseteq L^2(\mathcal{X})$]
\label{assumption:L2(X)}
For $\forall f\in\mathcal{F}$, we assume $f\in L^2(\mathcal{X})$ with respect to Lebesgue measure. Therefore, the Fourier transform  $\mathsf{F}(f)(\boldsymbol{\omega})$ exists for all $f$.
\end{assumption}

\begin{assumption}[Sample Compressibility]
\label{assumption:SC}
We assume that $\mathcal{F}$ admits $(\tau(\epsilon), t(\epsilon), m(\epsilon))$-sample compression for some not necessarily known functions $\tau, t, m: (0,1) \to \mathbb{N}$.    
\end{assumption}

Given an unknown distribution $f^* \in \mathcal{F}$, let
$
\boldX_1, \dots, \boldX_n \widesim[2.5]{i.i.d.} f^*,
$
be a sample set of size $n$ drawn from $f^*$. Our goal is to approximate $f^*$ given $\mathcal{F}$ and a set of \emph{perturbed} samples $\widetilde{\boldX}_1, \dots, \widetilde{\boldX}_n$, obtained as:
$$
\widetilde{\boldX}_i = \boldX_i + \boldsymbol{\zeta}_i, \quad \forall i \in [n],
$$
where $\boldsymbol{\zeta}_i$s are perturbation vectors. We consider two general perturbation models:
\\[1mm]
\noindent
{\bf {Additive Noise Model}}:
Let $ G $ be a symmetric $ d $-dimensional noise distribution with i.i.d. components, such as (but not limited to) a Gaussian distribution $ \mathcal{N}(\boldsymbol{0}, \sigma^2\boldsymbol{I}_d) $ for some $ \sigma > 0 $. The assumptions of symmetry and independence across dimensions are made for simplicity in the final formulations and can be relaxed straightforwardly. Suppose $\boldsymbol{\zeta}_1, \dots, \boldsymbol{\zeta}_n$ are i.i.d. samples from $G$, independent of the original samples $\boldX_1, \dots, \boldX_n$.
Equivalently, one can assume $\widetilde{\boldX}_1,\ldots,\widetilde{\boldX}_n$ as i.i.d. samples from $f^**G$, where $*$ denotes the convolution operator. Section \ref{sec:Gaussian} presents information-theoretic learnability guarantees for this scenario.
\\[1mm]
\noindent
{\bf {Adversarial Perturbations}}:
In this setting, an adversary can arbitrarily modify a subset of size $\leq s$ (for $s<n$) out of the samples $\boldX_1, \dots, \boldX_n$. Mathematically speaking, assume that the adversary designs perturbation vectors $\boldsymbol{\zeta}_1, \dots, \boldsymbol{\zeta}_n$ with full knowledge of: i) The class $\mathcal{F}$, ii) The true distribution $f^*$, and iii) The original sample set $\boldX_1, \dots, \boldX_n$. The adversary can design the vectors $\boldsymbol{\zeta}_i$ arbitrarily, subject to the constraints that only $s$ out of the $n$ perturbation vectors are nonzero, and $\left\Vert\boldsymbol{\zeta}_i\right\Vert_{\infty}\leq C$, for some budget $C\ge0$. Results for this scenario are presented in Section \ref{sec:adversarial}.

%% file: Gaussian.tex
\renewcommand{\labelenumi}{\roman{enumi})}

In this section, we analyze the scenario in which the samples $\boldsymbol{X}_1, \ldots, \boldsymbol{X}_n$ are corrupted by additive noise with a density $G \in L^2(\mathcal{X})$, as described in Section~\ref{sec:problemDef}. Our main results are stated under the assumption that $G$ has independent and symmetrically distributed coordinates. Nevertheless, as we later show in the proof of Theorem~\ref{main_theorem}, these results extend to a broad class of noise distributions, provided a certain Fourier-based condition is satisfied. We restrict our statements to the independent and symmetric case to avoid excessively complicated formulations, as the general setting introduces substantial technical overhead. In Corollary~\ref{corl:GaussianLaplaceMainThm}, we also instantiate our results for two practically important noise models: Gaussian and Laplace. To establish our results, we introduce structural assumptions on the distribution class $\mathcal{F}$—namely, Assumptions~\ref{assumption_decoder} and~\ref{class_assumption_2}. These assumptions are both \emph{necessary} (in a minimax sense)\footnote{By “minimax necessary,” we mean that there exist settings where the failure of at least one of these assumptions renders it provably impossible to approximate $f^*$ reliably.} and \emph{sufficient} for learning $\mathcal{F}$ from corrupted samples.

Our approach proceeds as follows. We begin by discussing Assumption \ref{assumption_decoder} in Section \ref{sec:subsec:Lipschitz}, and then establish the sample compressibility of the class $\mathcal{F} * G = \{ f * G \mid f \in \mathcal{F} \}$ in Section~\ref{sec:subsec:SampCompNoisyF}. Leveraging existing tools from learning theory, we then show that one can recover $f^* * G$ up to a controlled total variation error (Proposition~\ref{proposition:SampComp:NoisyF}). Next, in Section \ref{sec:noisy:mainResult}, we demonstrate that Assumption \ref{class_assumption_2} is both sufficient and minimax necessary for recovering $f^*$ from $f^* * G$ in the $\ell_2$ norm, as stated in Theorem \ref{main_theorem}. Finally, Section \ref{sec:subsec:L2toTV} extends these results to provide similar guarantees in total variation distance.

Before going into details, let us briefly discuss our assumptions. Assumption \ref{assumption_decoder} complements Assumption \ref{assumption:SC} (sample compressibility) by adding a stability condition. Specifically, it requires that $\mathcal{F}$ be not only sample-compressible but also \emph{stably} so. This means that infinitesimally small perturbations in the input samples should not cause the decoder to produce drastically different distributions, at least in a neighborhood of \emph{good} samples. We refer to this property as \emph{Local Lipschitz Decodability}. See Claim \ref{Impossibility_Decoder_Assumption} to see its minimax necessity. On the other hand, Assumption \ref{class_assumption_2} states that the probability density functions (pdfs) in $\mathcal{F}$ should not exhibit excessive fluctuations and must adhere to a certain degree of low-frequency behavior. Our impossibility results (e.g., Claim \ref{impossibility_assumption2_example}) theoretically show that functions with pronounced high-frequency components become increasingly difficult to recover from additive noise.

\subsection{Local Lipschitz Decodability}
\label{sec:subsec:Lipschitz}

As discussed earlier, to establish the main result of this section, we first introduce the following assumption.

\begin{assumption}[Local Lipschitz Decodability]
\label{assumption_decoder}
For any $f\in\mathcal{F}$, assume we have $n$ i.i.d. samples from $f$ in a dataset $S$. Then, Assumption \ref{assumption:SC} implies that for at least one decoder $\mathcal{J}$, for any $\delta,\epsilon\in(0,1)$, and having $n\ge m(\epsilon)\log\frac{1}{\delta}$, there exists a sequence of at most $t(\epsilon)$ bits denoted by $\mathbf{B}$, and a random sequence of $\tau(\epsilon)$ samples from $S$, vectorized as $\mathbf{L}\in\reals^{d\tau(\epsilon)}$, such that we have
$
\mathbb{P}\left(\mathsf{TV}\left(f,\mathcal{J}(\mathbf{L}, \mathbf{B})\right)
\leq\epsilon
\right)\ge1-\delta.
$
Additionally, assume that for any two sequences $\mathbf{B},\mathbf{L}$ with $\mathsf{TV}\left(f,\mathcal{J}(\mathbf{L}, \mathbf{B})\right)
\leq 1/2$, $\mathcal{J}$ behaves smoothly w.r.t. $\mathbf{L}$, i.e., there exists $r\ge0$ such that any perturbed copy of $\mathbf{L}$, denoted by $\mathbf{L}'\in\reals^{d\tau(\epsilon)}$, satisfies
\begin{equation}
\mathsf{TV}
\left(
\mathcal{J}(\mathbf{L}, \mathbf{B}), 
\mathcal{J}(\mathbf{L}', \mathbf{B})
\right) 
\leq
\frac{r}{2} \left\Vert \mathbf{L} - \mathbf{L}'\right\Vert_2.
\end{equation}
\end{assumption}

Assumption \ref{assumption_decoder} restricts the class $\mathcal{F}$ to include at least one decoder that behaves \emph{smoothly} with respect to very small changes in the input sample sequence $\mathbf{L}$. This assumption is crucial not only from a mathematical standpoint but also in practical applications, where data samples are stored digitally and quantized using a finite number of bits. Decoders that are excessively sensitive to small perturbations are therefore impractical. The term ``local" indicates that the Lipschitz property only needs to hold within a loose neighborhood around \emph{good} samples $\mathbf{L}$, specifically those satisfying $\mathsf{TV}\left(f, \mathcal{J}(\mathbf{L}, \mathbf{B})\right) \leq 1/2$. 

Next, we present a series of achievability results and one impossibility result related to Assumption \ref{assumption_decoder}. By examining several commonly used cases that adhere to this assumption, we demonstrate that it is both natural and not overly restrictive. Finally, in Claim \ref{Impossibility_Decoder_Assumption}, we establish a converse result: there exist scenarios where the assumption fails to hold, making learning provably impossible in such cases. 
The proofs for all the following claims is given in Appendix \ref{sec:app:proofs:claims:I}.

\begin{claim}[Lipschitz Decodability of Gaussians $\left\{\mathcal{N}(\boldsymbol{\mu},\sigma^2\boldsymbol{I})\vert~\boldsymbol{\mu}\in\reals^d,~\sigma\ge\sigma_0\right\}$ with $\sigma_0>0$]
\label{claim_lipschitz_gaussian}
Let $\mathcal{F} =\left\{\mathcal{N}(\boldsymbol{\mu},\sigma^2\boldsymbol{I})\vert~\boldsymbol{\mu}\in\reals^d,~\sigma\ge\sigma_0 > 0\right\}$ be a class of isotropic $d$-dimensional Gaussians with a component-wise variance of at least $\sigma_0^2$ for some $\sigma_0>0$. Then, $\mathcal{F}$ satisfies Assumption \ref{assumption_decoder} with a Lipschitz constant of 
$$r\leq \mathcal{O}\left(\frac{1}{\sigma_0 \sqrt{d \log (2d)}}\right).$$
\end{claim}
Claim \ref{claim_lipschitz_gaussian} can be extended to the general case of $\mathcal{N}(\boldsymbol{\mu}, \boldsymbol{\Sigma})$ with a positive definite $\boldsymbol{\Sigma} \succ 0$, provided that $\lambda_{\min}(\boldsymbol{\Sigma}) \geq \sigma_0 > 0$, where $\lambda_{\min}(\cdot)$ denotes the minimum eigenvalue. For brevity and readability of the proof in Section \ref{sec:app:proofs:claims:I}, we have stated the claims in their simpler form.

\begin{claim}[Lipschitz Decodability of Uniform Measures with Minimum Bandwidth $T>0$]
\label{claim_lipschitz_uniform}
For $T>0$, let 
$$
\mathcal{F} =\left\{\mathsf{Uniform}\left(\prod_{i}\left[a_i,b_i\right]\right)\Bigg\vert~a_i,b_i\in\reals,~b_i-a_i\ge T,~ \forall i\in [d]\right\}
$$
be the class of of uniform distributions over axis-aligned hyper-rectangles in $\reals^d$ with a minimum width-per-dimension of $T>0$. Then, $\mathcal{F}$ satisfies Assumption \ref{assumption_decoder} with a Lipschitz constant $r\leq\frac{8d}{T}$.
\end{claim}

The following claim is useful where learning a convex combination of a finite number of distributions from a given family—also known as a finite mixture—is required.

\begin{claim}[Conservation of Lipschitz Decodability under $k$-Mixture]
\label{claim_lipschitz_mixture}
Let $\mathcal{F}$ admit a sample compression scheme with at least one Lipschitz decoder according to Assumption \ref{assumption_decoder}, and a corresponding constant $r\ge0$. For $k\in\mathbb{N}$, consider the class of $k$-mixtures of $\mathcal{F}$ defined as 
$$
k\mathrm{-Mix}\left(\mathcal{F}\right) 
\triangleq 
\left \{ 
\sum_{i=1}^k \alpha_i f_i \Bigg\vert ~ f_i \in \mathcal{F},~ \alpha_i \geq 0~(\forall i\in[k]),~ \sum_{i=1}^k \alpha_i = 1 \right \}.
$$
Then, $k\mathrm{-Mix}\left(\mathcal{F}\right)$ is also sample compressible $($already proved by \cite{ashtiani2018nearly}$)$ and admits at least one locally Lipschitz decoder with a corresponding  constant of $r\sqrt{k}$.
\end{claim}
Next, we present an example of a sample-compressible family that lacks the local Lipschitz decodability property for any finite $r \geq 0$ and is provably unlearnable in the PAC sense.
\begin{claim} [Minimax Necessity of Assumption \ref{assumption_decoder}]
\label{Impossibility_Decoder_Assumption}
Suppose $\mathcal{F}$ is the class of distributions $ \{\mathcal{N}(\mu,\sigma^2)\vert ~\mu\in \reals, ~ \sigma > 0 \}$. This class admits $(2, 0, \widetilde{\mathcal{O}}({1}/{\epsilon}))$-s.c., but does not satisfy Assumption \ref{assumption_decoder} for any finite $r$. Let $G=\mathcal{N}(0,\sigma^2_0)$ for some $\sigma_0>0$, and define $\mathsf{A}(n,B)$ as the set of all decoders for $\mathcal{F}$ that take $n$ noisy samples and $B$ bits as input and output a corresponding $\widehat{f}\in\mathcal{F}$. Then, for any fixed $n,B\in\mathbb{N}$ 
and assuming $\mathbf{L}=\left\{\boldX_i\right\}_{i=1}^{n}\widesim[2.4]{i.i.d.} f^{*}*G$ for $f^*\in\mathcal{F}$, the following holds:
\begin{align}
\inf_{\mathscr{A}\in\mathsf{A}(n,B)}~
\sup_{f^*\in\mathcal{F}}~
\mathbb{P}
\left(
\min_{\mathbf{B}\in\left\{0,1\right\}^B}
\mathsf{TV}\left(
f^*,\mathscr{A}\left(\mathbf{L},\mathbf{B}\right)
\right)
\ge\frac{1}{200}
\right)
\ge\frac{1}{2}.
\end{align}
\end{claim}
The claim asserts that regardless of how large $ n $ and/or $ B $ are, no algorithm that receives $ n $ i.i.d. samples from $ f^* * G $ can learn $ \mathcal{F} $ in the TV sense. The proof relies on techniques from minimax theory, specifically Le Cam's method, to establish this impossibility result.

\begin{remark}  
Singular density functions with continuously varying degrees of freedom—such as a Gaussian distribution with an infinitesimally small variance and an arbitrary mean, or a uniform distribution with infinitesimally small bandwidth—cannot be learned in the TV error sense from samples corrupted by continuous noise, e.g., $\mathcal{N}(0, \sigma_0^2)$.  
\end{remark}  
While a formal proof of this claim is beyond the scope of this paper, it is still a meaningful and important observation. Intuitively, distributions that become (asymptotically) concentrated over a zero-measure region of the space, cannot be reliably learned (at least in TV error) in noisy regimes.


\subsection{Sample Compressibility of $\mathcal{F}*G$}
\label{sec:subsec:SampCompNoisyF}

We present the following proposition, which is a key component of our main result in Theorem \ref{main_theorem}. This proposition states that if a distribution class $\mathcal{F}$ admits sample compression over $\subseteq\reals^d$ with at least one decoder satisfying Assumption \ref{assumption_decoder}, then the noisy version $\mathcal{F} * G \triangleq \{ f * G \mid f \in \mathcal{F} \}$ is also sample compressible.

\begin{proposition}[Sample Compressibility of Noisy $\mathcal{F}$]
\label{proposition:SampComp:NoisyF}
For $d\in\mathbb{N}$, assume $\mathcal{F}$ be a class of $d$-dimensional distributions satisfying Assumptions \ref{assumption:L2(X)} and \ref{assumption:SC} for functions $\tau,t,m:(0,1)\to\mathbb{N}$, and let Assumption \ref{assumption_decoder} hold for a bounded constant $r\ge0$. Also, let $G$ be the density of an isotropic noise over $\reals^d$ with a component-wise CDF of $\Phi_G:\reals\to[0,1]$. Then, $\mathcal{F}*G\triangleq\{ f * G \mid f \in \mathcal{F} \}$ admits 
\begin{equation}
\left(
\tau\left(\tfrac{\epsilon}{2}\right),
t\left(\tfrac{\epsilon}{2}\right) +
d\tau\left(\tfrac{\epsilon}{2}\right)
\log_2\left(1+
\frac{r}{\epsilon}\sqrt{d\tau\left(\tfrac{\epsilon}{2}\right)}
\left|\Phi_G^{-1}\left(\frac{\delta}{4dm\left(\frac{\epsilon}{2}\right)\log\frac{2}{\delta}}\right)\right|
\right),
m\left(\tfrac{\epsilon}{2}\right)
\right)
\end{equation}
-sample compression, for any $\epsilon\in(0,1)$.
\label{sample_compression_theorem}
\end{proposition}
Full proof is presented in Appendix \ref{proof_of_sample_compression}. Here, we give a brief sketch of the proof.
\begin{proof}[Sketch of proof for Proposition \ref{sample_compression_theorem}]
Our methodology is based on \emph{denoising} the samples before applying the available decoder from $\mathcal{F}$. Specifically, we assume access to noisy samples $\widetilde{\boldX}_i \triangleq \boldX_i + \boldsymbol{\zeta}_i$, where $\boldsymbol{\zeta}_i$ are the noise vectors drawn from $G$. Our goal is to approximate each noise vector $\boldsymbol{\zeta}_i$ with a \emph{quantized} surrogate $\boldsymbol{\zeta}'_i$, ensuring that $\left\Vert \boldsymbol{\zeta}_i - \boldsymbol{\zeta}'_i \right\Vert_2 \leq \mathcal{O}(2^{-B})$ for a given $B \geq 1$. This way, we can use $\widetilde{\boldX}_i-\boldsymbol{\zeta}'_i$ as a partially denoised version of $\boldX_i$, with a small remaining residual noise due to non-ideal quantization. 

Since the values of $\boldsymbol{\zeta}_i$ are random and thus unknown, we approximate them by considering all high-probability values, which results in an overhead of $\mathcal{O}(B\times \tau(\epsilon))$ additional bits in the original sample compression scheme for $\mathcal{F}$. The small residuals can then be handles using local Lipschitz decodability assumption.  
\end{proof}

We also provide two explicit upper-bounds for the inverse CDF $ \Phi_G^{-1} $ of two practical cases of interest: Gaussian and Laplace noise.

\begin{remark}[Two Examples of $\Phi^{-1}_G$] 
\label{remark:NoiseCDFFormula}
For a product noise distribution $ G $, let $\Phi_G$ denote its component-wise CDF. Consider the following two examples: Let $ G \triangleq \mathcal{N}(\boldsymbol{0},\sigma^2\boldsymbol{I}_d) $ for some $\sigma > 0$, a common choice in many practical applications. Using Mill's ratio for the Gaussian distribution, we obtain:  
\begin{align}
\left\vert\Phi^{-1}_G(\Delta)\right\vert\leq \sigma\sqrt{2\log\left(\frac{1}{\sqrt{2\pi}\Delta}\right)},  
\quad\forall\Delta<1/\sqrt{2\pi}.
\end{align}  
Now, suppose each $ G_i \triangleq \mathrm{Laplace}(b) $ for some $ b > 0 $, a well-known choice, particularly in differential privacy research \cite{dwork2006differential,muthukrishnan2025differential}. Then, we have:  
\begin{align}
\left\vert\Phi^{-1}_G(\Delta)\right\vert= b\log\frac{1}{2\Delta},  
\quad\forall\Delta\leq1/2.
\end{align}  
\end{remark}

\subsection{$\ell_2$-Learnability of $\mathcal{F}$}
\label{sec:noisy:mainResult}

Proposition \ref{sample_compression_theorem} followed by Theorem \ref{thm:intro:sample_complexity_sc} ensures the \emph{PAC-learnability} of $f^* * G$ for any $f^*\in\mathcal{F}$. What remains is to prove the learnability of $f^*$ itself. First, let us state a seemingly counter-intuitive fact: When $G$ is a known noise pdf, and given mild identifiability conditions on $\mathcal{F}$ with respect to $G$, it follows that for any $f_1, f_2 \in \mathcal{F}$, the condition $\mathsf{TV}(f_1 * G, f_2 * G) = 0$ implies $f_1 = f_2$. However, the PAC-learnability of $f^* * G$ does not necessarily imply that $f^*$ can also be learned, at least in TV error. In other words, an \emph{absolute} zero TV distance between $f_1 * G$ and $f_2 * G$ is fundamentally different from a \emph{limiting} zero. Mathematically, there may exist a sequence $\left\{f_n\right\}_{n\in\mathbb{N}} \subset \mathcal{F}$ and a single density function $f^* \in \mathcal{F}$ such that
\begin{equation}
\lim_{n\to\infty}\mathsf{TV}\left(f_n * G, f^* * G\right) = 0,
\quad\text{but}\quad
\lim_{n\to\infty}\mathsf{TV}\left(f_n, f^*\right) > 0.
\label{eq:diffTVvsNoisyTVConverge}
\end{equation}
Before proving the existence of such sequences and similar to \cite{saberi2023sample}, let us first introduce a \emph{sufficient} condition on $\mathcal{F}$ that, as will become evident in Theorem \ref{main_theorem}, prevents this pathological phenomenon:
\begin{assumption}[Low-Frequency Property]
\label{class_assumption_2}
For a distribution family $\mathcal{F}$, assume there exist $\alpha \geq 0$ and $\xi < 1$ such that for each $p, q \in \mathcal{F}$, with respective Fourier transforms $P, Q$, we have
\begin{equation}
\frac{1}{(2\pi)^d}
\int_{\left\Vert \boldsymbol{w}\right\Vert_2\geq \alpha} \left\vert P(\boldsymbol{w})-Q(\boldsymbol{w})\right\vert^2 \mathrm{d}\boldsymbol{w}
\leq 
\xi 
\left\Vert
p-q
\right\Vert^2_2.
\end{equation}
More generally, $\xi$ can be a function of $\varepsilon\triangleq\left\Vert p-q\right\Vert_2$, even with $\lim_{\varepsilon\to 0^+}\xi(\varepsilon)=0$. However, we must have $\xi(\varepsilon)<1$ for all $\varepsilon>0$. Let $\mathsf{P}(\mathcal{F})$ denote the set of all $(\alpha,\xi)$ pairs that correspond to the above inequality for a class $\mathcal{F}$.
\end{assumption}
By Parseval's theorem, we have $\left\Vert p - q \right\Vert_2 = (2\pi)^{-d}\left\Vert P - Q \right\Vert_2$ for all $p, q \in \mathcal{F}$. However, Assumption \ref{class_assumption_2} restricts attention to the high-frequency energy component of $\left\Vert P - Q \right\Vert_2$, integrating only over the region $\left\Vert \boldsymbol{w} \right\Vert_2 \geq \alpha$. In return, the total energy is expected to decrease by a factor of $\xi < 1$, implying that a non-negligible fraction of the $\ell_2$-energy of $p - q$ is concentrated in lower frequencies. 
Several results in this section rely on this assumption to guarantee the recoverability of $f^*$ in terms of the $\ell_2$-norm or TV, provided that $f^* * G$ is learnable.

To demonstrate the minimax-necessity of Assumption \ref{class_assumption_2}, Claim \ref{impossibility_assumption2_example} provides an example of a distribution family that does not satisfy Assumption \ref{class_assumption_2} and show that the pathological phenomenon of \eqref{eq:diffTVvsNoisyTVConverge} can occur.
\begin{claim}[Minimax Necessity of Assumption \ref{class_assumption_2}]
\label{impossibility_assumption2_example}
Suppose the distribution class
$$
\mathcal{F}=
\left\{ f:x\to\frac{1 + (-1)^i \sin{kx}}{2\pi} \Big\vert~ k \in \mathbb{Z}_{\ge0}, ~ i\in \{0, 1\}\right\},
$$
for $x\in [0, 2\pi]$. Then, $\mathcal{F}$ does not satisfy Assumption \ref{class_assumption_2}. Also, there exists at least one sequence $\left\{f_n\right\}_{n\in\mathbb{N}}\subset\mathcal{F}$ such that \eqref{eq:diffTVvsNoisyTVConverge} happens with $G=\mathcal{N}(0,\sigma^2_0)$ for any $\sigma_0>0$.
\end{claim}
The proof of this claim can be found in Appendix \ref{sec:app:proofs:claims:II}. Before presenting our main result in Theorem \ref{main_theorem}, we first demonstrate that Assumption \ref{class_assumption_2} holds in various practically relevant scenarios.

\begin{claim}[Gaussian Family $\mathcal{N}(\boldsymbol{\mu},\sigma^2\boldsymbol{I}_d)$ with $\sigma\ge\sigma_0$]
\label{frequency_gaussian_claim}
Assume the restricted Gaussian family $\mathcal{F}=\left\{\mathcal{N}(\boldsymbol{\mu},\sigma^2\boldsymbol{I}_d)\vert~\boldsymbol{\mu}\in\reals^d,~\sigma\ge\sigma_0\right\}$, for $\sigma_0>0$. Then, Assumption \ref{class_assumption_2} holds with 
\begin{equation}
\mathsf{P}(\mathcal{F})\supseteq
\left\{\left(\alpha,2^{(d/2+2)}e^{-\sigma^2_0\alpha^2/2}\right)\Big\vert~\alpha>\frac{1}{\sigma_0}\sqrt{(d + 4)\log 2}\right\}.
\end{equation}
\end{claim}
Similar to Claim \ref{claim_lipschitz_gaussian}, the results can be extended to the more general case of $\mathcal{N}(\boldsymbol{\mu},\boldsymbol{\Sigma})$ with $\lambda_{\min}\left(\Sigma\right)\ge\sigma_0$, however, this might complicate the proofs.

\begin{claim}[$k$-Mixtures of Uniform Measures over $\reals$]
\label{claim:lowFreq:kMixtureUniform}
For any $k\in\mathbb{N}$ and the minimum bandwidth $T>0$, consider the following class of distributions over $\reals$:
\begin{align}
\mathcal{F}=\left\{
f:x\to
\frac{\mathbbm{1}(a\leq x\leq b)}{b-a}\bigg\vert~
a,b\in\reals, ~b-a\ge T
\right\}.
\nonumber
\end{align}
Then, letting $\varepsilon\triangleq\left\Vert p-q\right\Vert_2$, we have 
$$
\mathsf{P}\left(k\mathrm{-Mix}(\mathcal{F})\right)\supseteq
\left\{\left(\alpha,
1-\zeta\left(\frac{\alpha T^2\varepsilon^2}{2(4k-1)}\right)
\right)\bigg\vert~\alpha>0\right\},
$$ 
where function $\zeta(\cdot)$ is defined as
$\zeta(h)\triangleq
\frac{2}{\pi}
\int_{0}^{h}\frac{\sin^2 u}{u^2}\mathrm{d}u,\quad\forall h\ge0$.
\end{claim}
Proof is given in Appendix \ref{sec:app:proofs:claims:II}. Claim \ref{claim:lowFreq:kMixtureUniform} relies on a key property of the Fourier decay of indicator functions over convex bodies or shapes with smooth boundaries (see \cite{brandolini2003sharp}), which in one dimension reduces to intervals. Claim \ref{claim:lowFreq:kMixtureUniform} naturally extends to $\reals^d$, covering a broad class of uniform distributions over convex or smooth bodies—such as polygons and hyperellipses—with several applications in learning high-dimensional shapes from noisy uniform samples \cite{boissonnat2007learning,najafi2021statistical,saberi2023sample}. While deriving nearly-tight sample complexity bounds for such classes lies beyond the scope of this paper, we point to it as a compelling direction for future work. 
Meanwhile, in Section \ref{sec:examples}, we utilize this claim to derive (for the first time) a sample complexity bound for learning $k$-mixtures of $d$-dimensional uniform distributions under noise or adversarial perturbations.

The following theorem establishes a bound in recovering $f^*$ in $\ell_2$-norm:
\begin{theorem}[Main Result]
\label{main_theorem}
Let $\mathcal{F}$ be a distribution family over $\mathcal{X}\subseteq\reals^d$ satisfying Assumption \ref{assumption:L2(X)}, Assumption \ref{assumption:SC} with a sample compression scheme $(\tau,t, m)$, and Assumption \ref{assumption_decoder} with a bounded constant $r\ge0$. Moreover, let Assumption \ref{class_assumption_2} hold for the set of pairs $\mathsf{P}(\mathcal{F})=\left\{(\alpha,\xi)\right\}$. Assume $G\in L^2(\mathcal{X})$ be a symmetric product measure with component-wise CDF of $\Phi_G$. Define
$B_G(\alpha)\triangleq\inf_{\Vert\boldsymbol{\omega}\Vert_2\leq\alpha}
\left\vert
\mathsf{F}\left\{G\right\}(\boldsymbol{\omega})
\right\vert$ for $\alpha>0$.
For any unknown $f^*\in\mathcal{F}$ and $\epsilon,\delta \in (0,1)$, assume we have $n$ i.i.d. samples from $f^**G$ with
\begin{align}
\label{sample_complexity_main_bound}
n ~\ge~&
N^{\mathsf{Clean}}_{\tau,t,m}(6\epsilon,\delta/2)
~+
\\
&\mathcal{O}\left(
\frac{d\tau\left(\epsilon\right)}{\epsilon^2}
\log\left(
\frac{r}{\epsilon}\sqrt{d\tau\left(\epsilon\right)}
\left|\Phi_G^{-1}\left(\frac{\delta}{8dm\left(\epsilon\right)\log\frac{4}{\delta}}\right)\right|
\right)
\log\left( m(\epsilon) \log \left( \frac{1}{\delta} \right) \right) 
\right),
\nonumber
\end{align}
where $N^{\mathsf{Clean}}_{\tau,t,m}\left(\epsilon,\delta\right)$ is the sample complexity of the noiseless regime, as defined in Theorem \ref{thm:intro:sample_complexity_sc} (full details in Theorem \ref{sample_complexity_sc}).
Then, there exists a deterministic algorithm that takes the $n$ perturbed samples as input, and outputs $\widehat{f}\in\mathcal{F}$ such that the following bound holds with probability at least $1-\delta$:
\begin{equation}
\Vert \widehat{f} - f^* \Vert_2 \leq 
\epsilon\cdot
\left(
\inf_{(\alpha,\xi)\in\mathsf{P}(\mathcal{F})}
\frac{24}{\sqrt{B_G(\alpha)(1-\xi)}}
\right).
\end{equation}
\end{theorem}
The full proof is given in Appendix \ref{proof_of_gaussian_case}, and here we discuss a sketch of proof. Before that, let us investigate two special cases of Gaussian and  multi-dimensional Laplace noise distributions as candidates for $G$ (proof is given in Appendix \ref{sec:app:proofs_additive_noise}, as well).
\begin{corollary}[Gaussian and Laplace Noise Models]
\label{corl:GaussianLaplaceMainThm}
Consider the setting of Theorem \ref{main_theorem}. Assume two scenarios for noise distribution $G$: $\mathrm{i)}$ Gaussian noise  $G\triangleq\mathcal{N}\left(\boldsymbol{0},\sigma^2\boldsymbol{I}_d\right)$ for some $\sigma>0$, and $\mathrm{ii)}$ multi-dimensional Laplace noise $G_i\triangleq\mathrm{Laplace}(\sigma),~i\in[d]$. Then, assuming
\begin{align}
n&\ge
{\mathcal{O}}\left(
N^{\mathsf{Clean}}_{\tau,t,m}\left(\epsilon,\delta\right)
\right)
+
\widetilde{\mathcal{O}}\left(
\frac{d{\tau}(\epsilon)}{\epsilon^2}
\right)\log(1+\sigma r),
\end{align}
with probability at least $1-\delta$, the $\ell_2$ error $\Vert\widehat{f}-f^*\Vert_2$ corresponding to cases  $\mathrm{i)}$ and $\mathrm{ii)}$ is respectively bounded as
\begin{align}
\mathrm{i)}~
\Vert\widehat{f}-f^*\Vert_2\leq
\epsilon
\inf_{(\alpha,\xi)\in\mathsf{P}(\mathcal{F})}
24\sqrt{\frac{e^{(\sigma\alpha)^2}}{1 - \xi}}
\quad,\quad
\mathrm{ii)}~
\Vert\widehat{f}-f^*\Vert_2\leq
\epsilon 
\inf_{(\alpha,\xi)\in\mathsf{P}(\mathcal{F})}
\frac{24}{\sqrt{1-\xi}}\left(1+\frac{(\sigma\alpha)^2}{d}\right)^{d/2}.
\nonumber
\end{align}
\end{corollary}
\begin{proof}[Sketch of the proof for Theorem \ref{main_theorem}]
Based on Proposition \ref{sample_compression_theorem} and Theorem \ref{thm:intro:sample_complexity_sc}, we can deduce the existence of a deterministic algorithm that given noisy samples outputs $\widehat{f} \in \mathcal{F}$ such that, with high probability,  
$\mathsf{TV}(f^* * G, \widehat{f} * G) \leq \mathcal{O}(\epsilon)$.  The next step is to establish that this bound implies closeness between $ f^* $ and $ \widehat{f} $, at least in the $\ell_2$-norm sense. We later extend this result to total variation (TV) distance under additional necessary conditions. However, as discussed before such closeness guarantees—whether in TV or $\ell_2$ distance—do not hold universally for all functions $ f^* $ and $ \widehat{f} $ (see Claim \ref{impossibility_assumption2_example}). This is precisely where Assumption \ref{class_assumption_2} plays a crucial role.  

To proceed, we seek a function $ h(\cdot) $ such that  $\mathsf{TV}(f * G, g * G) \leq {\epsilon}$ results into $\left \Vert f - g \right \Vert_2 \leq h(\epsilon)$. Note that the opposite always holds, i.e., if $f$ and $g$ are close in $\ell_2$ or TV distance, they remain close after being convolved with $G$. Using Assumption \ref{class_assumption_2} and leveraging key invariance properties such as Parseval's theorem, we show that if $f-g$ retains non-negligible energy in low-frequency regions, a smooth noise distribution (e.g., Gaussian or Laplace) cannot entirely suppress it. This enables us to derive explicit formulations, such as  
$
h(\epsilon) = 
\left(B_G(\alpha)(1 - \xi)\right)^{-1/2}
\mathcal{O}(\epsilon),
$  
for any $\alpha,\xi$ that satisfies Assumption \ref{class_assumption_2}. This leads directly to the results stated in Theorem \ref{main_theorem}.
\end{proof}
Theorem \ref{main_theorem} essentially states that to learn $ f^*$ up to a $\ell_2$ error of $ \epsilon$ with high probability (at least $ 1 - \delta $), one requires  
$
\widetilde{\mathcal{O}}\left(
N^{\mathsf{Clean}}_{\tau,t,m}(\epsilon, \delta) 
+
\frac{d \tau(\epsilon)}{\epsilon^2} 
\right)
$ 
samples. The first term represents the \emph{vanilla} sample complexity of learning $ f^* $ from clean samples, while the second term accounts for the additional cost introduced by the presence of noise.  Notably, only the dimension $ d $ and $ \tau(\epsilon) $—the length of the decoder input samples—explicitly appear in the bound. Other factors, such as noise power (inherent in $ \Phi_G $) and the local Lipschitz constant $ r $, are encapsulated within polylogarithmic terms.  Moreover, the theorem asserts that learning $ f^* * G $ up to a TV error of $ \epsilon $ is equivalent to learning $ f^* $ up to an $ \ell_2 $ error proportional to $ \epsilon $. The proportionality constant depends on specific properties of the noise (its distribution and variance) and Fourier-based characteristics of $ \mathcal{F} $.  For example, in the case of both Gaussian and Laplace noise models, the final sample complexity might be even exponentially increasing w.r.t. variance of the noise (see Corollary \ref{corl:GaussianLaplaceMainThm}). This phenomenon has been already observed in, for example, learning high-dimensional simplices from noisy samples in \cite{saberi2023sample}.

\subsection{Guarantees on Total Variation Error}
\label{sec:subsec:L2toTV}

Theorem \ref{main_theorem} establishes PAC-learnability in $ \ell_2 $-norm. However, in many scenarios, learning guarantees under the TV norm are of greater interest, as in the noise-free sample compression scheme of \cite{ashtiani2018nearly}. Generally, the TV error cannot be directly bounded by the $ \ell_2 $-error, and several impossibility results exist in this regard (see \cite{devroye2013probabilistic}). Nonetheless, under certain sufficient—but not necessary—conditions on the tail decay rate of the PDFs in $ \mathcal{F} $, it is possible to derive such bounds. 

\begin{proposition}[Water-Filling Bound on TV Error via $\ell_2$-Norm] 
\label{proposition:WaterFilling} 
Let $f, g \in L^2(\mathcal{X})$, and assume $g(\boldsymbol{x}) \ge 0$ for all $\boldsymbol{x} \in \mathcal{X}$. Consider the following water-filling construction: find a Lebesgue-measurable set $A=A(\Vert f\Vert_2, g)\subseteq\mathcal{X}$ such that the followings hold:
\begin{align}
\mathrm{Vol}(A)\inf_{\boldsymbol{x}\in A}g^2(\boldsymbol{x})+
\int_{\mathcal{X}\backslash A}g^2(\boldsymbol{x})\mathrm{d}\boldsymbol{x}
=\Vert f\Vert^2_2,
\quad\mathrm{and}\quad
g(\boldsymbol{x})\ge g(\boldsymbol{y}),\quad
\forall( \boldsymbol{x}\in A,~\boldsymbol{y}\notin A).
\end{align}
Then, assuming $|f(\boldsymbol{x})| \le g(\boldsymbol{x})$ for all $ \boldsymbol{x} \in \mathcal{X} $, the following upper bound holds:
\begin{align}
\Vert f\Vert_1
\leq
\mathrm{Vol}(A)\inf_{\boldsymbol{x}\in A\left(\Vert f\Vert_2,g\right)}g(\boldsymbol{x})+
\int_{\mathcal{X}\backslash A\left(\Vert f\Vert_2,g\right)}g(\boldsymbol{x})\mathrm{d}\boldsymbol{x}.
\end{align}
\end{proposition}
As it becomes evident during the proof of Proposition \ref{proposition:WaterFilling} (see Appendix \ref{sec:app:proofs_additive_noise}), the above procedure to determine the set $A$ corresponds to a water-filling construction. The following corollary (also proved in Appendix \ref{sec:app:proofs_additive_noise}) illustrates two specific tail decay conditions—serving as concrete choices for the function $g$ in Proposition \ref{proposition:WaterFilling}—and derives explicit bounds on the TV error when recovering $f^* \in \mathcal{F}$ based on $n$ noisy samples drawn from $f^* * G$.

\begin{corollary}[Bounded Support or Sub-Gaussianity] \label{corl:noise:TVfromL2} Assume the setting of Theorem \ref{main_theorem}, and consider the following cases: $\mathrm{i)}$ Suppose that $ \mathcal{F} $ has bounded support; that is, there exists $ R > 0 $ such that $ \mathrm{supp}(f) \subseteq [-R, R]^d $ for all $ f \in \mathcal{F} $. Then, 
\begin{align} 
\label{eq:propBoundednessTVfromL2:Bound1} 
\mathsf{TV}(\widehat{f}, f^*) \leq (2R)^{d/2} \Vert\widehat{f}-f^*\Vert_2.
\end{align}
$\mathrm{ii)}$
Suppose every $f \in \mathcal{F}$ satisfies a sub-Gaussian bound: there exist constants $C_1, \gamma > 0$ such that for all $\boldsymbol{x} \in \mathbb{R}^d$, $ f(\boldsymbol{x}) \leq C_1 \exp(-\gamma \|\boldsymbol{x}-\boldsymbol{\mu}\|_2^2) $, where $\boldsymbol{\mu}$ is the mean of $f$. Then,
\begin{align}
\mathsf{TV}(\widehat{f}, f^*) \leq C_2 \Vert
\widehat{f}-f^*\Vert_2\cdot
\log\left(\frac{1}{\Vert
\widehat{f}-f^*\Vert_2}\right)^{d/2},
\end{align}
for some constant $C_2$ depending on $C_1$, $\gamma$, and $d$.
\end{corollary}

%% file: Adversary.tex
This section examines the adversarial perturbation model introduced in Section \ref{sec:problemDef}, in which an adversary may corrupt up to $s$ of the $n$ samples, with an $\ell_{\infty}$ budget of $C \ge 0$ per sample. As shown later in Remark \ref{remark:adversarial:logC}, the required sample complexity depends at most logarithmically on $C$. We allow the adversary to be quite powerful: it is assumed to have full knowledge of both the true distribution $f^*$ and the learning algorithm. Consequently, the resulting attacks are minimax-optimal from the adversary's perspective.

This adversarial setting is well-established in both theoretical and practical contexts (see, e.g., \cite{sinha2018certifying,mahloujifar2019curse}). However, to the best of our knowledge, prior work has not provided general learnability or sample complexity guarantees of the kind established here—especially with the level of generality attained in Proposition \ref{adversarial_sample_decoder} and Theorem \ref{adversarial_main_theorem}. An informal takeaway from this section is that the sample complexity under adversarial corruption increases by a factor of $s + ds^2/P$ relative to the clean case, where $P$ denotes the number of free parameters (i.e., degrees of freedom) of $\mathcal{F}$, typically $\widetilde{\mathcal{O}}(\tau + t)$. See Theorem \ref{adversarial_main_theorem} and Remark \ref{remark:adversarial:sqrtn} for precise bounds. A key difference from the noisy setting considered in Section \ref{sec:Gaussian} is that certain regularity assumptions—such as Assumption \ref{class_assumption_2}, which are minimax necessary in that context—are no longer needed here. We elaborate on this in Remark \ref{remark:adversarial:smoothnessNotNeeded}.



\subsection{TV-Learnability of $\mathcal{F}$}

The core technical tools employed in this section build upon those developed in Section \ref{sec:Gaussian}. However, here they are adapted to the adversarial setting via a novel application of the same underlying compression-based techniques.

\begin{proposition}[Adversarial Sample Compressibility of $\mathcal{F}$]
\label{adversarial_sample_decoder}
Let $\mathcal{F}$ be a family of distributions on $\reals^d$ satisfying Assumptions \ref{assumption:L2(X)}, and \ref{assumption:SC} with functions $(\tau, t, m)$. Also, let $\mathcal{F}$ satisfy Assumption \ref{assumption_decoder} with a bounded Lipschitz constant $r \ge 0$. For any $f^* \in \mathcal{F}$ and $n \in \mathbb{N}$, we draw $n$ i.i.d. samples from $f^*$, where an adversary can corrupt up to $s<n$ samples according to the procedure described in Section \ref{sec:problemDef} with a budget $C\ge0$. Then, for any $\epsilon, \delta \in (0,1)$, the class $\mathcal{F}$ admits a
$$
\left(
\tau\left(\tfrac{\epsilon}{2}\right),
t\left(\tfrac{\epsilon}{2}\right) +
ds \log_2\left(1 + \frac{Cr\sqrt{ds}}{\epsilon} \right) + 
s \log\left(\frac{em(\epsilon/2)}{s}\log \frac{1}{\delta} \right),
m\left(\tfrac{\epsilon}{2}\right) \log \frac{1}{\delta}
\right)
$$
-sample compression scheme that is robust to adversarial samples.
\end{proposition}
The full proof of Proposition \ref{adversarial_sample_decoder} is provided in Appendix \ref{sec:app:proofs_adversarial}. The proposition guarantees the existence of a robust decoder $\mathcal{J}_r$ for the class $\mathcal{F}$ with the following property: For any $f^* \in \mathcal{F}$, given $\Tilde{t}(\epsilon)$ specific bits and at most $\Tilde{\tau}(\epsilon)$ designated samples — which \emph{exist} with probability at least $1 - \delta$ among the $n \ge m(\epsilon/2) \log \frac{1}{\delta}$ i.i.d. samples $\boldX_i$ (potentially perturbed into $\Tilde{\boldX}_i$) — the decoder can output a distribution $\widehat{f}$ that is $\epsilon$-close to the true distribution $f^*$. The core of our proof for Proposition \ref{adversarial_sample_decoder} is as follows: we use approximately $\log\binom{n}{s} \simeq s \log(n/s)$ bits to identify which samples have been perturbed, and then apply a quantization scheme to denoise the corrupted samples. Notably, the core techniques from our earlier proofs in Section \ref{sec:Gaussian} are agnostic to the source of perturbation — whether it arises from independent noise or an adversarial process. 

The next step is to leverage the result of Proposition \ref{adversarial_sample_decoder} to learn the class $\mathcal{F}$ in a PAC framework. However, this setting introduces a crucial challenge not present in earlier perturbation models. The core idea behind learning via sample compression involves partitioning the samples drawn from $f^* \in \mathcal{F}$ into two groups: (i) the first group is used to construct a potentially exponential number of \emph{candidate} hypotheses, and (ii) the second, independent group is used in a \emph{hypothesis testing} phase to select the best candidate. In adversarial settings, however, the perturbations are non-i.i.d. — and indeed, non-statistical — which prevents a straightforward application of Theorem \ref{existence_of_algorithm_with_samples} to the outcome of Proposition \ref{adversarial_sample_decoder}. To overcome this, we introduce a new technique. We partition the samples into multiple groups such that a strict majority—at least half plus one—are guaranteed to be free from adversarial corruption. These clean groups can be used to generate \emph{good} candidates for $f^*$ using Theorem \ref{existence_of_algorithm_with_samples}, which can then be identified via a simple ``clique recovery'' procedure. Further details are provided following the next theorem, which establishes our main result in this section, i.e., PAC learnability guarantees for $\mathcal{F}$ in total variation distance under adversarial perturbations.


\begin{theorem}[Main Result]
\label{adversarial_main_theorem}
Under the same setting as Proposition \ref{adversarial_sample_decoder}, and for any $\epsilon,\delta \in (0,1)$ assume
\begin{align}
n \ge\mathcal{O}\left(
m(\epsilon)\log\frac{1}{\delta}
+
\frac{s}{\epsilon^2} \left[
t(\epsilon)
+
\left(\tau(\epsilon)+s\right)\log\left(
m\left(\epsilon\right)\log\frac{1}{\delta}
\right) +
ds \log\left(1 + \frac{Cr\sqrt{ds}}{12\epsilon}\right)
\right] \right).
\nonumber
\end{align}
Then, there exists a deterministic algorithm that takes the $n$ perturbed samples as input, and outputs $\widehat{f}\in\mathcal{F}$ such that $\mathbb{P}(\mathsf{TV} (\widehat{f} - f^*)  \leq 
12\epsilon)\ge 1-\delta$.
\end{theorem}
Proof of Theorem \ref{adversarial_main_theorem} is given in Appendix \ref{sec:app:proofs_adversarial}. Here, we give a brief sketch of proof.

\begin{proof}[Sketch of proof for Theorem \ref{adversarial_main_theorem}]
The core idea—building upon the procedure of Section \ref{sec:Gaussian}, and in particular Proposition \ref{proposition:SampComp:NoisyF}—is to design a \emph{denoising} strategy for handling adversarially corrupted samples. In this regard, Proposition \ref{adversarial_sample_decoder} provides a principled method, based on Assumption \ref{assumption_decoder}, to construct a hypothesis set $\{f_1, \ldots, f_M\}$ of controlled size $M$, such that with high probability, at least one $f_i$ is $\epsilon$-close to the target distribution $f^*$.

The next step, which involves a nontrivial idea, is to select a sufficiently good candidate from among the $f_i$’s using a hypothesis testing scheme. To achieve this, we form $2s+1$ disjoint groups of sufficiently large i.i.d. (but potentially corrupted) samples. By the pigeonhole principle, at least $s+1$ of these groups are guaranteed to be free from adversarial corruption. Applying Theorem \ref{existence_of_algorithm_with_samples} to each group yields $2s+1$ hypotheses. Among these, the $s+1$ clean outputs form a clique of size at least $s+1$, where any pair within the clique has total variation distance at most $2\epsilon$ from each other (and at most $\epsilon$ from $f^*$). The existence of such a large and tight clique implies that it can be provably detected. Hence, the final selected hypothesis—any of the estimates in this clique—is theoretically guaranteed to achieve the desired accuracy.
\end{proof}


\subsection{Discussion and Limitations}

In this subsection, we highlight several by-products and limitations of the results presented in Proposition \ref{adversarial_sample_decoder} and Theorem \ref{adversarial_main_theorem}.

\begin{remark}
\label{remark:adversarial:smoothnessNotNeeded}
Theorem \ref{adversarial_main_theorem} directly bounds the total variation (TV) error $\mathsf{TV}(\widehat{f},f^*)$, and unlike Theorem \ref{main_theorem}, it does not require any low-frequency or smoothness assumptions (e.g., Assumption \ref{class_assumption_2}).
\end{remark}

The reason for this distinction is that, unlike the noisy setting of Section \ref{sec:Gaussian}, the adversarial scenario considered here still provides access to clean samples—even though we do not know which samples are uncorrupted. In contrast, in Theorem \ref{main_theorem}, all samples are contaminated by noise, leading to the irreversible loss of fine structure in the underlying distribution (e.g., singular regions), which justifies the need for smoothness or spectral assumptions. On the other hand, the assumption of stable decodability (Assumption \ref{assumption_decoder}) is essential in both settings.

\begin{remark}
\label{remark:adversarial:sqrtn}
The sample complexity in Theorem \ref{adversarial_main_theorem} scales quadratically with the number of adversarial corruptions $s$:
$$
n \ge
\Tilde{\mathcal{O}}\left(
m(\epsilon)+\frac{s\left(t(\epsilon)+\tau(\epsilon)\right)}{\epsilon^2}
+
\frac{ds^2}{\epsilon^2}
\right).
$$
The clean sample complexity—i.e., the case with no adversarial interference—is 
$
N^{\mathrm{clean}}_{\tau,t,m}(\epsilon,\delta) = \Tilde{\mathcal{O}}\left(
m(\epsilon) + {\epsilon^{-2}}({\tau(\epsilon) + t(\epsilon)})
\right).
$
The additional $ds^2/\epsilon^2$ term in the adversarial case arises due to the use of the $\ell_{\infty}$-norm in defining the adversarial perturbations. Employing weaker adversaries (e.g., with respect to $\ell_2$ or $\ell_1$ norms) could significantly reduce the dependency on the ambient dimension $d$.
\end{remark}

\begin{remark}
\label{remark:adversarial:logC}
The sample complexity in Theorem \ref{adversarial_main_theorem} depends only logarithmically on the adversarial budget $C$, i.e., $n \ge \mathcal{O}(\log C)$.
\end{remark}

%% file: Example.tex

We demonstrate how the findings from the previous parts of this work come together to solve two theoretical examples that, to the best of our knowledge, have not been previously addressed.


\subsection{Learning $k$-UMMs in Noisy or Adversarial Regimes}
\label{sec:examples:UMM}

In our first example, we consider the problem of learning Uniform Mixture Models (UMMs). For $T>0$, consider the class of distributions $\mathcal{F}$, consisting of uniform distributions over axis-aligned hyper-rectangles in $\mathbb{R}^d$, defined as (also see Claim \ref{claim_lipschitz_uniform}):
\begin{align}
\label{eq:example:UMMdef}
\mathcal{F}=
\left\{
f:\boldsymbol{x}\mapsto
\prod_{i=1}^{d}
\frac{\mathbbm{1}\left(a_i\leq x_i\leq b_i\right)}{b_i-a_i},~\forall\boldsymbol{x}\in\mathbb{R}^d
\;\middle|\;
\boldsymbol{a},\boldsymbol{b}\in\mathbb{R}^d,\;b_i-a_i\ge T,\;\forall i\in[d]
\right\}.
\end{align}  
For any $k\in\mathbb{N}$, we aim to analyze the sample complexity of $k$-mixtures of $\mathcal{F}$, also known as $k$-uniform mixture models or $k$-UMMs, denoted $k\mathrm{-Mix}(\mathcal{F})$ under the perturbation models considered thus far (see Claim \ref{claim_lipschitz_mixture} for a formal definition).

This family is widely employed for modeling piecewise constant probability density functions \cite{browne2011model,brunot2019gaussian}. In fact, any density function with mild continuity properties (such as piecewise continuity) can be closely approximated by a $k$-UMM, provided $k$ is chosen sufficiently large. Hence, providing explicit sample complexity guarantees for learning such models under general perturbations (e.g., noise or adversarial interference) is of both theoretical and practical significance. We now verify that the assumptions required by Theorems \ref{main_theorem} and \ref{adversarial_main_theorem} are satisfied for this distribution class.

\begin{itemize}
\item 
Each component of the $k$-UMM $f^*$ is a product of $d$ one-dimensional uniform distributions. A one-dimensional uniform distribution over $\mathbb{R}$ admits $(2, 0, \frac{2}{\epsilon} \log \frac{2}{\delta})$-compression for any $\epsilon,\delta\in(0,1)$. This is because identifying the minimum and maximum of the support interval suffices to reconstruct the distribution, and with at most $\frac{2}{\epsilon} \log \frac{2}{\delta}$ samples, we can guarantee the existence of two $\epsilon/2$-close surrogates for these extremes.

\item 
It has been proved that products of compressible distributions remain compressible. In particular, according to Lemma 3.6 of \cite{ashtiani2018nearly}, $\mathcal{F}$ admits  
$$
\left(2d, 0, \tfrac{2d}{\epsilon} \log \tfrac{2}{\delta} \log(3d)\right)
\text{-sample~compression}.
$$

\item 
In addition, Lemma 3.7 of \cite{ashtiani2018nearly} implies that the class $k\mathrm{-Mix}(\mathcal{F})$ admits
$$
\left(2kd, k\log_2 \tfrac{4k}{\epsilon}, \tfrac{288dk}{\epsilon} \log \tfrac{2}{\delta} \log \tfrac{6k}{\epsilon} \log(3d)\right)
\text{-sample~compression},
$$
which satisfies Assumption \ref{assumption:SC}.

\item 
Moreover, from Claim \ref{claim_lipschitz_uniform}, each component of $\mathcal{F}$ satisfies Assumption \ref{assumption_decoder} with Lipschitz constant $r \leq \frac{8d}{T}$. Consequently, by Claim \ref{claim_lipschitz_mixture}, the class $k\mathrm{-Mix}(\mathcal{F})$ satisfies the same assumption with Lipschitz constant $r \leq \frac{8d}{T} \sqrt{k}$.
\end{itemize}
Finally, the following claim generalizes Claim \ref{claim:lowFreq:kMixtureUniform} to $d$ dimensions. It identifies an appropriate set of pairs $(\alpha, \xi(\cdot)) \in \mathsf{P}(k\mathrm{-Mix}(\mathcal{F}))$, which completes the requirements for applying Theorem \ref{main_theorem} and Corollary \ref{corl:GaussianLaplaceMainThm}.
\begin{claim}[Extension of Claim \ref{claim:lowFreq:kMixtureUniform} to $d$ dimensions]
\label{claim:lowFreq:UnifdDim}
Let $d, k \in \mathbb{N}$ and minimal bandwidth $T > 0$. Let the class $\mathcal{F}$ be as defined in \eqref{eq:example:UMMdef}. Then, for $\varepsilon \triangleq \|p - q\|_2$, the class $k\mathrm{-Mix}(\mathcal{F})$ satisfies Assumption \ref{class_assumption_2} with:
$$
\mathsf{P}(k\mathrm{-Mix}(\mathcal{F}))
\supseteq
\left\{
\left(\alpha,
1 - \zeta^d\left(\tfrac{\alpha}{2ck\sqrt{d}} (T\varepsilon)^{2/d}\right)
\right)
~\middle|~ \alpha > 0
\right\},
$$
where $\zeta(\cdot)$ is defined as in Claim \ref{claim:lowFreq:kMixtureUniform}, and $c > 0$ is a universal constant.
\end{claim}
Proof is given in Appendix \ref{sec:app:proofs:claims:II}. We now present our main results throgh the following set of propositions:

\begin{proposition}[Learnability of $k$-UMMs from Noisy Samples]
\label{proposition:example:noisy-kUMM}
Consider a target distribution $f^*\in k\mathrm{-Mix}(\mathcal{F})$, and assume we have access to $n$ i.i.d. samples corrupted by additive Gaussian noise:  
$\boldsymbol{\zeta}_1, \dots, \boldsymbol{\zeta}_n \widesim[2.5]{i.i.d.} \mathcal{N}(\boldsymbol{0}, \sigma^2\boldsymbol{I}_d)$ for a sufficiently large $\sigma>0$. Then, for any $\epsilon,\delta>0$, there exists an estimator $\widehat{f}$ such that upon having
$$
n \geq \mathcal{O}\left( \frac{d^2k}{\epsilon^2} \log^2\left( \frac{rdk\sigma}{\epsilon\delta} \right) \right),
$$
guarantees that $\Vert\widehat{f}-f^*\Vert_2^{2}
\leq
\frac{24\epsilon}{T}
\left(
\pi ck
\sigma
\sqrt{2e}
\right)^{d/2}$ with probability at least $1-\delta$.
\end{proposition}
The proof is provided in Appendix \ref{sec:app:example} and follows directly by applying the steps outlined in Theorem \ref{main_theorem} and Corollary \ref{corl:GaussianLaplaceMainThm} to the specific properties of the $k$-UMMs established above. When $\sigma$ is small (i.e., $\sigma \to 0$), the resulting bounds become too intricate to express in closed form due to the behavior of $\zeta(\cdot)$ in Claim \ref{claim:lowFreq:UnifdDim}. An interested reader is referred to the proof for the precise characterization of the bounds in the small-$\sigma$ regime.

A simple rearrangement reveals that the same sample complexity in Proposition \ref{proposition:example:noisy-kUMM} guarantees a high probability error bound of $\Vert\widehat{f}-f^*\Vert_2^{2}
\leq
\frac{24\epsilon}{T}\left(4ckb\right)^{d/2}$ for the case of having a $\mathrm{Laplace}(b)$-distributed noise instead of Gaussian noise. Next, we provide the following proposition (proved in Appendix \ref{sec:app:example}) which gives explicit sample complexity bounds for inferring a $k$-UMM from $s$ (out of $n$) adversarially corrputed samples:

\begin{proposition}[Learnability of $k$-UMMs in Adversarial Regimes]
\label{proposition:example:kUMM-adversarial}
Consider a target distribution $f^*\in k\mathrm{-Mix}(\mathcal{F})$, and assume an adversary that can corrupt up to $s$ out of $n$ i.i.d. samples from $f^*$. Each corrupted sample $\Tilde{\boldX}_i$ satisfies
$
\Vert \Tilde{\boldX}_i - \boldX_i \Vert_{\infty} \le C,
$
where $C \ge 0$ is the adversarial budget. Then, there exists an estimator $\widehat{f}$ such that for any $\epsilon,\delta>0$ and upon having
$$
n 
\ge 
\widetilde{\mathcal{O}} \left(
\frac{s^2+kds}{\epsilon^2} 
\right)
+ 
\widetilde{\mathcal{O}} \left(
\frac{dks}{\epsilon^2} 
\right)\log(1+C),
$$
guarantees $\mathsf{TV}(\widehat{f},f^*)\leq\epsilon$ with probability at least $1-\delta$.
\end{proposition}


\subsection{Adversarial Learnability of $k$-GMMs}

We now consider the problem of learning Gaussian Mixture Models (GMMs) from adversarial samples. Providing explicit sample complexity guarantees for this problem is of both theoretical and practical significance, and yet is unanswered prior to our work. For some $\sigma_0>0$, consider the class of distributions $\mathcal{F'}$, consisting of Gaussian distributions in $\mathbb{R}^d$, defined as (Also see Claim \ref{claim_lipschitz_gaussian}):
\begin{align}
\label{eq:example:GMMdef}
\mathcal{F'}=
\left\{
\mathcal{N}\left(\boldsymbol{\mu},\Sigma\right)
\;\middle|\;
\Sigma\in\reals^{d\times d}~\mathrm{with}~
\lambda_{\min}(\boldsymbol{\Sigma}) \geq \sigma_0
,~ \boldsymbol{\mu}\in \reals^d
\right\}.
\end{align}
For any $k\in\mathbb{N}$, we aim to analyze the sample complexity of $k$-mixtures of $\mathcal{F'}$, denoted $k\mathrm{-Mix}(\mathcal{F'})$, under the adversarial perturbation model in Section \ref{sec:adversarial}. We verify that the assumptions required by Theorem \ref{adversarial_main_theorem} are satisfied for this distribution class.
\begin{itemize}
\item 
Due to Lemma 3.7 of \cite{ashtiani2018nearly}, $k\mathrm{-Mix}(\mathcal{F'})$ admits a
$$
\big(\mathcal{O}(kd\log (2d) )
~,~ 
\mathcal{O}(kd^2 \log (2d) \log  (d/\epsilon) + k \log(k/\epsilon)) ~,~ 
\mathcal{O}(dk \log k \log (2d)/\epsilon )
\big)
$$
-sample compression scheme, which satisfies Assumption \ref{assumption:SC}.

\item 
Moreover, from Claims \ref{claim_lipschitz_gaussian}
and \ref{claim_lipschitz_mixture}, the class $k\mathrm{-Mix}(\mathcal{F'})$ satisfies Assumption \ref{assumption_decoder} with the Lipschitz constant $$r \leq \mathcal{O}\left ( \frac{\sqrt{k}}{\sigma_0 \sqrt{d\log(2d)}} \right).$$
\end{itemize}

\begin{proposition}[Learnability of $k$-GMMs from Adversarial Samples]
\label{proposition:example:GMM:adversarial}
Assume we have $n$ i.i.d. samples from an unknown target distribution $f^*\in k\mathrm{-Mix}(\mathcal{F'})$, and an adversary can corrupt up to $s<n$ samples such that each corrupted sample $\Tilde{\boldX}_i$ satisfies
$
\Vert \Tilde{\boldX}_i - \boldX_i \Vert_{\infty} \le C,
$
for some budget $C \ge 0$. Then, $f^*$ can be learned up to both $\ell_2$ and TV error of at most $\epsilon>0$ with probability at least $1-\delta$ (for any $\delta,\epsilon>0$) given that
$$
n\ge
\widetilde{\mathcal{O}}\left(
\frac{skd^2}{\epsilon^2}
\right)
\log\frac{1}{\delta}
+
\widetilde{\mathcal{O}}\left(
\frac{ds^2}{\epsilon^2}
\right)
\log\left(1+\frac{C}{\delta\sigma_0}\right).
$$
\end{proposition}
Proof is given in Appendix \ref{sec:app:example}. Also, recall that the sample complexity of the ideal (non-adversarial) regimes is $\widetilde{\mathcal{O}}(kd^2/\epsilon^2)\log\frac{1}{\delta}$.

%% file: conclusions.tex
In this work, we extended the theoretical framework of sample compressibility to accommodate perturbed data, encompassing both stochastic noise and adversarial corruption. We demonstrated that, under mild and general assumptions, sample-compressible distribution families remain learnable, with a quantifiable inflation in sample complexity due to perturbations. Along the way, we showed that many well-known parametric distribution families satisfy our assumptions with reasonable constants and coefficients, while also establishing minimax impossibility results that highlight the necessity of our conditions. The core technical contribution is a novel perturbation quantization technique that aligns naturally with the structure of sample compression, offering a perspective not attainable through traditional learnability frameworks such as Valiant’s PAC model \cite{valiant1984theory}. Nevertheless, assuming the sample compression conjecture—which posits the equivalence of PAC learnability and compressibility—our results carry broader generality. Our quantization strategy integrates seamlessly with existing compression schemes, enabling robust learning guarantees under both $\ell_2$ and total variation distance metrics. As concrete illustrations of our methods, we resolved two previously open problems: the learnability of finite mixtures of uniform distributions under noisy and adversarial perturbations, and the learning of Gaussian mixture models from adversarially corrupted samples.

%% file: open_problems.tex
Several important directions remain open for future investigation. We highlight a few examples:
\begin{itemize}
    \item 
    While our assumptions for learnability under perturbations are shown to be minimax-necessary, we only prove necessity in a limited (minimax) sense. A key open problem is to complete the necessity side of our results by characterizing conditions under which no PAC-learnability is achievable. In particular, it is unclear whether our conditions can be relaxed or fully characterized in a general necessary-and-sufficient form.
    
    \item  
    Another challenging question is whether some of the imposed assumptions (e.g., Local Lipschitz Decodability and the Low-Frequency Property) can be derived from one another. Establishing logical implications between these structural properties may help simplify or unify the current framework. However, such results would likely hinge on some deep conjectures such as the one asserting that PAC-learnability implies the existence of an efficient sample compression scheme—a conjecture which was suggested by \cite{ashtiani2018nearly} and still unresolved.
    
    \item
    Our analysis is information-theoretic in nature. Extending these results to efficient (polynomial-time) algorithms, particularly in high dimensions and under adversarial conditions, remains an open challenge.
\end{itemize}

%% file: App_Proposition_Sample_Compression_Proof.tex
\begin{proof}[Proof of Theorem \ref{sample_compression_theorem}]
\label{proof_of_sample_compression}
First, we use Assumption \ref{assumption:SC} on $\mathcal{F}$, which guarantees the existence of a decoder $\mathcal{J}$ satisfying Assumption \ref{assumption_decoder}. Specifically, for any $(\epsilon,\delta)\in(0,1)$, given $n$ i.i.d. and \emph{clean} samples $\boldX_1,\ldots,\boldX_n$ from any $f^*\in\mathcal{F}$, the decoder outputs $\widehat{f}\in\mathcal{F}$ satisfying $\mathsf{TV}(f^*,\widehat{f})\leq\epsilon$ with probability at least $1-\delta/2$. The decoder $\mathcal{J}$ requires a sequence of at most $\tau(\epsilon)$ samples, at most $t(\epsilon)$ bits, and it must hold that $n\ge m(\epsilon)\log(2/\delta)$.

Throughout the proof, let $\mathbf{L}\in\mathcal{X}^{\tau(\epsilon)}\subseteq\mathbb{R}^{d\tau(\epsilon)}$ denote this sequence of samples and let $\mathbf{B}\in\{0,1\}^{t(\epsilon)}$ represent the corresponding bit sequence. Importantly, sample compression does not require \emph{knowing} $\mathbf{L}$ or $\mathbf{B}$ explicitly—only their \emph{existence} is needed. The procedure for establishing learnability (Theorem \ref{thm:intro:sample_complexity_sc}) ensures these sequences are found by considering all possibilities. Mathematically, we have:
\begin{align}
\mathbb{P}\left(
\mathsf{TV}\left(f^*,\mathcal{J}(\mathbf{L},\mathbf{B})\right) \leq \epsilon
\right) \geq 1-\delta/2,
\quad \forall f^*\in\mathcal{F},
\label{eq:GaussianProof:mainStep1}
\end{align}
where the probability is taken over the randomness in generating clean samples $\boldX_1,\ldots,\boldX_n$. However, in our setting, we do not have access to clean samples; instead, we observe perturbed samples $\Tilde{\boldX}_i = \boldX_i + \boldsymbol{\zeta}_i$ for all $i\in[n]$. Thus, instead of knowing $\mathbf{L}$, we only know the existence of $\mathbf{L}_{\mathsf{N}}$, defined as:
$$
\mathbf{L}_{\mathsf{N}} \triangleq \left(\boldsymbol{X} + \boldsymbol{\zeta}\right)_{\boldsymbol{X}\in\mathbf{L}}.
$$

For $i\in[n]$ and $j\in[d]$, let $\zeta_{ij}$ denote the $j$th component of $\boldsymbol{\zeta}_i$. In order to do this, we try to approximate each $\zeta_{ij}$ by a member of a quantized grid: we quantize a symmetric interval in $\reals$ and generate a finite set of points, denoted by $I\subset \reals$ (with $\vert I\vert<\infty$), such that the following holds:
\begin{equation}
\label{eq:noisyGaussianQuantPrinc}
\mathbb{P}\left(
\forall i,j\Big\vert~
\min_{\widehat{\zeta}_{ij}\in I}~
\left\vert\zeta_{ij} - \widehat{\zeta}_{ij}\right\vert \leq \eta
\right)\ge 1-\delta/2,
\end{equation}
for some $\delta\in(0,1)$ and $\eta\ge0$, where $\mathbb{P}(\cdot)$ is with respect to the randomness of drawing noise vectors $\boldsymbol{\zeta}_i$.
We then establish the following lemma:
\begin{lemma}
\label{lemma:proofNoisyGaussian1}
For any $\epsilon,\delta\in(0,1)$,
assume there exists a decoder such that \eqref{eq:GaussianProof:mainStep1} holds for at least one clean sample sequence $\mathbf{L}$ and a corresponding bit sequence $\mathbf{B}$. Also, assume there exists a grid $I$ satisfying \eqref{eq:noisyGaussianQuantPrinc}. Then, there exist decoders $\mathcal{J}_1,\ldots,\mathcal{J}_M$ for $\mathcal{F}$ with $M = |I|^{d\tau(\epsilon)}$, such that:
\begin{align}
\mathbb{P}\left(
\exists i\in[M] ~\Big|~
\mathsf{TV}\left(f^*,\mathcal{J}_i(\mathbf{L}_{\mathsf{N}},\mathbf{B})\right) \leq \epsilon + \frac{1}{2} r \eta \sqrt{d\tau(\epsilon)}
\right) \geq 1 - \delta.
\end{align}
\end{lemma}
\begin{proof}
There are $\tau(\epsilon)$ clean samples, each of dimension $d$, in $\mathbf{L}$. Let $S\subseteq[n]$ denote the indices of these samples. Since \eqref{eq:noisyGaussianQuantPrinc} holds, we know that with probability at least $1 - \delta/2$, there exists at least one combination of the grid points $\widehat{\zeta}_{ij}$ for $i \in S, j \in [d]$ such that
$$
\left\vert \Tilde{X}_{ij} - \widehat{\zeta}_{ij} \right\vert = \left\vert X_{ij} + \zeta_{ij} - \widehat{\zeta}_{ij} \right\vert \leq \eta,\quad \forall i\in S, j\in[d].
$$
There are $\vert I\vert$ possible ways to denoise each dimension, and with $d \times \tau(\epsilon)$ instances corresponding to different samples and dimensions, the total number of possible denoising configurations in $\mathbf{L}_{\mathsf{N}}$ is given by $M = \vert I\vert^{d\tau(\epsilon)}$. Among these, at least one configuration results in a denoised sequence $\mathbf{L}'$ such that, with probability at least $1 - \delta/2$,
$$
\left\Vert\mathbf{L} - \mathbf{L}'\right\Vert_{\infty} \leq \eta \quad \Longrightarrow \quad \left\Vert \mathbf{L} - \mathbf{L}' \right\Vert_{2} \leq \eta \sqrt{d \tau(\epsilon)}.
$$
Using the Lipschitz continuity assumption (Assumption \ref{assumption_decoder}), we can then guarantee that
$$
\mathsf{TV}\left(
\mathcal{J}\left(\mathbf{L},\mathbf{B}\right),
\mathcal{J}\left(\mathbf{L}',\mathbf{B}\right)
\right) \leq \frac{r}{2} \left\Vert \mathbf{L} - \mathbf{L}' \right\Vert_{2} \leq \frac{1}{2} r\eta\sqrt{d\tau(\epsilon)}.
$$
By combining the results from \eqref{eq:GaussianProof:mainStep1} and \eqref{eq:noisyGaussianQuantPrinc}, applying a union bound over the probability of errors (i.e., $\delta/2 + \delta/2 = \delta$), and using the triangle inequality for total variation (TV) distance,
$$
\mathsf{TV}\left(f^*, \mathcal{J}\left(\mathbf{L}',\mathbf{B}\right)\right)
\leq \mathsf{TV}\left(f^*, \mathcal{J}\left(\mathbf{L},\mathbf{B}\right)\right) + \mathsf{TV}\left(\mathcal{J}\left(\mathbf{L},\mathbf{B}\right), \mathcal{J}\left(\mathbf{L}',\mathbf{B}\right)\right),
$$
we obtain
$$
\mathbb{P}\left( \mathsf{TV}\left(f^*,\mathcal{J}(\mathbf{L}',\mathbf{B})\right) \leq \epsilon + \frac{1}{2} r\eta\sqrt{d\tau(\epsilon)} \right) \geq 1 - \delta,
$$
which completes the proof of the lemma.
\end{proof}

The remaining task is to design a sufficiently fine grid that satisfies \eqref{eq:noisyGaussianQuantPrinc} and then use Lemma \ref{lemma:proofNoisyGaussian1} to integrate it into a new sample compression framework for $\mathcal{F} * G$. Let the grid of quantized points $I$ be given by $I = \{b_1, \dots, b_K\}$, where $K$ denotes the size of the grid. Without loss of generality, assume $b_1 < b_2 < \dots < b_K$. For simplicity, we assume the $b_k$ values are distributed such that each consecutive pair $b_k, b_{k+1}$ is spaced evenly. While more sophisticated grid designs exist, their impact on final sample complexity is negligible. Thus, we consider the following quantization format:
\begin{gather}
b_k - b_{k-1} = 2\eta, \quad \forall k \in \{2, ..., |I|\},
\\
b_{|I|} = -b_1 = \left| \Phi_G^{-1}\left(\frac{\delta}{4nd}\right) \right|.
\label{eq:first_and_last_b}
\end{gather}
Examining equation \eqref{eq:first_and_last_b}, we note that to maintain an overall quantization error probability of at most $\delta/2$, as required by \eqref{eq:noisyGaussianQuantPrinc}, the error in each of the $d$ dimensions of each sample $i \in [n]$ must not exceed $\delta/(2nd)$. Consequently, the quantization range must cover:
$$
\left( \Phi_G^{-1}\left(\frac{\delta}{4nd}\right), \Phi_G^{-1}\left(1 - \frac{\delta}{4nd}\right) \right),
$$
which justifies the conditions in \eqref{eq:first_and_last_b}. Here, we have used two facts: i) the density of the noise vector has i.i.d. components, thus each dimension is distributed according to $\Phi_G$ independently, and ii) the CDF is symmetric w.r.t. origin. Both of these constraints can be removed, however, at the cost of introducing more complex notations.

The grid $I$ can be represented using $\log_2(|I|)$ bits, implying that we require
$$
d\tau(\epsilon)\log_2\left(|I|\right) \leq d\tau(\epsilon) \log_2\left(1+ \frac{1}{\eta} \left| \Phi_G^{-1}\left(\frac{\delta}{4nd}\right) \right| \right)
$$
bits to construct all $M$ possible augmented decoders in Lemma \ref{lemma:proofNoisyGaussian1}. Additionally, we note that $n = m(\epsilon)\log(2/\delta)$. \footnote{Although at most $\tau(\epsilon)$ out of the total $n$ samples are used, the quantization criteria in \eqref{eq:noisyGaussianQuantPrinc} must hold for all $n$ samples, as we do not deterministically know which ones will form the chosen $\tau(\epsilon)$.} This necessitates an additional set of $d\tau(\epsilon)\log_2(|I|)$ bits, denoted $\mathbf{B}_{\mathsf{den}}$, to identify the chosen denoising scheme among the $M$ possibilities. Thus, the total number of required bits is $t(\epsilon) + d\tau(\epsilon)\log_2(|I|)$. So far, we have demonstrated that a sample compression scheme of
\begin{align}
\left[ \tau(\epsilon), t(\epsilon) + d\tau(\epsilon) \log_2\left( 1+\frac{1}{\eta} \left| \Phi_G^{-1}\left(\frac{\delta}{4dm(\epsilon)\log(2/\delta)}\right) \right| \right), m(\epsilon)\log(2/\delta) \right]
\end{align}
guarantees the existence of a decoder $\mathcal{J}_{\mathsf{den}}$ that can process $m(\epsilon)\log(2/\delta)$ noisy samples from $f^* * G$ and output a density $\widehat{f} \in \mathcal{F}$ such that
\begin{align}
\mathbb{P}\left( \mathsf{TV}(f^*, \widehat{f}) \leq \epsilon + \frac{1}{2} r\eta\sqrt{d\tau(\epsilon)} \right) \geq 1 - \delta.
\end{align}
There are multiple ways to \emph{optimize} the tradeoff between $\eta$ and $\epsilon$ in this setting, depending on the specific choices of $\tau$, $t$, and $m$ functions, as well as the noise CDF $\Phi_G$. Since the paper already contains sufficient technical detail, we proceed with the following simplified approach to guarantee a final total variation error of at most $\epsilon$:
\begin{align}
\left\{
\epsilon \leftarrow \epsilon/2
\quad \text{and} \quad
\frac{1}{2}r\eta\sqrt{d\tau(\epsilon/2)}
\leftarrow \epsilon/2
\right\}
\quad \Longrightarrow \quad
\eta \triangleq \frac{\epsilon}{r\sqrt{d\tau(\epsilon/2)}},
\end{align}
This yields the claimed sample compression scheme already stated in the theorem, as follows:
\begin{align}
\left[
\tau\left(\frac{\epsilon}{2}\right),
t\left(\frac{\epsilon}{2}\right) +
d\tau\left(\frac{\epsilon}{2}\right)
\log_2\left(1+
\frac{r\sqrt{d\tau\left({\epsilon}/{2}\right)}}{\epsilon}
\left|\Phi_G^{-1}\left(\frac{\delta}{4dm\left(\frac{\epsilon}{2}\right)\log\frac{2}{\delta}}\right)\right|
\right),
m\left(\frac{\epsilon}{2}\right)\log\frac{2}{\delta}
\right],
\nonumber
\end{align}
to ensure that $\mathbb{P}(\mathsf{TV}(f^*,\widehat{f}) \leq \epsilon) \ge 1 - \delta$. The final step of the proof is to demonstrate that having $\mathsf{TV}(f^*,\widehat{f}) \leq \epsilon$ also guarantees $\mathsf{TV}(f^* * G, \widehat{f} * G) \leq \epsilon$.
\begin{lemma}
\label{lemma_convolution_inequality}
For any trio of probability densities $p, q, r \in L^2(\mathcal{X})$ supported over $\mathcal{X}$, we have
$$
\mathsf{TV}(p * r, q * r) \leq \mathsf{TV}(p, q).
$$
\end{lemma}
\begin{proof}
Based on the definition of $d$-dimensional convolution, and the equivalence between TV distance and $\left\Vert\cdot\right\Vert_1/2$, we have
\begin{align}
\mathsf{TV}\left(p*r,q*r\right) &=
\int_{\mathcal{X}}\left\vert
(p*r)(\boldsymbol{t})-(q*r)(\boldsymbol{t})
\right\vert\mathrm{d}\boldsymbol{t}
\nonumber\\
&=
\frac{1}{2}
\int_{\mathcal{X}}\left\vert
\int_{\mathcal{X}}
\left(p(\boldsymbol{t})-q(\boldsymbol{t})\right)
r(\boldsymbol{u}-\boldsymbol{t})\mathrm{d}\boldsymbol{u}
\right\vert\mathrm{d}\boldsymbol{t}
\nonumber\\
&\stackrel{(i)}{\leq}
\frac{1}{2}
\int_{\mathcal{X}}
\int_{\mathcal{X}}
\left\vert
p(\boldsymbol{t})-q(\boldsymbol{t})
\right\vert
r(\boldsymbol{u}-\boldsymbol{t})\mathrm{d}\boldsymbol{u}
\mathrm{d}\boldsymbol{t}
\nonumber\\
&=
\frac{1}{2}
\int_{\mathcal{X}}
\left\vert
p(\boldsymbol{t})-q(\boldsymbol{t})
\right\vert
\left(
\int_{\mathcal{X}}
r(\boldsymbol{u}-\boldsymbol{t})\mathrm{d}\boldsymbol{u}
\right)
\mathrm{d}\boldsymbol{t}
\end{align}
where (i) is due to triangle inequality. The final term equals to $\mathsf{TV}(p, q)$ since we have
$$
\int_{\mathcal{X}}
r(\boldsymbol{u}-\boldsymbol{t})\mathrm{d}\boldsymbol{u}=1,\quad\forall\boldsymbol{t}\in\mathcal{X}.
$$
This argument proves the bound.
\end{proof}
We showed that using the claimed sample compression scheme, one can guarantee $\mathbb{P}(\mathsf{TV}(f^**G,\widehat{f}*G)\leq\epsilon)\ge1-\delta$, which completes the proof.
\end{proof}

%% file: Proof_of_Gaussian.tex
\begin{proof}[Proof of Theorem \ref{main_theorem}]
\label{proof_of_gaussian_case}

We use the result of Proposition \ref{sample_compression_theorem}, which assuming the $(\tau,t,m)$-sample compressibility of $\mathcal{F}$ guarantees the sample compressibility of $\mathcal{F}*G$. Specifically, for any $(\epsilon, \delta) \in (0,1)$, $\mathcal{F}*G$ admits 
\begin{equation}
\left[
\tau\left(\frac{\epsilon}{2}\right),
t\left(\frac{\epsilon}{2}\right) +
d\tau\left(\frac{\epsilon}{2}\right)
\log_2\left(1+
\frac{r\sqrt{d\tau\left({\epsilon}/{2}\right)}}{\epsilon}
\left|\Phi_G^{-1}\left(\frac{\delta}{4dm\left(\frac{\epsilon}{2}\right)\log\frac{2}{\delta}}\right)\right|
\right),
m\left(\frac{\epsilon}{2}\right)\log\frac{2}{\delta}
\right]
\nonumber
\end{equation}
-sample compression. Therefore, there exists a decoder $\mathcal{J}$ such that given $m\left(\frac{\epsilon}{2}\right)\log\frac{2}{\delta}$ i.i.d. samples from any $f^**G \in \mathcal{F}*G$, the decoder outputs $\widehat{f}*G \in \mathcal{F}*G$ that satisfies $\mathsf{TV}(f^**G,\widehat{f}*G)\leq\epsilon$ with probability at least $1-\delta$.

The next step is to use a seminal proposition from \cite{ashtiani2018nearly}, which combines Theorem \ref{existence_of_algorithm_with_samples} (originally from \cite{DevroyeLugosi2001}) and a number of concentration inequalities 
in order to prove the information-theoretic learnability of \emph{sample compressible} distribution classes in general. The following theorem, which is full version of Theorem \ref{thm:intro:sample_complexity_sc} establishes a fundamental theoretical connection between $(\tau,t,m)$-sample compressibility and PAC-learnability of a distribution class:
\begin{theorem}[Full version of Theorem 3.5 in \cite{ashtiani2018nearly}]
\label{sample_complexity_sc}
Suppose $\mathcal{F}$ admits $(\tau, t, m)$-sample compression for some functions $\tau,t,m:\left(0,1\right)\to\mathbb{N}$. For any $\epsilon,\delta\in(0,1)$, let $\tau'(\epsilon) \triangleq \tau\left(\epsilon\right)+t\left(\epsilon\right)$. Also, define $N^{\mathsf{Clean}}_{\tau,t,m}$ as
\begin{align}
N^{\mathsf{Clean}}_{\tau,t,m}(\epsilon,\delta)
\triangleq
\mathcal{O}\left(
m\left(\frac{\epsilon}{6}\right) \log_3 \left( \frac{2}{\delta} \right) + 
\frac{32}{\epsilon^2}\left[ \tau'\left( \frac{\epsilon}{6} \right) \log\left( m\left(\frac{\epsilon}{6}\right) \log_3 \left( \frac{2}{\delta} \right) \right) + \log \left( \frac{6}{\delta} \right) \right]
\right).
\nonumber
\end{align}
Then, there exists a deterministic algorithm that by having $n\ge
N^{\mathsf{Clean}}_{\tau,t,m}\left(\epsilon,\delta\right)$ i.i.d. samples from any unknown distribution $f^* \in \mathcal{F}$, outputs $\widehat{f} \in \mathcal{F}$ where we have $\mathsf{TV}(f^*,\widehat{f}) \leq \epsilon$, with probability of at least $1-\delta$.
\end{theorem}

Combining the results from Proposition \ref{sample_compression_theorem} and Theorem \ref{sample_complexity_sc}, one can deduce that there exists a deterministic algorithm, that for any $\epsilon,\delta\in(0,1)$, upon having
\begin{align}
n \ge& 
N^{\mathsf{Clean}}_{\tau,t,m}(6\epsilon,\delta/2)
\\
&+
\mathcal{O}\left[
\frac{2d\tau\left(\epsilon\right)}{9\epsilon^2}
\log_2\left(
\frac{r\sqrt{d\tau\left(\epsilon\right)}}{2\epsilon}
\left|\Phi_G^{-1}\left(\frac{\delta}{8dm\left(\epsilon\right)\log\frac{4}{\delta}}\right)\right|
\right)
\log\left( m\left(\epsilon\right) \log_3 \left( \frac{4}{\delta} \right) \right) 
\right]
\nonumber
\end{align}
i.i.d. samples from any $f^**G \in \mathcal{F}*G$, outputs $\widehat{f}*G \in \mathcal{F}*G$ which is an $12\epsilon$-approximation 
of $f^**G$ with probability at least $1-\delta$. Mathematically, we have:
\begin{equation}
\label{tv_relation_of_noisy_distribution}
\mathbb{P}(\mathsf{TV}(f^**G,\widehat{f}*G)\leq 12\epsilon) \ge 1-\delta,
\end{equation}
where $\mathbb{P}(\cdot)$ is with respect to the randomness of generating samples from $f^**G$. The final step is to establish an upper bound for $\Vert \widehat{f} - f^* \Vert_2$ with probability at least $1 - \delta$, leveraging \eqref{tv_relation_of_noisy_distribution}. As discussed in Section \ref{sec:Gaussian}, such a relationship does not necessarily hold, since convolution with \emph{low-frequency} noise densities (e.g., Gaussian noise) attenuates high-frequency components of the density difference $\widehat{f} - f^*$. Consequently, ensuring that $\mathsf{TV}(f^* * G, \widehat{f} * G)$ is small does not directly imply that $\widehat{f}$ is close to $f^*$ in a general sense.  

However, under Assumption \ref{class_assumption_2}, we have assumed that density differences in $\mathcal{F}$ retain a non-negligible portion of their energy in low-frequency regions. This ensures that recovering $f^* * G$ also leads to recovering $f^*$ itself. The following lemma formally establishes this fact:
\begin{lemma}
\label{lemma_relation_h_norm_2}
For any $\epsilon, \delta \in (0, 1)$, assume that there exists an algorithm such that \eqref{tv_relation_of_noisy_distribution} holds for the output distribution $\widehat{f}$. Also, assume that Assumption \ref{class_assumption_2} holds for a set of pairs $\mathsf{P}(\mathcal{F})=\left\{(\alpha,\xi)\right\}$. Then, with probability of at least $1-\delta$, the following bound holds uniformly for all $(\alpha,\xi)\in\mathsf{P}(\mathcal{F})$:
\begin{equation}
\Vert \widehat{f}-f^* \Vert_2 \leq \frac{24 \epsilon}{
\sqrt{B_G(\alpha)(1-\xi)}
}.
\end{equation}
\end{lemma}
\begin{proof}
    Based on \eqref{tv_relation_of_noisy_distribution} and the equivalence between TV distance and $\left\Vert\cdot\right\Vert_1/2$, with probability at least $1-\delta$, we have:
    \begin{equation}
        \int_{\reals^d} \left\vert f^**G(\boldsymbol{t})-\widehat{f}*G(\boldsymbol{t})\right\vert\mathrm{d}\boldsymbol{t} \leq 24\epsilon. 
    \end{equation}
    Let $P$ and $Q$ denote the Fourier transforms of $f^*$ and $\widehat{f}$, respectively. Using the fact that convolution operator turns into point-wise multiplication in the Fourier domain, and also applying Parseval's theorem, we obtain the following chain of relations for any $\alpha\ge0$:
    \begin{align}
       576\epsilon^2 &\geq \left(\int_{\reals^d} \left\vert f^**G(\boldsymbol{t}) - \widehat{f}*G(\boldsymbol{t})\right\vert \mathrm{d}\boldsymbol{t}\right)^2 
       \nonumber \\
       &{\ge} 
       \int_{\reals^d} 
       \left\vert f^**G(\boldsymbol{t}) - \widehat{f}*G(\boldsymbol{t})
       \right\vert^2
       \mathrm{d}\boldsymbol{t} 
       \nonumber \\
        &\stackrel{\mathrm{(i)}}{=} 
        \int_{\reals^d} 
        \left\vert 
        P(\boldsymbol{\omega}) 
        G(\boldsymbol{\omega}) 
        - 
        Q(\boldsymbol{\omega})
        G(\boldsymbol{\omega})
        \right\vert^2
        \mathrm{d}\boldsymbol{\omega} 
        \nonumber \\
        &= 
        \int_{\reals^d} 
        \left\vert 
        P(\boldsymbol{\omega}) - Q(\boldsymbol{\omega})
        \right\vert^2
        G(\boldsymbol{\omega})
        \mathrm{d}\boldsymbol{\omega} 
        \nonumber \\
        &\ge 
        \inf_{\left\Vert\boldsymbol{\omega}\right\Vert_2\leq\alpha}
        \left\Vert G(\boldsymbol{\omega})\right\Vert^2
        \int_{\left\Vert\boldsymbol{\omega}\right\Vert_2 \leq \alpha} \left\vert P(\boldsymbol{\omega}) - Q(\boldsymbol{\omega})\right\vert^2
        \mathrm{d}\boldsymbol{\omega},
    \end{align}
    where (i) is due to Parseval's theorem. Therefore, the following bound has been achieved:
    \begin{equation}
    \label{concentration_fourier_bound}
        \int_{\left\Vert\boldsymbol{\omega}\right\Vert_2 \leq \alpha} \left\vert 
        P(\boldsymbol{\omega})-Q(\boldsymbol{\omega})
        \right\vert^2
        \mathrm{d}\boldsymbol{\omega} 
        \leq 
        \frac{576\epsilon^2}{
        \inf_{\left\Vert\boldsymbol{\omega}\right\Vert_2\leq\alpha}
        \left\vert G(\boldsymbol{\omega})\right\vert^2
        }
        =
        \frac{576\epsilon^2}{B_G(\alpha)}.
    \end{equation}
    What remains is to use Assumption \ref{class_assumption_2}. Based on the definition of $\mathsf{P}(\mathcal{F})$, for any pair $\left(\alpha,\xi\right)\in\mathsf{P}(\mathcal{F})$ (where $\alpha\ge0$ and $\xi<1$), we have
    $$
    \int_{\left\Vert\boldsymbol{\omega}\right\Vert_2 \ge \alpha} \left\vert 
        P(\boldsymbol{\omega})-Q(\boldsymbol{\omega})
    \right\vert^2
    \mathrm{d}\boldsymbol{\omega}
    \leq
    \xi\Vert f^*-\widehat{f} \Vert_2^2.
    $$
    Therefore, the following set of relations hold:
    \begin{align}
        \Vert \widehat{f} - f^* \Vert_2^2 
        &= 
        \int_{\reals^d} 
        \left\vert 
        f^*(\boldsymbol{t}) - \widehat{f}(\boldsymbol{t})
        \right\vert^2 
        \mathrm{d}\boldsymbol{t} 
        \nonumber \\
        &= 
        \int_{\reals^d} 
        \left\vert 
        P(\boldsymbol{\omega}) -Q(\boldsymbol{\omega})
        \right\vert^2
        \mathrm{d}\boldsymbol{\omega} 
        \nonumber \\
        &=
        \int_{\left\Vert\boldsymbol{\omega}
        \right\Vert_2 \leq \alpha} 
        \left\vert 
        P(\boldsymbol{\omega})-Q(\boldsymbol{\omega})
        \right\vert^2\mathrm{d}\boldsymbol{\omega} 
        + 
        \int_{\left\Vert \boldsymbol{\omega}\right\Vert_2\geq \alpha} 
        \left\vert 
        P(\boldsymbol{\omega})-Q(\boldsymbol{\omega})
        \right\vert^2 
        \mathrm{d}\boldsymbol{\omega} 
        \nonumber \\
        &\stackrel{\mathrm{(i)}}{\leq} 
        \frac{576\epsilon^2}{B_G(\alpha)} + 
        \int_{\left\Vert \boldsymbol{\omega}\right\Vert_2^\geq \alpha} 
        \left\vert  
        P(\boldsymbol{\omega})-Q(\boldsymbol{\omega})\right\vert^2 
        \mathrm{d}\boldsymbol{\omega} 
        \nonumber \\
        &\stackrel{\mathrm{(ii)}}{\leq} 
        \frac{576\epsilon^2}{B_G(\alpha)} + 
        \xi \Vert \widehat{f} - f^* \Vert_2^2 
    \end{align}
    where (i) holds due to \eqref{concentration_fourier_bound}, and (ii) holds owing to Assumption \ref{class_assumption_2}. Note that $\xi$ can itself be a function of $\Vert \widehat{f} - f^* \Vert_2$. This argument proves the bound.
\end{proof}
Using Lemma \ref{lemma_relation_h_norm_2} and some simple algebra,
it has been shown that we can guarantee the existence of a deterministic algorithm that by using $n$ i.i.d. samples from $f^* * G$, outputs $\widehat{f}$ which holds in the following inequality:
\begin{equation}
\Vert \widehat{f} - f^* \Vert_2 \leq 
\frac{24\epsilon}{\sqrt{B_G(\alpha)(1-\xi)}},
\end{equation}
with probability at least $1-\delta$, uniformly over all pairs $(\alpha,\xi)\in\mathsf{P}(\mathcal{F})$. Therefore, the bound also holds for the infimum, i.e.,
\begin{align}
\mathbb{P}\left(
\Vert \widehat{f} - f^* \Vert_2 
\leq 
\inf_{(\alpha,\xi)\in\mathsf{P}(\mathcal{F})}
\frac{24\epsilon}{\sqrt{B_G(\alpha)(1-\xi)}}
\right)\ge1-\delta.
\end{align}
In case $\xi$ is not a constant, and instead is a function of $\Vert \widehat{f} - f^* \Vert_2$, the bound turns into the following probabilistic inequality:
\begin{align}
\mathbb{P}\left(
\Vert \widehat{f} - f^* \Vert_2 
\sqrt{1-\xi\left(\Vert \widehat{f} - f^* \Vert_2\right)}
\leq 
\frac{24\epsilon}{\sqrt{B_G(\alpha)}}
\bigg\vert~\forall(\alpha,\xi)\in\mathsf{P}(\mathcal{F})
\right)\ge1-\delta,
\end{align}
which completes the proof.
\end{proof}


\vspace{4mm}
\begin{proof}[Proof of Corollary \ref{corl:GaussianLaplaceMainThm}]

There are two major components in both the sample complexity and the final $\ell_2$ error that depend on the noise distribution ($G$): the component-wise noise CDF $\Phi_G(\cdot)$ and the quantity $B_G(\cdot)$. We first determine the order of the sample complexity for both the Gaussian and Laplace cases by substituting their respective exact formulations and proving the claimed expressions.  

\noindent
{\bf {Gaussian Noise}}: According to Remark \ref{remark:NoiseCDFFormula}, we have  
$
\left|\Phi_G^{-1}\left(\Delta\right)\right|=\sigma\mathcal{O}\left(\sqrt{\log\Delta^{-1}}\right)$, for $\Delta < 1/\sqrt{2\pi}.
$
Substituting $\Delta = \frac{\delta}{8dm\left(\epsilon\right)\log\frac{4}{\delta}}$, we obtain  
$$
\left|\Phi_G^{-1}\left(\frac{\delta}{8dm\left(\epsilon\right)\log\frac{4}{\delta}}\right)\right|
=
\sigma
\mathcal{O}\left(
\sqrt{
\log
\left(
\frac{dm\left(\epsilon\right)}{\delta}
\log\frac{1}{\delta}
\right)
}
\right)
=
\sigma
\mathcal{O}\left(
\sqrt{
\log
\frac{dm\left(\epsilon\right)}{\delta}
}
\right).
$$

\noindent
{ \bf{Laplace Noise}}: Again, based on Remark \ref{remark:NoiseCDFFormula}, we have  
$
\left|\Phi_G^{-1}\left(\Delta\right)\right|=b\mathcal{O}\left(\log\Delta^{-1}\right).
$
Substituting $\Delta = \frac{\delta}{8dm\left(\epsilon\right)\log\frac{4}{\delta}}$, this results in  
$$
\left|\Phi_G^{-1}\left(\frac{\delta}{8dm\left(\epsilon\right)\log\frac{4}{\delta}}\right)\right|
=
b
\mathcal{O}\left(
\log
\left(
\frac{dm\left(\epsilon\right)}{\delta}
\log\frac{1}{\delta}
\right)
\right)
=
b
\mathcal{O}\left(
\log
\frac{dm\left(\epsilon\right)}{\delta}
\right).
$$
Substituting these results into the sample complexity expression in Theorem \ref{main_theorem}, and setting $\lambda = \sigma$ for $G = \mathcal{N}(\boldsymbol{0},\sigma^2\boldsymbol{I}_d)$ and $\lambda = b$ for Laplace noise, we obtain the claimed order-wise sample complexity. Next, we derive explicit formulations for the final $\ell_2$ errors $\Vert f^*-\widehat{f} \Vert_2$ in both scenarios.
\\[1mm]
\noindent
{\bf {Gaussian Noise}}:
Using the Fourier transform of an isotropic Gaussian density, we have  
\begin{align}
\mathsf{F}\left\{G\right\}(\boldsymbol{\omega})
=
\mathsf{F}\left\{
\frac{1}{(2\pi\sigma^2)^{d/2}}
e^{-\Vert \boldsymbol{\zeta}\Vert^2_2/2}
\right\}(\boldsymbol{\omega})
=
e^{-\sigma^2\Vert\boldsymbol{\omega} \Vert^2_2/2}.
\end{align}  
Thus, the explicit formulation for $B_G(\cdot)$ is given by  
\begin{align}
B_G(\alpha) \triangleq
\inf_{\left\Vert\boldsymbol{\omega}\right\Vert_2\leq\alpha}
\left\vert 
\mathsf{F}\left\{G\right\}(\boldsymbol{\omega})
\right\vert
=
e^{-(\sigma\alpha)^2/2}.
\end{align}  
Substituting this into the final result of Theorem \ref{main_theorem} yields result (i) in Corollary \ref{corl:GaussianLaplaceMainThm}.
\\[1mm]
\noindent
{ \bf{Laplace Noise}}:  
Following a similar approach for multivariate Laplace noise with independent and identically distributed components, and using the fact that the Fourier transform factorizes over independent dimensions, we derive the following formulation for $B_G(\cdot)$:  
\begin{align}
B_G(\alpha)
&\triangleq
\inf_{\left\Vert\boldsymbol{\omega}\right\Vert_2\leq\alpha}
\left\vert 
\mathsf{F}\left\{G\right\}(\boldsymbol{\omega})
\right\vert
\nonumber\\
&\stackrel{\mathrm{(i)}}{=}
\inf_{\Vert\boldsymbol{\omega}\Vert_2\leq\alpha}~
\prod_{i=1}^{d}\frac{1}{1+b^2\omega^2_i}
\nonumber\\
&\stackrel{\mathrm{(ii)}}{=}
\left(1+\frac{(b\alpha)^2}{d}\right)^{-d},
\end{align}  
where (i) follows from the fact that  
$$
\left\vert 
\mathsf{F}\left\{
G(\boldsymbol{\zeta})
=
(2b)^{-d}
\prod_{i=1}^{d}
e^{-\vert\zeta_i\vert/b}
\right\}(\boldsymbol{\omega})
\right\vert
=
\prod_{i=1}^{d}
\frac{1}{1+b^2\zeta_i^2},
$$  
and (ii) holds since the infimum is attained when $\zeta_i^2 = \alpha^2/d$ for each $i \in [d]$. This completes the proof.  
\end{proof}

\vspace{4mm}

\begin{proof}[Proof of Proposition \ref{proposition:WaterFilling}]

We seek to upper bound the optimal value of the following functional optimization problem: 
\begin{align}
\sup_{f\in L^2(\mathcal{X})}~
\int\vert f\vert
\quad\mathrm{subject~to}\quad
f(\boldsymbol{x})\leq g(\boldsymbol{x}),~\forall\boldsymbol{x}\in\mathcal{X},
\quad\mathrm{and}\quad
\int f^2=\varepsilon^2,
\end{align} 
where $\varepsilon \triangleq \Vert f \Vert_2$ is assumed to be fixed and given. Since the objective is symmetric in $f$, the supremum is achieved when $f \ge 0$. Thus, without loss of generality, we may restrict the feasible set to non-negative functions. The problem becomes: 
\begin{align}
\inf_{f\in L^2(\mathcal{X})}~
-\int f
\quad\mathrm{subject~to}\quad
0\leq
f(\boldsymbol{x})\leq g(\boldsymbol{x}),~\forall\boldsymbol{x}\in\mathcal{X},
\quad\mathrm{and}\quad
\int f^2\leq\varepsilon^2.
\end{align}
This is a convex optimization problem: the objective is linear, the inequality constraints are convex (box constraints), and the quadratic equality constraint defines a convex level set. Hence, by standard results in convex analysis (e.g., Slater’s condition), strong duality holds, and the optimal value can be characterized via the Karush–Kuhn–Tucker (KKT) conditions \cite{Boyd2004}. We define the Lagrangian functional as follows:
\begin{align}
\mathcal{L}(f,\lambda_1,\lambda_2,\nu)\triangleq
-\int f
-\int \lambda_1f
+\int \lambda_2(f-g)
+\nu\left(\int f^2-\varepsilon^2\right),\quad
\forall f\in L^2(\mathcal{X}),
\end{align}
where $\lambda_1(\boldsymbol{x})$, $\lambda_2(\boldsymbol{x}) \ge 0$ are dual variables for the pointwise lower and upper bound constraints respectively, and $\nu \ge 0$ is the dual variable for the quadratic constraint. The stationarity condition with respect to $f$ yields: 
\begin{align} 
\nabla_f\mathcal{L}(\boldsymbol{x}) = -1 - \lambda_1(\boldsymbol{x}) + \lambda_2(\boldsymbol{x}) + 2 \nu f(\boldsymbol{x}) = 0, 
\quad\forall\boldsymbol{x}\in\mathcal{X},
\end{align} 
which leads to the expression for the optimizer: 
\begin{align} 
\text{(i)}\quad
f^*(\boldsymbol{x}) 
= 
\frac{1 + \lambda_1(\boldsymbol{x}) - \lambda_2(\boldsymbol{x})}{2\nu}. 
\end{align}
The KKT conditions also include: 
\begin{align} 
\text{(ii)} &\quad \lambda_1(\boldsymbol{x}) \ge 0, \quad \lambda_2(\boldsymbol{x}) \ge 0, \quad \nu \ge 0 \nonumber\\
\text{(iii)} &\quad \lambda_1(\boldsymbol{x}) f^*(\boldsymbol{x}) = 0, \quad \lambda_2(\boldsymbol{x}) (f^*(\boldsymbol{x}) - g(\boldsymbol{x})) = 0 
\nonumber\\ 
\text{(iv)} &\quad \nu \left( \int_{\mathcal{X}} f^{*2}(\boldsymbol{x}) \mathrm{d}\boldsymbol{x} - \varepsilon^2 \right) = 0. 
\end{align}
From these conditions, we can deduce the following structural properties of $f^*$: a) On the set where $f^*(\boldsymbol{x}) < g(\boldsymbol{x})$, the upper bound is inactive, so $\lambda_2(\boldsymbol{x}) = 0$, implying $f^*(\boldsymbol{x}) = \frac{1 + \lambda_1(\boldsymbol{x})}{2\nu}$. But by complementary slackness, if $f^*(\boldsymbol{x}) > 0$, then $\lambda_1(\boldsymbol{x}) = 0$, and thus $f^*(\boldsymbol{x}) = \frac{1}{2\nu}$.
b) On the set where $f^*(\boldsymbol{x}) = g(\boldsymbol{x})$, the upper bound is active, so $\lambda_2(\boldsymbol{x})$ may be non-zero.

Thus, the optimal solution $f^*$ takes the value $\min\left\{ \frac{1}{2\nu}, g(\boldsymbol{x}) \right\}$ almost everywhere. This is equivalent to a “water-filling” procedure: fill the region under $g$ until the $\ell_2$ norm constraint $\|f^*\|_2 = \varepsilon$ is met. Therefore, the function $f^*$ is the outcome of water-filling the subgraph of $g$ up to a certain level which satisfies the energy constraint. The total mass (i.e., $\ell_1$ norm) of $f^*$ is maximized under this procedure, yielding the bound stated in the proposition and the proof is complete.
\end{proof}

\vspace{4mm}

\begin{proof}[Proof of Corollary \ref{corl:noise:TVfromL2}]
We consider two cases: i) when $\text{supp}(f^* - \widehat{f}) \subseteq [-R, R]^d$, the result follows directly from the Cauchy–Schwarz inequality:
\begin{align}
\mathsf{TV}(\widehat{f}, f^*)^2
= 
\left(\int_{\mathcal{X}} |f^*(\boldsymbol{x}) - \widehat{f}(\boldsymbol{x})| \, \mathrm{d}\boldsymbol{x} \right)^2
&\leq 
\left( \int_{\mathcal{X}} (f^*(\boldsymbol{x}) - \widehat{f}(\boldsymbol{x}))^2 \, \mathrm{d}\boldsymbol{x} \right)
\cdot \left( \int_{[-R, R]^d} 1 \, \mathrm{d}\boldsymbol{x} \right)
\nonumber\\
&= 
(2R)^d \| \widehat{f} - f^* \|_2^2.
\end{align}
This establishes the bound in \eqref{eq:propBoundednessTVfromL2:Bound1}.

ii) Consider the case where $f^*$ is upper bounded by a Gaussian envelope, i.e., $g(\boldsymbol{x}) \triangleq C_1 \exp\left(-\gamma \|\boldsymbol{x}\|_2^2\right)$ for all $\boldsymbol{x} \in \mathcal{X} \subseteq \mathbb{R}^d$. Without loss of generality, assume the center of the Gaussian is at the origin (i.e., $\boldsymbol{\mu} = 0$ for the particular $f$ under consideration). According to Proposition \ref{proposition:WaterFilling}, the $\ell_1$-optimal function $f^*$ under the $\ell_2$ constraint $\| f^* - \widehat{f} \|_2 = \varepsilon$ is obtained by truncating $g$ at height $1/(2\nu)$ to form the function $f^*(\boldsymbol{x}) = \min\left\{ \frac{1}{2\nu}, g(\boldsymbol{x}) \right\}$. This yields a superlevel set $A(\varepsilon, g) \triangleq \{ \boldsymbol{x} \in \mathcal{X} : g(\boldsymbol{x}) \ge \frac{1}{2\nu} \}$ which, due to radial symmetry, is a Euclidean ball centered at the origin with radius $R$. The radius $R$ must satisfy the constraint on the $\ell_2$ norm:
\begin{align}
\label{eq:proof:proposition:defineR}
R^d \mathrm{Vol}\left( \mathbb{B}^d_2(1) \right)
+
C_1^2 \mathrm{Vol}\left( \mathbb{B}^{d-1}_2(1) \right)
\int_{R}^{\infty} r^{d-1} e^{-2\gamma r^2} \, dr
= \varepsilon^2,
\end{align}
where $\mathbb{B}_2^d(1)$ denotes the unit $\ell_2$-ball in $\mathbb{R}^d$, and its volume is given by
$$
\mathrm{Vol}\left( \mathbb{B}_2^d(1) \right) = \frac{\pi^{d/2}}{\Gamma\left( \frac{d}{2} + 1 \right)}.
$$

The integral in \eqref{eq:proof:proposition:defineR} involves the \emph{incomplete gamma function} and, in general, does not admit a closed-form solution. However, from standard asymptotics for the tail of the Gaussian integral (see, e.g., \cite{temme1996special}), one can show:
$$
R \ge \mathcal{O}\left( \sqrt{\frac{1}{\gamma} \log \frac{1}{\varepsilon}} \right),
$$
where the hidden constant depends on $C_1$ and the dimension $d$. Substituting this lower bound on $R$ into the expression for the total variation bound given by Proposition \ref{proposition:WaterFilling}, we conclude the desired result and completes the proof.
\end{proof}

%% file: lipschitz_gaussian_claim_proof.tex
\begin{proof}[Proof of claim \ref{claim_lipschitz_gaussian}]

    The Gaussian family in general, and the isotropic axis-aligned Gaussian family in particular, are known to be sample compressible \cite{ashtiani2018nearly}. It has been shown that the following decoder $\mathcal{J}$ can achieve the above-mentioned sample compression scheme: The decoder $\mathcal{J}$ maps a sample sequence $\mathbf{L} = \{ \boldsymbol{X}_1, ..., \boldsymbol{X}_{\tau(\epsilon)} \}$ and bits $\boldsymbol{B}$ to a Gaussian distribution $\mathcal{N}(\widehat{\boldsymbol{\mu}}, \widehat{\sigma}^2 \boldsymbol{I}_d)$, where
    $$\widehat{\boldsymbol{\mu}} = \frac{1}{\tau(\epsilon)} \sum_{i=1}^{\tau(\epsilon)} \boldsymbol{X}_i$$
    and 
    $$\widehat{\sigma}^2 = \max \left \{ \sigma_0^2~,~ \frac{1}{\tau(\epsilon)} \sum_{i=1}^{\tau(\epsilon)} \left \Vert \boldsymbol{X}_i - \widehat{\boldsymbol{\mu}} \right \Vert^2 \right\}.$$

    In this regard, let $\mathbf{L}$ and $\mathbf{L}'$ be two sample sequences with $\left \Vert \mathbf{L} - \mathbf{L}' \right \Vert_2 \leq \Delta$. Without loss of generality, for each sample $\boldsymbol{X}_i \in \mathbf{L}$, let $\boldsymbol{X}'_i \in \mathbf{L}'$ satisfy $\left \Vert \boldsymbol{X}_i - \boldsymbol{X}'_i \right \Vert_2 \leq \delta_i$, where $\sum_{i=1}^{\tau(\epsilon)} \delta_i^2 = \Delta^2$. Then, the difference in means is bounded by:
    \begin{align}
        \left \Vert \widehat{\boldsymbol{\mu}} - \widehat{\boldsymbol{\mu}}' \right \Vert_2 
        &\leq 
        \frac{1}{\tau(\epsilon)} \sum_{i=1}^{\tau(\epsilon)} \left \Vert \boldsymbol{X}_i - \boldsymbol{X}'_i \right \Vert_2 \nonumber\\
        &\leq 
        \sqrt{\frac{1}{\tau(\epsilon)} \sum_{i=1}^{\tau(\epsilon)} \left \Vert \boldsymbol{X}_i - \boldsymbol{X}'_i \right \Vert_2^2} 
        \nonumber\\
        &=
        \frac{1}{\sqrt{\tau(\epsilon)}} \left \Vert \mathbf{L} - \mathbf{L}' \right \Vert_2 
        \nonumber\\
        &\leq
        \frac{\Delta}{\sqrt{\tau(\epsilon)}}.
    \end{align}
    On the other hand, the difference between true and empirical component-wise variances, i.e., $\sigma^2$ and $\widehat{\sigma}^2$, satisfies:
    \begin{align}
        \left \vert \widehat{\sigma}^2 - (\widehat{\sigma}')^2 \right \vert 
        \leq \left \vert 
        \frac{1}{\tau(\epsilon)} \sum_{i=1}^{\tau(\epsilon)} \left( \left \Vert \boldsymbol{X}_i - \widehat{\boldsymbol{\mu}} \right \Vert^2_2 - \left \Vert \boldsymbol{X}'_i - \widehat{\boldsymbol{\mu}}' \right \Vert^2_2 \right)
        \right \vert,
    \end{align}
    where
        $\left \Vert \boldsymbol{X}_i - \widehat{\boldsymbol{\mu}} \right \Vert^2_2 - \left \Vert \boldsymbol{X}'_i - \widehat{\boldsymbol{\mu}}' \right \Vert^2_2 = 
        \left \Vert \boldsymbol{X}'_i - \boldsymbol{X}_i \right \Vert^2_2 + \left \Vert \widehat{\boldsymbol{\mu}} - \widehat{\boldsymbol{\mu}}' \right \Vert^2_2$.
    As a result we have:
    \begin{align}
        \left \vert \widehat{\sigma}^2 - (\widehat{\sigma}')^2 \right \vert 
        &\leq
        \left \vert 
        \frac{1}{\tau(\epsilon)} \sum_{i=1}^{\tau(\epsilon)} \left( \left \Vert \boldsymbol{X}'_i - \boldsymbol{X}_i \right \Vert^2_2 + \left \Vert \widehat{\boldsymbol{\mu}} - \widehat{\boldsymbol{\mu}}' \right \Vert^2_2
        \right)
        \right \vert 
        \nonumber\\
        &\leq
        \frac{\Delta^2}{\tau(\epsilon)} + \frac{\Delta^2}{\tau(\epsilon)^2} 
        \nonumber\\
        &\leq 
        \mathcal{O}\left(\frac{\Delta^2}{\tau(\epsilon)}\right).
    \end{align}

    For two isotropic Gaussian distributions $\mathcal{N}\left({\boldsymbol{\mu}}, {\sigma}^2 \boldsymbol{I}_d\right)$ and $\mathcal{N}\left({\boldsymbol{\mu}}', ({\sigma}')^2 \boldsymbol{I}_d\right)$, the TV distance is known to obey the following upper-bound
    \begin{align}
        \mathsf{TV}\left(
        \mathcal{N}({\boldsymbol{\mu}}, {\sigma}^2 \boldsymbol{I}_d), 
        \mathcal{N}({\boldsymbol{\mu}}', ({\sigma}')^2 \boldsymbol{I}_d)
        \right) 
        \leq 
        \frac{1}{2}\left(
        \frac{
        \left \Vert \mu - \mu' \right \Vert_2
        }{
        \min \left\{ \sigma, \sigma'\right\}
        } 
        +
        \frac{
        \left \vert \sigma^2 - (\sigma')^2 \right \vert
        }{ 
        \min \left\{ \sigma^2, (\sigma')^2\right\}
        }
        \right).
    \end{align}
    Substituting into the above results, we have
    \begin{align}
        \mathsf{TV}
        \left(
        \mathcal{J}(\mathbf{L}, \mathbf{B}), 
        \mathcal{J}(\mathbf{L}', \mathbf{B})
        \right) 
        &\leq
        \frac{\Delta}{2\sigma_0\sqrt{\tau(\epsilon)}} + 
        \mathcal{O}\left(\frac{\Delta^2}{\tau(\epsilon)}\right) 
        \nonumber\\
        &\leq
        \mathcal{O} \left(\frac{1}{\sigma_0 \sqrt{\tau(\epsilon)}}\right) \Delta
    \end{align}
    On the other hand,
    according to \cite{ashtiani2018nearly}, for Gaussian distribution family we have $\tau(\epsilon) = \min\left\{2,\mathcal{O}(d \log (2d))\right\}$, which gives us the following result:
    \begin{align}
        \mathsf{TV}
        \left(
        \mathcal{J}(\mathbf{L}, \mathbf{B}), 
        \mathcal{J}(\mathbf{L}', \mathbf{B})
        \right) 
        \leq
        \mathcal{O}
        \left(
        \frac{1}{\sigma_0 \sqrt{d \log (2d)}}
        \right) 
        \left \Vert \mathbf{L} - \mathbf{L}' \right \Vert_2,
    \end{align}
    and the proof is complete.
\end{proof}

%% file: lipschitz_uniform_claim_proof.tex
\begin{proof}[Proof of claim \ref{claim_lipschitz_uniform}]
Consider the following decoder for this class of distributions: The decoder $\mathcal{J}$, given a sample sequence $\mathbf{L} = \{ \boldsymbol{X}_1, ..., \boldsymbol{X}_{\tau(\epsilon)} \}$, computes the empirical \emph{minimum} and \emph{maximum} of the samples alongside each dimension $i \in [d]$ as follows:
\begin{equation}
\widehat{a}_i = \min_{j\in[\tau(\epsilon)]} \boldsymbol{X}_{j, i}
\quad, \quad
\widehat{b}_i = \max_{j\in[\tau(\epsilon)]} \boldsymbol{X}_{j, i},
\end{equation}
where $\boldsymbol{X}_{j, i}$ is the $i$th dimension of $\boldsymbol{X}_j$. Then, $\mathcal{J}$ outputs the following uniform distribution: 
$$
\mathsf{Uniform}\left(\prod_{i}\left[\widehat{a}_i,\widehat{b}_i\right]\right),
$$
which is another uniform measure over axis-aligned hyper-rectangles and thus belongs to $\mathcal{F}$. A simple analysis can reveal that this decoder achieves a $(\tau,t,m)$-sample compression scheme for $\mathcal{F}$. However, exact knowledge of the specific functions $\tau,t,m:(0,1)\to\mathbb{N}$ is not needed for this proof. We compute them later, in Section \ref{sec:examples}.

Similar to the proof of Claim \ref{claim_lipschitz_gaussian}, let $\mathbf{L}$ and $\mathbf{L}'$ be two sample sequences with $\left \Vert \mathbf{L} - \mathbf{L}' \right \Vert_2 = \Delta$. Without loss of generality, for each sample $\boldsymbol{X}_i \in \mathbf{L}$, let $\boldsymbol{X}'_i \in \mathbf{L}'$ satisfy $\left \Vert \boldsymbol{X}_i - \boldsymbol{X}'_i \right \Vert_2 \leq \delta_i$, where $\sum_{i=1}^{\tau(\epsilon)} \delta_i^2 = \Delta^2$. Hence, for each sample $j\in[\tau(\epsilon)]$ and dimension $i\in[d]$, the perturbation satisfies $\left \vert \boldsymbol{X}_{j, i} - \boldsymbol{X}_{j, i}' \right \vert \leq \Delta$. It should be noted that much tighter bounds can be attained here using a more detailed analysis, however, we have sacrificed tightness for the sake of brevity and readability of the proof.

The empirical min/max values corresponding to dimension $i$ in $\mathbf{L}$ and $\mathbf{L}'$ satisfy
     \begin{equation}
         \left \vert \widehat{a}_i - \widehat{a}_i' \right \vert \leq \Delta
         \quad, \quad
         \left \vert \widehat{b}_i - \widehat{b}_i' \right \vert \leq \Delta.
     \end{equation}
     Thus, the perturbed hyper-rectangle $\prod_{i}\left[\widehat{a}_i',\widehat{b}_i'\right]$ differs from the original by at most $\Delta$ in each endpoint of the boundary. For two uniform distributions 
$$
U = \mathsf{Uniform}\left(\prod_{i}\left[\widehat{a}_i,\widehat{b}_i\right]\right)
\quad\mathrm{and}\quad
U' = \mathsf{Uniform}\left(\prod_{i}\left[\widehat{a}_i',\widehat{b}_i'\right]\right),
$$
the TV distance is bounded as
     \begin{equation}
         \mathsf{TV}(U, U') \leq 1 - 
         \frac{
         \mathsf{Vol}_d\big(\mathrm{supp}(U) \cap \mathrm{supp}(U')\big)
         }{
         \mathsf{Vol}_d\big(\mathrm{supp}(U) \cup \mathrm{supp}(U')\big)},
     \end{equation}
     where $\mathrm{supp}(\cdot)$ denotes the support of a distribution, and $\mathsf{Vol}_d$ is the $d$-dimensional volume (i.e., Lebesgue measure) of a set.
     The symmetric difference in each dimension contributes additively. Hence, for a single dimension $i\in[d]$, the overlap loss is bounded by:
     \begin{align}
         1 - 
         \prod_{i=1}^{d}\left(
         \frac{
         \min \left(\widehat{b}_i, \widehat{b}_i'\right) 
         - 
         \max \left(\widehat{a}_i, \widehat{a}_i'\right)
         }{
         \max \left(\widehat{b}_i, \widehat{b}_i'\right) 
         - 
         \min \left(\widehat{a}_i, \widehat{a}_i'\right)
         } 
         \right)
         &\leq 
         1-\left(1-\frac{2\Delta}{T/2}\right)^d
         \nonumber\\
         &\leq
         \frac{4d\Delta}{T},
     \end{align}
     where we have used the fact that having assumed the samples in $\mathbf{L}$ have achieved the total variation of at most $<1/2$, we have $\widehat{b}_i-\widehat{a}_i\ge(b_i-a_i)/2\ge T/2$ for each $i\in[d]$. Therefore, we have
     \begin{equation}
         \mathsf{TV}(U, U') \leq \frac{4d\Delta}{T} = \frac{8d}{T} \times\frac{1}{2}\left \Vert \mathbf{L} - \mathbf{L}' \right \Vert_2,
     \end{equation}
     which completes the proof.
\end{proof}

%% file: lipschitz_mixture_claim_proof.tex
\begin{proof}[Proof of Claim \ref{claim_lipschitz_mixture}]
    We use the proof of Lemma 4.6 in \cite{ashtiani2018full}. Assume the class $\mathcal{F}$ admits a sample compression scheme using a decoder $\mathcal{J}$ that adheres to Assumption \ref{assumption_decoder}.
    For each $i \in [k]$, the chosen sequences $\mathbf{L}$ and $\mathbf{B}$ should contain subsets $\mathbf{L}_i$ and $\mathbf{B}_i$, which are used to estimate $f_i$ using decoder $\mathcal{J}$. In this regard, we can define the $k$-mixture decoder $\mathcal{J}_k$ for class $k\mathrm{-Mix}(\mathcal{F})$ as follows:
    \begin{equation}
        \mathcal{J}_k(\mathbf{L}, \mathbf{B}) = \sum_{i=1}^k \widehat{\alpha}_i \mathcal{J}(\mathbf{L}_i, \mathbf{B}_i)
    \end{equation}
    where $\widehat{\alpha}_i$ are the quantized weights (approximations of the true latent $\alpha_i$s) which are decoded using the bits in $\mathbf{B}$. With some abuse of notation, let
    $\mathbf{L} = (\mathbf{L}_1, ..., \mathbf{L}_k)$ and $\mathbf{L}' = (\mathbf{L}'_1, ..., \mathbf{L}'_k)$ where each $\mathbf{L}_i,\mathbf{L}'_i\in\mathcal{X}^{t(\epsilon)}$.  Hence, they can be viewed as $kdt(\epsilon)$-dimensional Euclidean vectors. Here, $\mathbf{L}$ is the sample sequence with $\mathsf{TV}(f^*,\mathcal{J}(\mathbf{L},\mathbf{B}))\leq 1/2$, and $\mathbf{L}'$ represents its perturbed version.
    Also, note that we have:
    \begin{equation}
        \left \Vert \mathbf{L} - \mathbf{L}' \right \Vert_2^2 = \sum_{i=1}^k \left \Vert \mathbf{L}_i - \mathbf{L}'_i \right \Vert_2^2.
    \end{equation}
According to the assumed properties of the class $\mathcal{F}$, for each component $i\in [k]$ of the latent mixture, we have:
    \begin{equation}
        \mathsf{TV}
        \left(
        \mathcal{J}(\mathbf{L}_i, \mathbf{B}_i), 
        \mathcal{J}(\mathbf{L}'_i, \mathbf{B}_i)
        \right) 
        \stackrel{a.s.}{\leq}
        \frac{r}{2} \left\Vert \mathbf{L}_i - \mathbf{L}'_i\right\Vert_2.
    \end{equation}
Due to the convexity of TV distance for mixtures and using Jensen's inequality \cite{Boyd2004}, we have:
\begin{align}
\mathsf{TV}
\left(
\mathcal{J}_k(\mathbf{L}, \mathbf{B}), 
\mathcal{J}_k(\mathbf{L}', \mathbf{B})
\right) 
&\leq 
\sum_{i=1}^k 
        \widehat{\alpha}_i
        \mathsf{TV}
        \left(
        \mathcal{J}(\mathbf{L}_i, \mathbf{B}_i), 
        \mathcal{J}(\mathbf{L}'_i, \mathbf{B}_i)
        \right) 
        \nonumber\\
        &\leq 
        \sum_{i=1}^k 
        \frac{r}{2}
        \widehat{\alpha}_i
        \left \Vert \mathbf{L}_i - \mathbf{L}'_i \right \Vert_2.
 \end{align}
Define
$\boldsymbol{a} \triangleq ({r}/{2})\left( \widehat{\alpha}_1,\dots,\widehat{\alpha}_k\right)$ 
and 
$\mathbf{B} \triangleq \left( \left \Vert \mathbf{L}_1 - \mathbf{L}'_1 \right \Vert_2, \dots, \left \Vert \mathbf{L}_k - \mathbf{L}'_k \right \Vert_2 \right)$. 
By Holder's inequality (according to Theorem 3.8 in \cite{Rudin1987}), we have:
    \begin{align}
        \frac{r}{2}\sum_{i=1}^k \widehat{\alpha}_i
        \left \Vert \mathbf{L}_i - \mathbf{L}'_i \right \Vert_2 
        &\leq \left\Vert \boldsymbol{a} \right\Vert_2 \cdot \left\Vert \mathbf{B} \right\Vert_2 
        \nonumber\\
        &= 
         \left\Vert \boldsymbol{a} \right\Vert_2 \cdot \left \Vert \mathbf{L} - \mathbf{L}' \right \Vert_2 \nonumber\\
        &\leq 
        \frac{1}{2}r \sqrt{k} \left \Vert \mathbf{L} - \mathbf{L}' \right \Vert_2.
    \end{align}
As a result, the following bound holds:
    \begin{equation}
        \mathsf{TV}
        \left(
        \mathcal{J}_k(\mathbf{L}, \mathbf{B}), 
        \mathcal{J}_k(\mathbf{L}', \mathbf{B})
        \right) \leq 
        \frac{1}{2}r \sqrt{k} \left \Vert \mathbf{L} - \mathbf{L}' \right \Vert_2,
    \end{equation}
    and the proof is complete.
\end{proof}

%% file: Claim_Proof_Impossibility.tex
\begin{proof}[Proof of Claim \ref{Impossibility_Decoder_Assumption}]
    
The proof consists of two parts: i) First, we show that $\mathcal{F}$ cannot be learned in a PAC manner in the sense of total variation (TV) distance. ii) Next, we show that $\mathcal{F}$ does not satisfy Assumption \ref{assumption_decoder}, even though it is sample compressible.

Define $ \mu^* \in \mathbb{R} $ and $ \sigma^* > 0 $ such that $ f^* = \mathcal{N}(\mu^*,(\sigma^*)^2) $. Thus, in order to learn the class $ \mathcal{F} $ of distributions
$\mathcal{F} =\{\mathcal{N}(\mu,\sigma^2) \mid \mu \in \mathbb{R}, ~ \sigma > 0 \} $. we need to estimate both $ \mu^* $ and $ \sigma^* $. When a series of independent Gaussian noise values distributed according to $ G = \mathcal{N}(0,\sigma_0^2) $ (for some $ \sigma_0 > 0 $) are added to the i.i.d. samples from $ f^* $, the resulting samples are equivalently drawn from $ \mathcal{N}(\mu^*, (\sigma^*)^2+\sigma_0^2) $.

Suppose there exists an algorithm (decoder for $ \mathcal{F} $) $ \mathscr{A} \in \mathsf{A}(n,B) $ that, using $ n $ noisy samples drawn from $ f^* * G $ and activating one of its $ 2^B $ internal states, can estimate an $ \epsilon $-approximation (in TV error sense) of $ f^* $ for any $ f^* \in \mathcal{F} $ with high probability. Here, we expect that $ \epsilon $ asymptotically decreases as $ n \to \infty $. Let us define $ \widehat{f} = \mathcal{N}(\widehat{\mu},\widehat{\sigma}^2) $ as the algorithm's estimation of $ f^* $. Based on Theorem 1.3 of \cite{Devroye2018TheTV}, if we assume $ \epsilon < \frac{1}{200} $, we have:
\begin{align}
\frac{| \mu^* - \widehat{\mu} |}{5\sigma^*} 
\leq 
\mathsf{TV}(f^*, \widehat{f}) 
\leq
\epsilon. 
\label{eq:minimaxLeCamEq1}
\end{align}
As a result, we must have $ | \mu^* - \widehat{\mu} | \leq 5 \epsilon \sigma^* $ as a necessary (but not sufficient) condition for the algorithm $ \mathscr{A} $ to output an $ \epsilon $-approximation. Therefore, the algorithm should also be capable of reliably solving the following two-point hypothesis testing problem:
\begin{itemize}
\item 
    Null hypothesis $H_0$: $\widehat{\mu} 
    \gets \mu_0\triangleq\mu^*$,
\item 
    Alternative hypothesis $H_1$: $\widehat{\mu} 
    \gets \mu_1\triangleq\mu^* + 10 \epsilon \sigma^*$.
\end{itemize}
Note that since $ \mu $ in $ \mathcal{F} $ is a continuous degree of freedom and can take any value in $ \mathbb{R} $, both values $ \mu^* $ and $ \mu^* + 10\epsilon\sigma^* $ can be chosen for $ f^* $. The KL divergence between $ \mathcal{N}(\mu_0, (\sigma^*)^2+\sigma_0^2) $ and $ \mathcal{N}(\mu_1, (\sigma^*)^2+\sigma_0^2) $ is:
\begin{align}
\mathsf{KL}(H_0 \Vert H_1) = \frac{(10 \epsilon \sigma^*)^2}{2(\sigma^*)^2+2\sigma_0^2} = \frac{50 \epsilon^2 (\sigma^*)^2}{(\sigma^*)^2+\sigma_0^2},
\end{align}
where with some abuse of notation we replaced the \emph{distribution} according to hypothesis $H_i$ with $H_i$ itself (for $i=0,1$). For $ n $ samples, due to the independence of the noisy samples, the total ($n$-sample) KL divergence is:
\begin{align}
\mathsf{KL}(H_0^n \Vert H_1^n) = n \frac{50 \epsilon^2 (\sigma^*)^2}{(\sigma^*)^2+\sigma_0^2}, 
\end{align}
where $ H^n_i $ denotes the product probability measure of $ n $ independent samples from the distribution associated with hypothesis $ H_i $. Using Pinsker’s inequality (Lemma 2.5 in \cite{tsybakov2009introduction}), the TV distance satisfies:
\begin{align}
\mathsf{TV}(H_0^n , H_1^n) \leq \sqrt{\frac{1}{2} \mathsf{KL}(H_0^n \Vert H_1^n)} = \sqrt{n \frac{25 \epsilon^2 (\sigma^*)^2}{(\sigma^*)^2+\sigma_0^2}} \leq \frac{5 \epsilon \sigma^*}{\sigma_0} \sqrt{n}. 
\end{align}
At this point, we can apply Le Cam's lemma (Lemma 2.3 in \cite{tsybakov2009introduction}), and lower-bound the minimum probability of $n$-sample misclassification between $H_0$ vs. $H_1$ as:
\begin{align}
\inf_{\mathscr{A}\in\mathsf{A}(n,B)} \big( \mathbb{P}_{H_0} ( | \widehat{\mu}_{\mathscr{A}} - \mu_0 | \ge 5\epsilon\sigma^* ) + \mathbb{P}_{H_1} ( | \widehat{\mu}_{\mathscr{A}} - \mu_1 | \ge 5\epsilon\sigma^* ) \big) \ge \frac{1}{2}(1-\mathsf{TV}(H_0^n , H_1^n)),
\end{align}
where $\widehat{\mu}_{\mathscr{A}}$ denotes the value of the mean $\widehat{\mu}$ returned by algorithm $\mathscr{A}$. Due to prior discussions around the inequalities in \eqref{eq:minimaxLeCamEq1}, we have the following:
    \begin{align}
        &\inf_{\mathscr{A}\in\mathsf{A}(n,B)}
        \left(
            \sup_{f^*\in\mathcal{F}}
            \mathbb{P}_{f^*} \left( 
            \mathsf{TV}(\widehat{f},f^*) 
            \ge \epsilon
            \right)
        \right)
        \nonumber\\
        \ge~
        &\inf_{\mathscr{A}\in\mathsf{A}(n,B)}
        \left(
        \mathbb{P}_{H_0} \left( 
            \mathsf{TV}(\widehat{f},f^*\gets H_0) 
        \ge \epsilon
        \right) + 
        \mathbb{P}_{H_1} \left( 
            \mathsf{TV}(\widehat{f},f^*\gets H_1) 
        \ge \epsilon 
        \right)
        \right)
        \nonumber\\
        \ge 
        &~\frac{1}{2}(1-\mathsf{TV}(H_0^n , H_1^n))
        \nonumber\\
        \ge&~
        \frac{1}{2}\left(1-\frac{5 \epsilon \sigma^*}{\sigma_0}\sqrt{n}\right).
    \end{align}
Since we can asymptotically decrease $ \sigma^* $ toward zero in $ \mathcal{F} $, the right-hand side of the above inequality can become arbitrarily close to $ 1/2 $. Hence, we have shown that for any $ \epsilon \leq 1/200 $, achieving such an error with probability at least $ 1/2 $ is impossible for any estimator $ \mathscr{A} $, regardless of how large $ n $ and $ B $ are.

The remaining task (i.e., part (ii) of the proof) is to show that this distribution family cannot satisfy Assumption \ref{assumption_decoder} for any finite $r \geq 0$. Recalling Definition \ref{sample_compression} of sample compression, consider $\mathbf{L}$ as the sequence of $\tau(\epsilon)$ samples chosen from the target distribution $f^* = \mathcal{N}(\mu^*, (\sigma^*)^2)$. We denote by $\mathbf{L}'$ the perturbed samples. Let $\mathcal{J}$ be any decoder for $\mathcal{F}$ that achieves a given sample compression scheme. Define
\begin{align}
    \widehat{\mu} &\triangleq \mu_{\mathcal{J}(\mathbf{L},\mathbf{B})}, 
    \nonumber\\
    \widehat{\mu}' &\triangleq \mu_{\mathcal{J}(\mathbf{L}',\mathbf{B})},
\end{align}
as the mean values returned by decoder $\mathcal{J}$ based on observing the sample sequence $\mathbf{L}$ and its perturbed version $\mathbf{L}'$, respectively.

We show that for any sequence of bits $\mathsf{B}$, even if the perturbed sample set $\mathbf{L}'$ is chosen arbitrarily close to $\mathbf{L}$, the total variation (TV) error term
\begin{align*}
    \mathsf{TV}\left(
        \mathcal{J}\left(\mathbf{L},\mathbf{B}\right),
        \mathcal{J}\left(\mathbf{L}',\mathbf{B}\right)
    \right)
\end{align*}
either becomes larger than any value $\leq 1/200$ for at least one $f^* \in \mathcal{F}$, or $\mu_{\mathcal{J}(\mathbf{L},\mathbf{B})}$ must be constant with respect to $\mathbf{L}$, contradicting the assumption that $\mathcal{J}$ is a decoder for $\mathcal{F}$.

To establish this, using \eqref{eq:minimaxLeCamEq1}, we obtain
\begin{align}
    \mathsf{TV}\left(
        \mathcal{J}\left(\mathbf{L},\mathbf{B}\right),
        \mathcal{J}\left(\mathbf{L}',\mathbf{B}\right)
    \right)
    \geq \frac{\left\vert \widehat{\mu} - \widehat{\mu}' \right\vert}{5\sigma^*},
\end{align}
under the assumption that
$
    \mathsf{TV}\left(
        \mathcal{J}\left(\mathbf{L},\mathbf{B}\right),
        \mathcal{J}\left(\mathbf{L}',\mathbf{B}\right)
    \right) < \frac{1}{200}
$.

Now, if $\left\vert \widehat{\mu} - \widehat{\mu}' \right\vert > 0$, one can choose $\sigma^*$ sufficiently small such that the TV distance does not fall below $1/200$. On the other hand, if $\left\vert \widehat{\mu} - \widehat{\mu}' \right\vert = 0$, since no specific assumptions were made regarding $\mathbf{L}'$, it follows that $\mu_{\mathcal{J}(\mathbf{L},\mathbf{B})}$ is independent of the samples in $\mathbf{L}$ and thus is a constant. Consequently, $\mathcal{J}$ cannot be a decoder, completing the proof.

\end{proof}

%% file: Claim_proof_low_frequency.tex
\begin{proof}[Proof of Claim \ref{impossibility_assumption2_example}]
    First, we prove $\mathcal{F}$ does not satisfy Assumption \ref{class_assumption_2}. The proof is based on contradiction. Suppose that there exists $\alpha\ge0$ and $\xi<1$ such that the assumption holds. Let $p(x) = ({1 + \sin \left(kx\right)})/({2\pi})$ and $q(x) = ({1 - \sin \left(kx\right)})/({2\pi})$ for some $k\in\mathbb{Z}_{\ge0}$. Also, assume $P,Q:\reals\to\mathbb{C}$ represent their respective Fourier transforms. Also, note that we have
    \begin{align}
    P(\omega)-Q(\omega)&=
    \int_{0}^{2\pi}
    \frac{1}{2\pi}
    \left(1+\sin(kx)-1+\sin(kx)\right)e^{-i\omega x}\mathrm{d}x
    \nonumber\\
    &=
    \frac{1}{\pi}
    \int_{0}^{2\pi}
    \sin(kx)e^{-i\omega x}\mathrm{d}x
    \nonumber\\
    &=
    \frac{1}{2\pi}
    \int_{0}^{2\pi}
    \left(e^{ikx}-e^{-ikx}\right)e^{-i\omega x}\mathrm{d}x
    \nonumber\\
    &=
    -\frac{2k\sin(\pi\omega)}{\pi\left(k^2-\omega^2\right)}
    e^{-i\pi\omega}.
    \end{align}
    As observed, $ P(\omega) - Q(\omega) $ becomes singular at $ \omega = \pm k $, indicating that most of its energy (i.e., its $\ell_2$-norm) is concentrated around $ \pm k $. Consequently, increasing $ k $ shifts the majority of the energy in the frequency domain away from the origin or any bounded $\alpha$-neighborhood of the origin. The following argument provides a formal mathematical justification:
    \begin{align}
    \int_{\left\vert\omega\right\vert\ge\alpha}
    \left\vert P(\omega)-Q(\omega)\right\vert^2
    \mathrm{d}\omega
    &=
    \left\Vert p-q\right\Vert_2^2-\int_{-\alpha}^{\alpha}
     \left\vert P(\omega)-Q(\omega)\right\vert^2
    \mathrm{d}\omega
    \nonumber\\
    &=
    \left\Vert p-q\right\Vert_2^2
    -\frac{4}{\pi^2}\int_{-\alpha}^{\alpha}
    \frac{k^2\sin^2(\pi\omega)}{(k^2-\omega^2)^2}
    \mathrm{d}\omega
    \nonumber\\
    &\ge
    \left\Vert p-q\right\Vert_2^2
    -
    \frac{8k^2\alpha}{\pi^2(k^2-\alpha^2)^2},
    \end{align}
    where for the last inequality we assumed $k>\alpha$.
    Then, it can be deduced that for every $\xi<1$, there exists some $k_0\in\mathbb{N}$ such that
    $$
    \left\Vert p-q\right\Vert_2^2
    -
    \frac{8k_0^2\alpha}{\pi^2(k_0^2-\alpha^2)^2}
    > \xi\left\Vert p-q\right\Vert_2^2,
    $$
    which contradicts the initial assumption that Assumption \ref{class_assumption_2} holds for some $\alpha$ and $\xi<1$.
    
    The next step is to create a distribution sequence $f_1,f_2,\ldots\in\mathcal{F}$ such that $f_n*G$ converges in TV sense, but $f_n$ does not. This is straightforward by, for example, considering the following sequence:
    \begin{align}
        f_n\triangleq
        \frac{1+\sin(nx)}{2\pi},\quad
        \forall x\in[0,2\pi],~n\in\mathbb{N}.
    \end{align}
    It can be seen that
    \begin{align}
    \lim_{n\to\infty}f_n*G
    &=
    \lim_{n\to\infty}
    \left(\frac{1+\sin(nx)}{2\pi}\right)*G
    \nonumber\\
    &=
    \mathsf{Uniform}\left([0,2\pi]\right)*G
    +
    \frac{1}{2\pi}
    \lim_{n\to\infty}
    \sin(nx)*G,
    \end{align}
    where equalities are point-wise. Since $G=\mathcal{N}(0,\sigma^2_0)$, the residual function can be computed as:
    \begin{align}
    \left\{
    \lim_{n\to\infty}
    \sin(nx)*G
    \right\}(x)
    &=
    \frac{1}{\sigma_0\sqrt{2\pi}}
    \lim_{n\to\infty}
    \int_{-\infty}^{\infty}
    \sin(nu)
    e^{-(x-u)^2/(2\sigma_0^2)}
    \mathrm{d}u,\quad
    x\in\reals,
    \nonumber\\
    &\stackrel{\mathrm{(i)}}{=}
    \lim_{n\to\infty}
    \frac{-1}{n\sqrt{2\pi}}
    e^{-(x-u)^2/(2\sigma_0^2)}\bigg\vert_{-\infty}^{\infty}
    \nonumber\\
    &\quad~~
    -
    \frac{1}{\sigma_0\sqrt{2\pi}}
    \lim_{n\to\infty}
    \int_{-\infty}^{\infty}
    \frac{\cos(nu)}{n}
    \left(\frac{u-x}{\sigma^2_0}\right)
    e^{-(x-u)^2/(2\sigma_0^2)}
    \mathrm{d}u,
    \end{align}
    which equals to zero for all $x\in\reals$. The equality (i) comes from applying integration by part. Therefore, we have
    \begin{align}
    \lim_{n\to\infty}
    \mathsf{TV}
    \left(f_n,
    \mathsf{Uniform}\left([0,2\pi]\right)*G\right)
    =0.
    \end{align}
    On the other hand, the distributional sequence $f_n,~n\in\mathbb{N}$ is not a Cauchy series with respect to total variation distance, since
    \begin{align}
    \liminf_{n\to\infty}
    \mathsf{TV}\left(f_{2n},f_n\right)
    &=
    \liminf_{n\to\infty}
    \frac{1}{2\pi}\int_{0}^{2\pi}
    \left\vert \sin(2nx)-\sin(nx)\right\vert
    \mathrm{d}x
    \nonumber\\
    &=\frac{1}{\pi}
    \liminf_{n\to\infty}
    \int_{0}^{2\pi}
    \left\vert \sin\left(\frac{nx}{2}\right)\right\vert
    \cdot
    \left\vert \cos\left(\frac{3nx}{2}\right)\right\vert
    \mathrm{d}x
    \nonumber\\
    &>0.
    \end{align}
    Therefore, $\left\{f_n\right\}_{n\in\mathbb{N}}$ does not converge in the sense of total variation distance. From a different perspective, we have
    \begin{align}
    \mathsf{TV}\left(f_n,\mathrm{Uniform}
    \left([0,2\pi]\right)\right)
    &=\frac{1}{2\pi}\int_{0}^{2\pi}
    \left\vert\sin(nx)\right\vert\mathrm{d}x
    \nonumber\\
    &=4,
    \end{align}
    for all $n\in\mathbb{N}$. Hence, $f_n$s do not converge (in TV distance) to any distribution, specially $\mathsf{Uniform}([0,2\pi])$. This completes the proof.
\end{proof}

%% file: frequency_gaussian_claim_proof.tex
\begin{proof}[Proof of Claim \ref{frequency_gaussian_claim}]
Assume to arbitrary distributions in $\mathcal{F}$, namely
$p = \mathcal{N}(\boldsymbol{\mu}_1, \sigma_1^2\boldsymbol{I}_d)$ and $q = \mathcal{N}(\boldsymbol{\mu}_2, \sigma_2^2\boldsymbol{I}_d)$, where their respective Fourier transforms $P,Q:\reals^d\to\mathbb{C}$ can be written as follows:
\begin{align}
    P(\boldsymbol{w}) = 
    e^{i\boldsymbol{\mu}^T_1\boldsymbol{w}-
        \frac{1}{2}\sigma_1^2\left\Vert \boldsymbol{w} \right\Vert_2^2},
    \quad
    Q(\boldsymbol{w}) = 
    e^{i\boldsymbol{\mu}^T_2\boldsymbol{w}-\frac{1}{2}\sigma_2^2\left\Vert \boldsymbol{w} \right\Vert_2^2}.
\end{align}
Without loss of generality assume that $\sigma_1 \leq \sigma_2$. Therefore, we have
\begin{align}
    \left \vert 
        P(\boldsymbol{w}) - Q(\boldsymbol{w}) 
    \right \vert^2 
    &= 
    \left \vert 
        e^{i\boldsymbol{\mu}^T_1\boldsymbol{w}-
        \frac{1}{2}\sigma_1^2\left\Vert \boldsymbol{w} \right\Vert_2^2} 
        - 
        e^{i\boldsymbol{\mu}^T_2\boldsymbol{w}-
        \frac{1}{2}\sigma_2^2\left\Vert \boldsymbol{w} \right\Vert_2^2}
    \right \vert^2 
    \nonumber\\
    &= 
    e^{-\sigma_1^2\left\Vert \boldsymbol{w} \right\Vert_2^2/2}
    \cdot
    \left \vert 
    e^{-\frac{1}{4}\sigma_1^2
    \left\Vert \boldsymbol{w} \right\Vert_2^2} 
    -
    e^{i\left(\boldsymbol{\mu}_2-\boldsymbol{\mu}_1\right)^T\boldsymbol{w}}
    e^{-\frac{1}{2}(\sigma_2^2 - \sigma_1^2/2)
    \left\Vert \boldsymbol{w} \right\Vert_2^2}
    \right \vert^2
    \nonumber\\
    &\triangleq
    f(\left\Vert\boldsymbol{w}\right\Vert_2)\cdot g(\boldsymbol{w}).
\end{align}
Note that since we have $\sigma_1\ge\sigma_0>0$, $f$ is strictly decreasing and both $f$ and $g$ become exponentially small as $\left\Vert\boldsymbol{w} \right\Vert_2\to\infty$. In this regard, one can write:
\begin{align}
\int_{\left\Vert \boldsymbol{w}\right\Vert_2\geq \alpha} \left\vert P(\boldsymbol{w})-Q(\boldsymbol{w})\right\vert^2 \mathrm{d}\boldsymbol{w} 
&=
\int_{\left\Vert \boldsymbol{w}\right\Vert_2\geq \alpha} 
f(\left\Vert\boldsymbol{w}\right\Vert_2)g(\boldsymbol{w})
\mathrm{d}\boldsymbol{w}
\nonumber\\
&\leq
\int_{\reals^d} 
\min\left\{
f(\left\Vert\boldsymbol{w}\right\Vert_2),f(\alpha)
\right\}
g(\boldsymbol{w})
\mathrm{d}\boldsymbol{w}
\nonumber\\
&\leq
f(\alpha)
\int_{\reals^d} 
g(\boldsymbol{w})
\mathrm{d}\boldsymbol{w}.
\end{align}
Now, it should be noted that due to Parseval's theorem, we have
\begin{align}
\int_{\reals^d} 
g(\boldsymbol{w})
\mathrm{d}\boldsymbol{w}
&=
\int_{\reals^d} 
\left \vert 
    e^{-\frac{1}{4}\sigma_1^2
    \left\Vert \boldsymbol{w} \right\Vert_2^2} 
    -
    e^{i\left(\boldsymbol{\mu}_2-\boldsymbol{\mu}_1\right)^T\boldsymbol{w}}
    e^{-\frac{1}{2}(\sigma_2^2 - \sigma_1^2/2)
    \left\Vert \boldsymbol{w} \right\Vert_2^2}
    \right \vert^2
\mathrm{d}\boldsymbol{w}
\nonumber\\
&=
\left\Vert
\mathcal{N}\left(\boldsymbol{\mu}_1,\frac{\sigma^2_1}{2}\boldsymbol{I}_d\right)
-
\mathcal{N}\left(\boldsymbol{\mu}_2,\left(\sigma^2_2-\frac{\sigma^2_1}{2}\right)\boldsymbol{I}_d\right)
\right\Vert^2_2.
\end{align}
Therefore, so far we have shown that
\begin{align}
\int_{
\left\Vert \boldsymbol{w}\right\Vert_2\geq \alpha
} \left\vert 
P(\boldsymbol{w})-Q(\boldsymbol{w})
\right\vert^2 
\mathrm{d}\boldsymbol{w} 
&\leq   
e^{-\sigma^2_1\alpha^2/2}
\left\Vert
\mathcal{N}\left(
\boldsymbol{\mu}_1,\frac{\sigma^2_1}{2}\boldsymbol{I}_d
\right)
-
\mathcal{N}\left(\boldsymbol{\mu}_2,\left(\sigma^2_2-\frac{\sigma^2_1}{2}\right)\boldsymbol{I}_d\right)
\right\Vert^2_2.
\end{align}
Next, we use the following formula to the $\ell_2$-norm between two isotropic Gaussian densities in $\reals^d$:
\begin{align}
\left\Vert p-q\right\Vert^2_2
&=
\left\Vert
\mathcal{N}\left(\boldsymbol{\mu}_1,\sigma^2_1\boldsymbol{I}_d\right)
-
\mathcal{N}\left(\boldsymbol{\mu}_2,\sigma^2_2\boldsymbol{I}_d\right)
\right\Vert^2_2
\nonumber\\
&=
\frac{1}{(4\pi)^{d/2}} 
\left( 
\frac{1}{\sigma_1^d} + \frac{1}{\sigma_2^d} - \frac{2^{1 + d/2}}{(\sigma_1^2 + \sigma_2^2)^{d/2}} 
\exp
\left( 
    -\frac{\|\boldsymbol{\mu}_1 - \boldsymbol{\mu}_2\|^2}{2 (\sigma_1^2 + \sigma_2^2)}
\right) 
\right).
\end{align}
In this regard, for fixed parameters $d\in\mathbb{N}$, $\theta,\gamma\ge0$, let us define the function
\begin{align}
\label{eq:claimProof:h13farvardin}
h(\sigma)
=
h(\sigma\vert\theta,\gamma,d)
&\triangleq
\left(4\pi\right)^{d/2}
\left\Vert
\mathcal{N}\left(
\boldsymbol{0},
\sigma^2\boldsymbol{I}_d
\right)
-
\mathcal{N}\left(\Delta\boldsymbol{\mu},(\sigma^2+\theta)\boldsymbol{I}_d\right)
\right\Vert^2_2
\\
&=
\frac{1}{\sigma^d}
+
\frac{1}{
\left(\sigma^2+\theta\right)^{d/2}}
-\frac{2^{1+d/2}
}{
\left(2\sigma^2+\theta\right)^{d/2}
}
\exp\left(
-\frac{\gamma}{4\sigma^2+2\theta}
\right),
\nonumber
\end{align}
where $\gamma\triangleq\left\Vert\Delta\boldsymbol{\mu}\right\Vert^2_2\ge0$. Then, it can be seen that we have
\begin{align}
\int_{
\left\Vert \boldsymbol{w}\right\Vert_2\geq \alpha
} \left\vert 
P(\boldsymbol{w})-Q(\boldsymbol{w})
\right\vert^2 
\mathrm{d}\boldsymbol{w} 
&\leq   
e^{-\sigma^2_1\alpha^2/2}
\left\Vert p-q\right\Vert^2_2
\cdot
\frac{h(\sigma_1/\sqrt{2})}{h(\sigma_1)},
\end{align}
with $\gamma\gets\left\Vert\boldsymbol{\mu}_2-\boldsymbol{\mu}_1\right\Vert^2_2$ and $\theta\gets\sigma^2_2-\sigma^2_1$. On the other hand, algebraic analysis of \eqref{eq:claimProof:h13farvardin} reveals that we always have
\begin{align}
\label{eq_ratio_gaussian_low_frequency_proof}
\sup_{\sigma_1\ge\sigma_0}~
\sup_{\theta,\gamma\ge0}~
\frac{
h\left(\frac{\sigma_1}{\sqrt{2}}\big\vert\gamma,\theta,d\right)
}{
h\left(\sigma_1\big\vert\gamma,\theta,d\right)
}
\leq 4\cdot2^{d/2}.
\end{align}
As a result, and considering $\sigma_1\ge\sigma_0$, the following bound holds and the proof is complete:
\begin{align}
\int_{
\left\Vert \boldsymbol{w}\right\Vert_2\geq \alpha
} \left\vert 
P(\boldsymbol{w})-Q(\boldsymbol{w})
\right\vert^2 
\mathrm{d}\boldsymbol{w} 
&\leq 
2^{(d/2+2)}e^{-\sigma^2_0\alpha^2/2}
\left\Vert p-q\right\Vert^2_2.
\end{align}

\end{proof}

%% file: frequency_uniform_claim_proof.tex
\begin{proof}[Proof of Claim \ref{claim:lowFreq:kMixtureUniform}]
Let $p,q\in k\mathrm{-Mix}(\mathcal{F})$. Then, there exist coefficient vectors $\boldsymbol{\alpha}=(\alpha_1,\ldots,\alpha_k)$ and $\boldsymbol{\beta}=(\beta_1,\ldots,\beta_k)$ with $\boldsymbol{\alpha},\boldsymbol{\beta}\in\Delta^{k-1}$, such that
\begin{align}
p(x)-q(x)=
\sum_{i=1}^{k}
\alpha_i
\frac{\mathbbm{1}(a_i\leq x\leq b_i)}{b_i-a_i}
-
\sum_{i=1}^{k}
\beta_i
\frac{\mathbbm{1}(a'_i\leq x\leq b'_i)}{b'_i-a'_i},
\end{align}
where we have $a_i,b_i,a'_i,b'_i\in\reals$ for all $i\in[k]$, and $b_i-a_i,b'_i-a'_i\ge T$. We first note that
$$
\left\Vert p-q\right\Vert_{\infty}=\sup_{x\in\reals}~
\left\vert
p(x)-q(x)
\right\vert\leq \frac{1}{T}.
$$
Let $c_1\leq c_2\leq\ldots\leq c_{4k}$ be (at most) $4k$ unique points representing the sorted values of $a_i,b_i,a'_i$ and $b'_i$ for all $i$. In this regard, the function $p(x)-q(x)$ is piecewise constant on intervals $(c_{i-1},c_i)$ for $i\in[4k]$, and is particularly zero in both $(-\infty,c_1)$ and $(c_{4k},+\infty)$. Using this fact, $p-q$ can be rewritten as the sum of (at most) $4k-1$ separate pulse functions, as follows:
\begin{align}
p(x)-q(x)=\sum_{i=1}^{4k-1}
h_i\mathbbm{1}\left(x\in I_i\right),
\end{align}
where $I_i=(c_{i},c_{i+1})$, and $-1/T\leq h_i\leq 1/T$. Note that this does come at the loss of generality, since we can always assume some of $h_i$s are zero. Also, we have neglected the values of $p-q$ at the discontinuity points. Also, let $t_i\triangleq \mathrm{len}(I_i)$ denote the length of the interval $I_i$.

Assume $P(w),Q(w)$ represent the Fourier transforms of $p,q$, respectively. Then, for any $\alpha>0$ we have
\begin{align}
\frac{1}{2\pi}
\int_{\left\vert w\right\vert\ge\alpha}
\left\vert P(w)-Q(w)\right\vert^2\mathrm{d}w
&=
\frac{1}{\pi}\int_{\alpha}^{\infty}
\left\vert P(w)-Q(w)\right\vert^2\mathrm{d}w
\nonumber\\
&=
\frac{1}{\pi}\int_{\alpha}^{\infty}
\left\vert 
\sum_{i=1}^{4k-1}
h_i\mathsf{F}\left\{
\mathbbm{1}\left(x\in I_i\right)
\right\}(w)
\right\vert^2\mathrm{d}w
\nonumber\\
&\stackrel{(*)}{=}
\frac{1}{\pi}\int_{\alpha}^{\infty}
\sum_{i=1}^{4k-1}
h^2_i
\left\vert 
\mathsf{F}\left\{
\mathbbm{1}\left(x\in I_i\right)
\right\}(w)
\right\vert^2\mathrm{d}w
\nonumber\\
&=
\frac{4}{\pi}
\sum_{i=1}^{4k-1}
h^2_i
\int_{\alpha}^{\infty}
\frac{\sin^2\left(wt_i/2\right)}{w^2}
\mathrm{d}w
\nonumber\\
&=
\frac{2}{\pi}
\sum_{i=1}^{4k-1}
h^2_it_i
\left(\frac{\pi}{2}-
\int_{0}^{\alpha t_i/2}
\frac{\sin^2\left(u\right)}{u^2}
\mathrm{d}u
\right)
\nonumber\\
&=
\sum_{i=1}^{4k-1}h^2_it_i\left(1-\zeta\left(
\frac{\alpha t_i}{2}
\right)\right).
\end{align}
The equality (*) holds since functions $h_i\mathbbm{1}\left(x\in I_i\right)$ and $h_j\mathbbm{1}\left(x\in I_j\right)$ for $i\neq j$ are \emph{orthogonal} due to non-overlapping supports $I_i$ and $I_j$. Fourier transform, similar to any other orthonormal transformation, preserves orthogonality. Define $\varepsilon^2 \triangleq \left\Vert p - q \right\Vert_2^2$ and note that we have 
$$
\varepsilon^2=\sum_{i=1}^{4k-1} h_i^2 t_i.
$$ 
Then for fixed $\alpha$ and $\varepsilon$, and over varying $p, q \in k\text{-}\mathrm{Mix}(\mathcal{F})$, we obtain:
\begin{align}
\frac{1}{2\pi}
\int_{\left\vert w\right\vert\ge\alpha}
\left\vert P(w)-Q(w)\right\vert^2\mathrm{d}w
&\leq
\sup_{(h_i,t_i),~\forall i\in[4k-1]}~
\sum_{i=1}^{4k-1}h^2_it_i\left(1-\zeta\left(
\frac{\alpha t_i}{2}
\right)\right)
\nonumber\\
&=
\varepsilon^2-
\inf_{(h_i,t_i),~\forall i\in[4k-1]}~
\sum_{i=1}^{4k-1}h^2_it_i\zeta\left(
\frac{\alpha t_i}{2}
\right)
\nonumber\\
&\quad\mathrm{subject~to}\quad
\sum_{i=1}^{4k-1}h^2_it_i=\varepsilon^2,
\quad t_i\ge0,~\left\vert h_i\right\vert\leq\frac{1}{T}.
\end{align}
Since $\zeta$ is non-decreasing, the minimum is achieved when all $t_i$ are equal. Thus,
$$
t^*_i=t\triangleq\frac{\varepsilon^2}{\sum_{i=1}^{4k-1}h^2_i},\quad \forall j\in[4k-1].
$$
Substituting into the bound, we have
\begin{align}
\frac{1}{2\pi}
\int_{\left\vert w\right\vert\ge\alpha}
\left\vert P(w)-Q(w)\right\vert^2\mathrm{d}w
&\leq
\varepsilon^2\left(
1-
\inf_{h_1,\ldots,h_{4k-1}}
\zeta\left(\frac{\alpha\varepsilon^2}{2\sum_{i}h^2_i}\right)
\right)
\nonumber\\
&=
\varepsilon^2\left(
1-
\zeta\left(\frac{\alpha T^2\varepsilon^2}{2(4k-1)}\right)
\right),
\end{align}
which completes the proof.
\end{proof}

\vspace{4mm}

\begin{proof}[Proof of Claim \ref{claim:lowFreq:UnifdDim}]
The proof closely follows the argument in Claim \ref{claim:lowFreq:kMixtureUniform} (see Appendix \ref{sec:app:proofs:claims:II}). Let $\boldsymbol{\alpha}, \boldsymbol{\beta} \in \Delta^{k-1}$ be $d$-dimensional discrete probability vectors, and consider an arbitrary difference measure $p - q$ with $p, q \in k\text{-Mix}(\mathcal{F})$:
\begin{align}
p(\boldsymbol{x})-q(\boldsymbol{x})
=
\sum_{i=1}^{k}\alpha_i
\prod_{j=1}^{d}\frac{\mathbbm{1}\left(x_j\in I_{i,j}\right)}{
\left\vert I_{i,j}\right\vert
}
-
\sum_{i=1}^{k}\beta_i
\prod_{j=1}^{d}\frac{\mathbbm{1}\left(x_j\in I'_{i,j}\right)}{
\left\vert I'_{i,j}\right\vert
},\quad\forall\boldsymbol{x}\in\reals^d,
\end{align}
where $I_{i,j}$ and $I'_{i,j}$ for $i\in[k]$ and $j\in[d]$ are arbitrary intervals with a minimum length (i.e., Lebesgue measure) of $T$. This function is piecewise constant and nonzero over the union of at most $2k$ axis-aligned and potentially overlapping rectangles in $\mathbb{R}^d$. According to several known results in high-dimensional or computational geometry, particularly those concerning orthogonal range decomposition or boolean combinations of boxes (see, for example, \cite{deBerg2008computational, chazelle1988functional, overmars1991new}), this sum can be decomposed into a linear combination of at most $\Theta\left(k^d\right)$ disjoint axis-aligned rectangles:
\begin{align}
p(\boldsymbol{x})-q(\boldsymbol{x})
=
\sum_{i=1}^{\Theta\left(k^d\right)}
h_i\prod_{j=1}^{d}\frac{\mathbbm{1}\left(
x_i\in \mathsf{T}_{i,j}
\right)}{\left\vert t_{i,j} \right\vert},
\end{align}
where $\mathsf{T}_{i,j}$ denotes a 1D interval of length $t_{i,j}$, and the coefficients satisfy $|h_i| \le 1/T$. Proceeding analogously to Claim \ref{claim:lowFreq:kMixtureUniform}, we have
\begin{align}
\frac{1}{(2\pi)^d}
\int_{\left\Vert\boldsymbol{w}\right\Vert_2\ge\alpha}
\left\vert
P(\boldsymbol{w})-Q(\boldsymbol{w})
\right\vert^2\mathrm{d}\boldsymbol{w}
&\leq
\frac{1}{(2\pi)^d}
\int_{\left\Vert\boldsymbol{w}\right\Vert_{\infty}\ge\alpha/\sqrt{d}}
\left\vert
P(\boldsymbol{w})-Q(\boldsymbol{w})
\right\vert^2\mathrm{d}\boldsymbol{w}
\nonumber\\
&=
\varepsilon^2-
\sum_{i=1}^{\Theta\left(k^d\right)}
h^2_i
\prod_{j=1}^{d}\frac{2}{\pi}
t_{i,j}
\int_{0}^{\alpha t_{i,j}/2\sqrt{d}}
\frac{\sin^2 u}{u^2}\mathrm{d}u
\nonumber\\
&\leq
\varepsilon^2-
\inf_{h_i,t_{i,j}}
\sum_{i=1}^{\Theta\left(k^d\right)}
h^2_i
\prod_{j=1}^{d}
t_{i,j}
\zeta\left(\frac{\alpha t_{i,j}}{2\sqrt{d}}\right)
\nonumber\\
&\quad\mathrm{subject~to}\quad
\sum_{i=1}^{\Theta\left(k^d\right)}
h^2_i
\prod_{j=1}^{d}
t_{i,j}=\varepsilon^2.
\end{align}
Applying the method of Lagrange multipliers to this constrained optimization problem, the minimum is achieved when all $t_{i,j}$ are equal (refer to the proof of Claim \ref{claim:lowFreq:kMixtureUniform} in Appendix \ref{sec:app:proofs:claims:II}), yielding
\begin{align}
t^*_{i,j}=\left(\frac{\varepsilon^2}{\sum_{i=1}^{\Theta\left(k^d\right)}h^2_i}
\right)^{1/d}
\ge
\frac{1}{c}
\left(
\frac{T^2\varepsilon^2}{k^d}\right)^{1/d},
\quad\forall i,j,
\end{align}
where $c$ is a universal constant. Substituting this into the earlier bound completes the proof.
\end{proof}

%% file: Adversarial_Sample_Decoder_Proof.tex

\begin{proof}[Proof of Proposition \ref{adversarial_sample_decoder}]
\label{proof_of_adversarial_sample_decoder}

Most of the machinery is already developed in Section \ref{sec:Gaussian}, and we follow a similar path here.
By Assumption \ref{assumption:SC}, the function class $\mathcal{F}$ admits a decoder $\mathcal{J}$ satisfying Assumption \ref{assumption_decoder}. Thus, for any clean i.i.d. samples $\boldX_1, \ldots, \boldX_n \sim f^* \in \mathcal{F}$, the decoder reconstructs $\widehat{f}=\mathcal{J}(\mathbf{L},\mathbf{B})$ such that
$$
\mathbb{P}\left( \mathsf{TV}\left(f^*, \mathcal{J}(\mathbf{L}, \mathbf{B})\right) \le \epsilon \right) \ge 1 - \delta,
$$
provided that $n \ge m(\epsilon)\log(1/\delta)$ and the decoder uses at most $\tau(\epsilon)$ samples shown by the sequence $\mathbf{L}\in\{\boldX_1,\ldots,\boldX_n\}^{\tau(\epsilon)}$ and $t(\epsilon)$ bits shown by $\mathbf{B}\in\{0,1\}^{t(\epsilon)}$.
However, we only observe noisy samples $\widetilde{\boldX}_i = \boldX_i + \boldsymbol{\zeta}_i$. Define the corrupted version of the compression sequence as
$$
\mathbf{L}_{\mathsf{N}} := \left( \boldX + \boldsymbol{\zeta} \right)_{\boldX \in \mathbf{L}},
$$
where for at least $n-s$ samples, $\boldsymbol{\zeta}_i = \boldsymbol{0}$, and for the remaining $s$ samples, we have $\|\boldsymbol{\zeta}_i\|_{\infty} \le C$. To correct for the additive adversarial attacks, we follow the quantization strategy introduced in the proof of Proposition \ref{proposition:SampComp:NoisyF} (see Lemma \ref{lemma:proofNoisyGaussian1}). In particular, we assume the existence of a finite grid $I$ such that every noise vector $\boldsymbol{\zeta}_i$ can be approximated coordinate-wise within an arbitrary resolution $\eta > 0$. The goal is to approximate the perturbation values up to $\eta$ per coordinate and then compensate accordingly.

Unlike the Gaussian noise case, here the adversary is restricted to corrupting only $s$ out of $n$ samples. To compensate for these perturbations, a naive approach is to first guess which $s$ samples were corrupted. There are at most $\binom{n}{s} \le (ne/s)^s$ such choices. Then, for each possible corruption pattern and each possible quantized perturbation in $I^{ds}$ ($s$ chosen samples and $d$ coordinates), we apply the analogue of Lemma \ref{lemma:proofNoisyGaussian1}. That is, for a fixed bit set $\mathbf{B}$, over all $(ne/s)^s |I|^{ds}$ quantized corrections to $\mathbf{L}_{\mathsf{N}}$, there exists at least one corrected sequence $\mathbf{L}'$ such that
$$
\left\Vert \mathbf{L} - \mathbf{L}' \right\Vert_2 \le \eta\sqrt{ds},
\quad \text{and hence} \quad
\mathsf{TV}\left( \mathcal{J}(\mathbf{L}, \mathbf{B}), \mathcal{J}(\mathbf{L}', \mathbf{B}) \right) \le \frac{1}{2} r\eta\sqrt{ds},
$$
using the Lipschitz continuity of $\mathcal{J}$ (Assumption \ref{assumption_decoder}). Combining this with the clean decoder guarantee, we obtain
$$
\mathbb{P}\left( \mathsf{TV}(f^*, \mathcal{J}(\mathbf{L}', \mathbf{B})) \le \epsilon + \frac{1}{2}r\eta\sqrt{ds} \right) \ge 1 - \delta.
$$
We now specify the quantization grid $I$ as in the proof of Proposition \ref{proposition:SampComp:NoisyF}: let $I = \{-C, -C + 2\eta, \ldots, C\}$ so that $|I| = 1 + \frac{C}{\eta}$ covers the support of each $\zeta_{ij} \in [-C, C]$. To encode the choice of quantized corrections, we require $ds \log_2 |I| + \lceil s \log_2(ne/s)\rceil$ bits. Substituting $|I| = 1 + \frac{C}{\eta}$ and $n = m(\epsilon)\log(1/\delta)$, the total bit budget becomes
$$
t(\epsilon) + ds \log_2\left(1 + \frac{C}{\eta} \right) + s \log_2\left(\frac{e}{s} m(\epsilon)\log(1/\delta) \right),
$$
where we neglect $\lceil\cdot\rceil$ to enhance clarity. To ensure that the final total variation error does not exceed $\epsilon$, we set
$$
\epsilon \leftarrow \epsilon/2
\quad \text{and} \quad
\frac{1}{2}r\eta\sqrt{ds} \leftarrow \epsilon/2,
\quad \text{yielding} \quad
\eta = \frac{\epsilon}{r\sqrt{ds}}.
$$
Plugging this into the bit expression and using the updated $\epsilon/2$, we obtain a new compression scheme:
$$
\left[
\tau\left(\frac{\epsilon}{2}\right),
t\left(\frac{\epsilon}{2}\right) + 
ds \log_2\left(1 + \frac{Cr\sqrt{ds}}{\epsilon} \right) +
s \log_2\left(\frac{e}{s} m\left(\frac{\epsilon}{2}\right)\log(1/\delta)\right),
m\left(\frac{\epsilon}{2}\right)\log(1/\delta)
\right].
$$
This guarantees that $\mathbb{P}(\mathsf{TV}(f^*, \widehat{f}) \le \epsilon) \ge 1 - \delta$, completing the proof.
\end{proof}

%% file: adversarial_main_thoerem_proof.tex

\begin{proof}[Proof of Theorem \ref{adversarial_main_theorem}]
\label{proof_of_adversarial_main_theorem}

We begin by invoking Proposition \ref{adversarial_sample_decoder}, which states that if $\mathcal{F}$ is $(\tau, t, m)$-sample compressible in the clean (non-adversarial) setting, then it remains compressible even when $s$ out of the $m$ i.i.d. samples are corrupted adversarially. Specifically, for any $(\epsilon, \delta) \in (0,1)$, $\mathcal{F}$ admits a sample compression scheme of size
\begin{equation}
\left[
\tau\left(\tfrac{\epsilon}{2}\right),
t\left(\tfrac{\epsilon}{2}\right) +
ds \log_2\left(1 + \frac{1}{\epsilon}Cr\sqrt{ds}\right) +
s \log_2\left(m\left(\tfrac{\epsilon}{2}\right)\log\tfrac{1}{\delta}\right),
m\left(\tfrac{\epsilon}{2}\right)\log\tfrac{1}{\delta}
\right].
\nonumber
\end{equation}
This implies the existence of a decoder $\mathcal{J}$ such that, given $m(\epsilon/2)\log(1/\delta)$ i.i.d. samples from some $f^* \in \mathcal{F}$, with at most $s$ adversarial corruptions, the decoder outputs a hypothesis $\widehat{f} \in \mathcal{F}$ satisfying $\mathsf{TV}(f^*, \widehat{f}) \le \epsilon$ with probability at least $1 - \delta$. Fix $(\epsilon', \delta') \in (0,1)$, and define
\begin{align}
\label{sample_complexity_main_bound_adversary}
M &\triangleq
\left(m\left(\epsilon'/2\right)\log\tfrac{1}{\delta'}\right)^{\tau(\epsilon'/2)} \cdot
2^{t(\epsilon'/2) +
ds \log_2\left(1 + \tfrac{2Cr\sqrt{ds}}{\epsilon'}\right) +
s \log_2\left(\tfrac{e}{s}m(\epsilon'/2)\log\tfrac{1}{\delta'}\right)}.
\end{align}
Now consider the $M$ candidate functions $f_1, \dots, f_M \in \mathcal{F}$ generated by applying decoder $\mathcal{J}$ to all valid compression pairs $(\mathbf{L}, \mathbf{B})$ (of respective sizes $\tau(\epsilon'/2)$ and $\Tilde{t}(\epsilon'/2)$) from the corrupted sample set. By Proposition \ref{sample_compression_theorem}, we have:
$$
\mathbb{P}\left(
\min_{i \in [M]}~
\Vert f_i - f^* \Vert_1 \leq \epsilon'
\right)
\ge 1 - \delta'.
$$
Now assume we additionally draw $n$ i.i.d. samples from $f^*$, for the moment assuming these are clean. Then, by Theorem \ref{existence_of_algorithm_with_samples}, for any $(\epsilon'', \delta'') \in (0,1)$, there exists a deterministic algorithm $\mathcal{A}$ which, given
\[
n \ge \frac{\log \left( M^2 / \delta'' \right)}{2(\epsilon'')^2},
\]
outputs a candidate $\widehat{f}$ such that
\begin{align}
\mathbb{P}\left( \mathsf{TV}(f^*, \widehat{f}) \le 
\min_{i \in [M]} \Vert f_i - f^* \Vert_1 + 4\epsilon'' \right)
\ge 1 - \delta'',
\nonumber \\
\Longrightarrow \quad
\mathbb{P}\left( \mathsf{TV}(f^*, \widehat{f}) \le 
\epsilon' + 4\epsilon'' \right)
\ge 1 - \delta' - \delta''.
\end{align}
However, in the adversarial setting, we cannot assume access to perfectly clean samples. Therefore, we collect an additional $(2s+1)n$ samples, partitioned into $2s+1$ groups of $n$ samples each. By the pigeonhole principle, at least $s+1$ of these groups must be uncorrupted.

Apply algorithm $\mathcal{A}$ independently to each of the $2s+1$ groups based on the candidate set $\{f_1, \dots, f_M\}$, and let the resulting hypotheses be $\widehat{g}_1, \dots, \widehat{g}_{2s+1} \in \mathcal{F}$. Then, using a union bound, with probability at least $1 - \delta' - (2s+1)\delta''$, we have:
$$
\left| \mathcal{G} \triangleq 
\left\{
j ~\middle|~
\Vert \widehat{g}_j - f^* \Vert_1 \le \epsilon' + 4\epsilon''
\right\}
\right| \ge s+1.
$$
Using the triangle inequality for total variation distance, it follows that with the same probability,
$$
\Vert \widehat{g}_i - \widehat{g}_j \Vert_1 \le 2\epsilon' + 8\epsilon''
\quad \text{for all } i, j \in \mathcal{G}.
$$

Therefore, there exists a clique of size at least $s+1$ within $\{\widehat{g}_i\}$, where the pairwise distances are at most $2\epsilon' + 8\epsilon''$. Each member of such a clique must then be within TV distance $3\epsilon' + 12\epsilon''$ of $f^*$. Hence, we can construct a deterministic algorithm $\mathscr{B}$ (based on $\mathcal{A}$ and this clique-selection procedure) that, using $m(\epsilon'/2)\log(1/\delta') + (2s+1)n$ samples (with up to $s$ adversarial perturbations), outputs a hypothesis $\widehat{f}^* \in \{f_1,\dots,f_M\}$ such that
$$
\mathbb{P}\left( \Vert \widehat{f}^* - f^* \Vert_1 \le 3\epsilon' + 12\epsilon'' \right)
\ge 1 - \delta' - (2s+1)\delta''.
$$
Finally, choose
\begin{align}
\epsilon' \gets \epsilon / 6, \quad
\epsilon'' \gets \epsilon / 24, \quad
\delta' \gets \delta / 2, \quad
\delta'' \gets \delta / [2(2s+1)],
\end{align}
so that $\epsilon = 3\epsilon' + 12\epsilon''$ and $\delta = \delta' + (2s+1)\delta''$. Substituting these choices, the total sample complexity becomes:
\begin{align}
n &\ge m\left(\tfrac{\epsilon}{12}\right)\log\frac{2}{\delta} +
(2s+1)\frac{\log\left(2(2s+1)M^2(\epsilon/6,\delta/2)/\delta\right)}{2(\epsilon/24)^2}
\nonumber \\
&= \mathcal{O}\left(
m\left(\tfrac{\epsilon}{12}\right)\log\tfrac{1}{\delta}
+
\frac{s}{\epsilon^2} \left[
\tau\left(\tfrac{\epsilon}{12}\right)\log\left(
m\left(\tfrac{\epsilon}{12}\right)\log\tfrac{1}{\delta}
\right) +
t\left(\tfrac{\epsilon}{12}\right)
+
ds \log_2\left(1 + \tfrac{Cr\sqrt{ds}}{\epsilon}\right)
\right. \right.
\nonumber\\
&\hspace{1cm}
\left.\left.
+
s \log_2\left(\tfrac{e}{s}m\left(\tfrac{\epsilon}{12}\right)\log\tfrac{1}{\delta}\right)
\right]
\right),
\end{align}
which concludes the proof.
\end{proof}

%% file: App_Examples_Proof.tex
\begin{proof}[Proof of Proposition \ref{proposition:example:noisy-kUMM}]
Based on Theorem \ref{main_theorem} and Corollary \ref{corl:GaussianLaplaceMainThm}, in this problem setting, we can guarantee that there exists an algorithm such that upon having perturbed samples from $f^*$, it outputs $\widehat{f}$ that
$\Vert\widehat{f}-f^*\Vert_2^{2}
\leq
\frac{24\epsilon}{T}
\left(
\pi ck
\sigma
\sqrt{2e}
\right)^{d/2}$. The required number of samples for this purpose will be 
\begin{align}
    n ~\ge~&
    \mathcal{O} \left( 
    \frac{288dk}{\epsilon}\log \frac{2}{\delta}\log \frac{6k}{\epsilon}\log (3d)
    \log_3 \left( \frac{4}{\delta} \right) \right)
     \nonumber \\ & ~+ 
     \mathcal{O} \left( 
     \frac{8}{9\epsilon^2} 
     \left[ 
     \left( 2kd + k\log_2\frac{4k}{\epsilon} \right) \log \left( \frac{288dk}{\epsilon}\log \frac{2}{\delta}\log \frac{6k}{\epsilon}\log (3d)
     \log_3 \left( \frac{4}{\delta} \right) \right) + \log\left( \frac{12}{\delta} \right)
     \right]
    \right)
    \nonumber \\ & ~+
    \mathcal{O}\left[
    \frac{4d^2 k}{9\epsilon^2}
    \log_2\left(
    \frac{rd\sqrt{2k}}{2\epsilon}
    \left|\Phi_{\mathcal{N}(\boldsymbol{0}, \sigma^2\boldsymbol{I}_d)}^{-1}\left(\frac{\delta \epsilon}{2304 k d^2 \log \left(\frac{4}{\delta}\right)
    \log \left(\frac{2}{\delta}\right) \log \left(\frac{6k}{\epsilon} \right)\log \left( 3d \right)
    }\right)\right|
    \right)
    \right]
    \nonumber \\
    &~\times 
    \log\left( 
    \frac{288dk}{\epsilon} \log \left(\frac{2}{\delta}\right) \log \left(\frac{6k}{\epsilon}\right) \log \left( 3d \right)
    \log_3 \left( \frac{4}{\delta} \right) \right) 
    .
    \nonumber
\end{align}
In the gaussian noise case, $B_G(\alpha) = e^{-(\sigma \alpha)^2}$, therefore by having 
\begin{align}
n ~\ge~&
N^{\mathsf{Clean}}_{\tau,t,m}(6\epsilon,\delta/2)
~+
\nonumber \\
&\mathcal{O}\left[
\frac{4d^2 k}{9\epsilon^2}
\log_2\left(
\frac{rd\sqrt{2k}}{2\epsilon}
\left|\Phi_{\mathcal{N}(\boldsymbol{0}, \sigma^2\boldsymbol{I}_d)}^{-1}\left(\frac{\delta \epsilon}{2304 k d^2 \log \left(\frac{4}{\delta}\right)
\log \left(\frac{2}{\delta}\right) \log \left(\frac{6k}{\epsilon} \right)\log \left( 3d \right)
}\right)\right|
\right)
\right]
\nonumber \\
&\times 
\log\left( 
\frac{288dk}{\epsilon} \log \left(\frac{2}{\delta}\right) \log \left(\frac{6k}{\epsilon}\right) \log \left( 3d \right)
\log_3 \left( \frac{4}{\delta} \right) \right) 
,
\nonumber
\end{align}
which can be subsequently simplified as
\begin{align}
\label{eq:example:lowerBoundForN}
n \geq \mathcal{O}\left( \frac{d^2k}{\epsilon^2} \log^2\left( \frac{rdk\sigma}{\epsilon\delta} \right) \right)
\end{align}
noisy samples, there exists a deterministic algorithm that takes these perturbed samples as input, and outputs $\widehat{f}\in\mathcal{F}$ such that the following bound holds with probability at least $1-\delta$:
\begin{align}
\Vert\widehat{f}-f^*\Vert_2\leq
\epsilon 
\inf_{(\alpha,\xi)\in\mathsf{P}(\mathcal{F})}
24\sqrt{\frac{e^{(\sigma\alpha)^2}}{1 - \xi}}
\end{align}
where 
$$\xi = 1-\zeta^d\left(\frac{\alpha}{2ck\sqrt{d}}\left(T \Vert\widehat{f}-f^*\Vert_2 \right)^{2/d}\right).$$
Hence, we can conclude the following bound holds for small $\frac{\alpha}{2ck\sqrt{d}}\left(T \Vert\widehat{f}-f^*\Vert_2 \right)^{2/d}$ (due to the fact that $\zeta(h) = \frac{2}{\pi}h$ for small $h$): 
\begin{align}
\Vert\widehat{f}-f^*\Vert_2^{2}
\leq
\epsilon 
\inf_{\alpha > 0}
\frac{24}{T}
\sqrt{
\left(\pi ck\sqrt{d}\right)^d
\frac{e^{(\sigma\alpha)^2}}{\alpha^d}}
=
\frac{24\epsilon}{T}
\left(\pi ck\sqrt{d}\right)^{d/2}
\sqrt{
\inf_{\alpha > 0}
\frac{e^{(\sigma\alpha)^2}}{\alpha^d}}
\nonumber
\end{align}
and the $\inf_{\alpha>0}$ happens when $\alpha = \frac{1}{\sigma}\sqrt{\frac{d}{2}}$, so this bound will become:
\begin{align}
\label{eq:example:L2basedOnEpsGaussianNoise}
\Vert\widehat{f}-f^*\Vert_2^{2}
&\leq
\frac{24\epsilon}{T}
\left(
\pi ck
\sigma
\sqrt{2e}
\right)^{d/2}.
\end{align}
Looking back as \eqref{eq:example:lowerBoundForN}, we can upper-bound $\epsilon$ based on $n$ (and other parameters $k,d,\delta$ and $\sigma$) as follows:
\begin{align}
\label{eq:example:upperBoundEpsilon1}
\epsilon \leq \mathcal{O}\left( \sqrt{ \frac{d^2k}{n} } \log\left( \frac{n\sigma}{d^2k\delta} \right) \right).
\end{align}
Plugging \eqref{eq:example:upperBoundEpsilon1} into \eqref{eq:example:L2basedOnEpsGaussianNoise}, we get
\begin{align}
\Vert\widehat{f}-f^*\Vert_2
\leq
\mathcal{O}\left(
\frac{k^{(d+1)/4} \sigma^{d/4}}{n^{1/4}}
\sqrt{\frac{d}{T}
\log\left( \frac{n\sigma}{d^2k\delta} \right)}
\right).
\end{align}
Checking to see of the approximation of $\zeta(h)\simeq\tfrac{2}{\pi}h$ has not caused issues.
based on the above arguments, we have 
\begin{align}
\frac{\alpha}{2ck\sqrt{d}}\left(T \Vert\widehat{f}-f^*\Vert_2 \right)^{2/d} 
&= 
\frac{\sqrt{\pi \sqrt{e}}}{2\sqrt{\sigma ck\sqrt{2}}}\left( 24T\epsilon \right)^{1/d}
\nonumber\\
&=
\mathcal{O}\left(
\frac{(T\epsilon)^{1/d}}{\sqrt{k\sigma}}
\right)
\nonumber\\
&\leq
\mathcal{O}\left(
\frac{1}{\sqrt{k\sigma}}
\left(Td\sqrt{\frac{k}{n}}
\log\left( \frac{n\sigma}{d^2k\delta} \right)
\right)^{1/d}
\right),
\end{align}
which is assumed to be moderate (i.e., $\leq\mathcal{O}(1)$), the approximation that we used for $\zeta(\cdot)$ holds,
which completes the proof.
\end{proof}

\vspace*{4mm}


\begin{proof}[Proof of Proposition \ref{proposition:example:kUMM-adversarial}]
According to discussion earlier, we have all the conditions which completes the requirements for applying Proposition \ref{adversarial_sample_decoder} and Theorem \ref{adversarial_main_theorem}.

Based on Proposition \ref{adversarial_sample_decoder} and Theorem \ref{adversarial_main_theorem}, in this problem setting, we can guarantee that there exists an algorithm such that upon having an adversary that perturb samples from $f^*$ as introduced in Proposition \ref{adversarial_sample_decoder}, it outputs $\widehat{f}$ that
$\mathsf{TV} (\widehat{f} - f^*)  \leq 
\epsilon$ with probability of at least $1-\delta$. The required number of samples for this purpose will be 

\begin{align}
n \ge&\mathcal{O}\left(
\tfrac{3456dk}{\epsilon}  \log \tfrac{1}{\delta} \log \tfrac{2}{\delta} \log \tfrac{72k}{\epsilon} \log(3d)
+ 
\frac{144sk}{\epsilon^2}\log_2 \tfrac{48k}{\epsilon}
\right) 
\nonumber \\
&
+ \mathcal{O}\left(
\frac{144s}{\epsilon^2} \left[
\left(2kd+s\right)\log\left(
\tfrac{3456dk}{\epsilon}  \log \tfrac{1}{\delta} \log \tfrac{2}{\delta} \log \tfrac{72k}{\epsilon} \log(3d)
\right) +
dk \log\left(1 + \frac{8Cd\sqrt{dsk}}{T\epsilon}\right)
\right] \right),
\nonumber
\end{align}

which can be subsequently simplified as
\begin{align}
    n \ge \Tilde{\mathcal{O}} \left(
    \tfrac{dk}{\epsilon} 
    + \tfrac{s^2+2kds}{\epsilon^2} \log\left(
\tfrac{dk}{\epsilon}  
\right)
    + \tfrac{dks}{\epsilon^2} \log \left( 1+\tfrac{Cd\sqrt{dks}}{T\epsilon} \right)
    \right).
\end{align}
This completes the proof.
\end{proof}

\vspace*{4mm}


\begin{proof}[Proof of Proposition \ref{proposition:example:GMM:adversarial}]
We have all the conditions which completes the requirements for applying Proposition \ref{adversarial_sample_decoder} and Theorem \ref{adversarial_main_theorem}.

Based on Proposition \ref{adversarial_sample_decoder} and Theorem \ref{adversarial_main_theorem}, in this problem setting, we can guarantee that there exists an algorithm such that upon having an adversary that perturb samples from $f^*$ as introduced in Proposition \ref{adversarial_sample_decoder}, it outputs $\widehat{f}$ that
$\mathsf{TV} (\widehat{f} - f^*)  \leq 
\epsilon$ with probability of at least $1-\delta$. The required number of samples for this purpose will be 

\begin{align}
n \ge&\mathcal{O}\left(
\frac{dk}{\epsilon}\log k \log(2d)
\log\frac{1}{\delta}
+
\frac{ds^2}{\epsilon^2} \log\left(1 + \frac{C\sqrt{ks}}{12\epsilon\sigma_0 \sqrt{\log(2d)}}\right)
\right)
\nonumber\\
&+
\mathcal{O}\left(
\frac{s}{\epsilon^2} \left[
kd^2 \log(2d) \log \left( \frac{d}{\epsilon} \right) + k \log \left( \frac{k}{\epsilon} \right)
+
\left(
kd \log(2d)
+s\right)\log\left(
\frac{dk}{\epsilon} \log k \log(2d)
\log\frac{1}{\delta}
\right) 
\right] 
\right),
\nonumber
\end{align}

which can be subsequently simplified as
\begin{align}
    n \ge \Tilde{\mathcal{O}} \left(
    \tfrac{dk}{\epsilon} 
    + \tfrac{skd}{\epsilon^2} \left( d\log\left(
\tfrac{d}{\epsilon} \right) +
\log(2d) \log\left(
\tfrac{dk}{\epsilon} 
\right)
\right)
    + \tfrac{ds^2}{\epsilon^2} \log \left( 1+\tfrac{C\sqrt{ks}}{12\epsilon\sigma_0\sqrt{\log(2d)}} \right)
    \right).
\end{align}
This completes the proof.
\end{proof}